\documentclass[letterpaper]{article}

\usepackage[final,nonatbib]{neurips_2021}

\usepackage[utf8]{inputenc} %
\usepackage[T1]{fontenc}    %
\usepackage{hyperref}       %
\usepackage{url}            %
\usepackage{booktabs}       %
\usepackage{amsfonts}       %
\usepackage{nicefrac}       %
\usepackage{microtype}      %
\usepackage{xcolor}         %
\usepackage{graphicx}
\usepackage{subcaption}
\usepackage{booktabs} %
\usepackage{lipsum}
\usepackage{amsmath}
\usepackage{amssymb, amsthm, latexsym}
\usepackage{multirow}
\usepackage{wrapfig}

\usepackage{algorithm,algorithmic}

\usepackage{enumitem}
\usepackage{caption}
\setlist[itemize]{itemsep=0pt}

\newcommand{\R}{\mathbb{R}}

\def\<#1,#2>{\left\langle #1,#2\right\rangle}

\usepackage{thmtools}

\newtheorem{lemma}{Lemma}[section]
\newtheorem{theorem}{Theorem}[section]
\newtheorem{definition}{Definition}[section]

\newtheorem{assumption}{Assumption}[section]

\newtheorem{remark}{Remark}[section]

\newcommand{\argmin}{\mathop{\arg\!\min}}

\newcommand{\cC}{{\cal C}}

\newcommand{\cO}{{\cal O}}

\newcommand{\mM}{{\bf M}}

\newcommand{\EE}{\mathbb{E}}
\newcommand{\PP}{\mathbb{P}}

\newcommand{\Mod}[1]{\ \mathrm{mod}\ #1}

\title{Moshpit SGD: Communication-Efficient\\ Decentralized Training\\ on Heterogeneous Unreliable Devices}

\author{%
  Max Ryabinin\thanks{Equal contribution. Correspondence to \texttt{mryabinin0@gmail.com}.} \\
  Yandex, Russia\\
  HSE University, Russia\\
  \And
  Eduard Gorbunov\footnotemark[1]\\
  MIPT, Russia\\
  HSE University, Russia\\
  Yandex, Russia\\
  \And
  Vsevolod Plokhotnyuk\\
  Yandex, Russia\\
  HSE University, Russia\\
  \And
  Gennady Pekhimenko\\
  University of Toronto, Canada\\
  Vector Institute, Canada
}

\begin{document}

\maketitle

\begin{abstract}
Training deep neural networks on large datasets can often be accelerated by using multiple compute nodes. 
This approach, known as distributed training, can utilize hundreds of computers via specialized message-passing protocols such as Ring All-Reduce.
However, running these protocols at scale requires reliable high-speed networking that is only available in dedicated clusters.
In contrast, many real-world applications, such as federated learning and cloud-based distributed training, operate on unreliable devices with unstable network bandwidth.
As a result, these applications are restricted to using parameter servers or gossip-based averaging protocols.
In this work, we lift that restriction by proposing Moshpit All-Reduce --- an iterative averaging protocol that exponentially converges to the global average.
We demonstrate the efficiency of our protocol for distributed optimization with strong theoretical guarantees.
The experiments show 1.3x speedup for ResNet-50 training on ImageNet compared to competitive gossip-based strategies and 1.5x speedup when training ALBERT-large on preemptible compute nodes.
\end{abstract}

\section{Introduction}\label{sect:intro}

Many recent influential discoveries in deep learning were enabled by the trend of scaling model and dataset size.
Over the last decade, computer vision has grown from training models with 60 million parameters~\cite{alexnet} on 1.3 million images~\cite{imagenet_cvpr09} to 15 times more parameters~\cite{Kolesnikov2020BigT} and 200 times more training data~\cite{jft-300m}. In natural language processing, the state-of-the-art language models~\cite{gpt3} with 175 billion parameters are trained on over 570GB of texts, and even this does not saturate the model quality~\cite{kaplan2020scaling}.
Training these large models can take years even with a top-of-the-line GPU server~\cite{gpt3costlambda}. As a result, researchers and practitioners often have to run distributed training with multiple machines~\cite{mlperf}.

The dominant approach to distributed deep learning is data-parallel training~\cite{valiant1990bridging}, where each worker processes a fraction of the training batch and then exchanges its gradients with peers. If done naïvely, the gradient exchange step can overload the network as the number of workers increases. To combat this issue, modern distributed training algorithms take advantage of communication-efficient protocols, such as all-reduce~\cite{bandwidth_optimal_allreduce}. These protocols 
allow workers to collectively compute the global average gradient with a constant communication overhead, regardless of the total number of peers.

However, this efficiency makes the protocols more fragile: if any single participant fails or takes too long to process its batch, all other nodes are stalled.
Therefore, scaling all-reduce protocols beyond a couple of servers requires specialized infrastructure with dedicated ultra-high bandwidth networking~\cite{mlperf}.
This kind of infrastructure is notoriously expensive compared to regular
GPU servers or preemptible cloud VMs (see Appendix~\ref{sect:cloud_costs} for details).

Hence, it is tempting to consider distributed training on cheap unreliable instances as a cost-efficient alternative. A similar scenario arises in federated learning~\cite{mcmahan2017communication}, where a single model is trained on heterogeneous devices due to privacy concerns.
In both scenarios, workers use a shared network, where both latency and bandwidth can vary drastically due to interference from other users~\cite{variability_azure}\nocite{variability_aws}. Furthermore, compute nodes are also subject to failure (or preemption) caused by factors beyond the protocol's control.

Running large-scale distributed training in these circumstances requires fault- and latency-tolerant algorithms~\cite{lian2017can,sgpush}. Most of these algorithms replace all-reduce averaging with \textbf{gossip}: each participant periodically downloads the latest parameters from their neighbors in a sparsely connected communication graph and averages the results. The updates gradually propagate through the graph over multiple rounds of averaging.
However, the communication required to perform gossip grows linearly with the number of neighbors. Hence, when scaling to hundreds of peers, decentralized SGD has to keep the communication graph sparse, slowing down the convergence.

In this work, we propose an alternative approach. Instead of relying on a predefined communication graph, participants dynamically organize themselves into groups using a fully decentralized matchmaking algorithm called \textbf{Moshpit All-Reduce}. This strategy allows us to use communication-efficient all-reduce protocols that significantly reduce the network load compared to gossip-based averaging, while still being able to operate in unreliable hardware and network conditions.

Our contributions can be summarized as follows:
\begin{itemize}
    \item We propose {\bf Moshpit All-Reduce} --- a novel decentralized averaging protocol for large-scale training with unreliable communication-constrained devices. According to our analysis, this method has exponential convergence rate independent of network topology and size.
    \item Armed with this averaging protocol, we develop {\bf Moshpit SGD} for distributed optimization. We derive convergence rates for this algorithm and establish its equivalence to Centralized (Local) SGD in terms of iteration complexity under realistic assumptions.
    \item Our experiments demonstrate that Moshpit All-Reduce is significantly more efficient under network latency in realistic conditions. In particular, we train ResNet-50 on ImageNet to 75\% accuracy 1.3 times faster than existing decentralized training algorithms and pretrain ALBERT-large 1.5 times faster on preemptible cloud VMs.\footnote{Implementation and code of experiments are at \href{https://github.com/yandex-research/moshpit-sgd}{\texttt{github.com/yandex-research/moshpit-sgd}}.}
\end{itemize}

\vspace{-10px}%
\section{Related Work}\label{sect:related}
\vspace{-4px}
\subsection{Data parallel training}\label{sect:related_data_parallel}
\vspace{-4px}

The most popular way to accelerate neural network training with multiple devices is data-parallel training~\cite{valiant1990bridging,goyal2017accurate,You2020Large}. On each optimization step, this strategy splits the training batch among participants. Each participant then runs forward and backward passes to obtain gradients of the objective function on their part of the training batch. After that, we can aggregate the gradients from workers and perform an optimization step. There are two main strategies for this aggregation.

Historically, the first solution to gradient aggregation was to use Parameter Server (PS)~\cite{parameter_server_first}: a separate process or a dedicated server that keeps track of model parameters and optimizer statistics. After each round, the PS accumulates the gradients from each worker and updates the model parameters using SGD or any other optimizer, such as Adam~\cite{adam}. Finally, the server distributes the updated model parameters to workers.

This strategy is robust and easy to implement, but it requires the server to regularly download full model gradients from every single worker. As a result, the parameter server can quickly become a bottleneck for large-scale training~\cite{survey_distributed2}\nocite{survey_distributed}. Since the original PS, researchers have proposed several modifications that reduce the communication load: accumulating multiple batches~\cite{localsgd_first}, compression~\cite{lin2018deep,pmlr-v97-koloskova19a}, server sharding~\cite{sharded_ps_first,byteps}. A more detailed overview is given in Appendix~\ref{sect:post_related}.

In turn, many practical distributed training systems have instead switched to averaging with All-Reduce ~\cite{goyal2017accurate,mikami2019massively,shoeybi2019megatron,You2020Large}. This name refers to a collection of protocols originally developed for HPC applications. Workers can follow these protocols to collectively compute the average\footnote{All-Reduce works with any commutative associative operation, such as min, max, or product.} gradient more efficiently than with a central server.

\subsection{Communication-efficient All-Reduce}\label{sect:related_allreduce}
There are several all-reduce protocols optimized for different network topologies. The simplest one is known as Butterfly All-Reduce~\cite{bandwidth_optimal_allreduce}. Each of $N$ participants splits its local vector into $N$ chunks. Then, $i$-th worker aggregates $i$-th chunk of data from all peers and sends back the averaged chunk.

\begin{figure}[h!]
    \centering
    \includegraphics[width=0.65\linewidth]{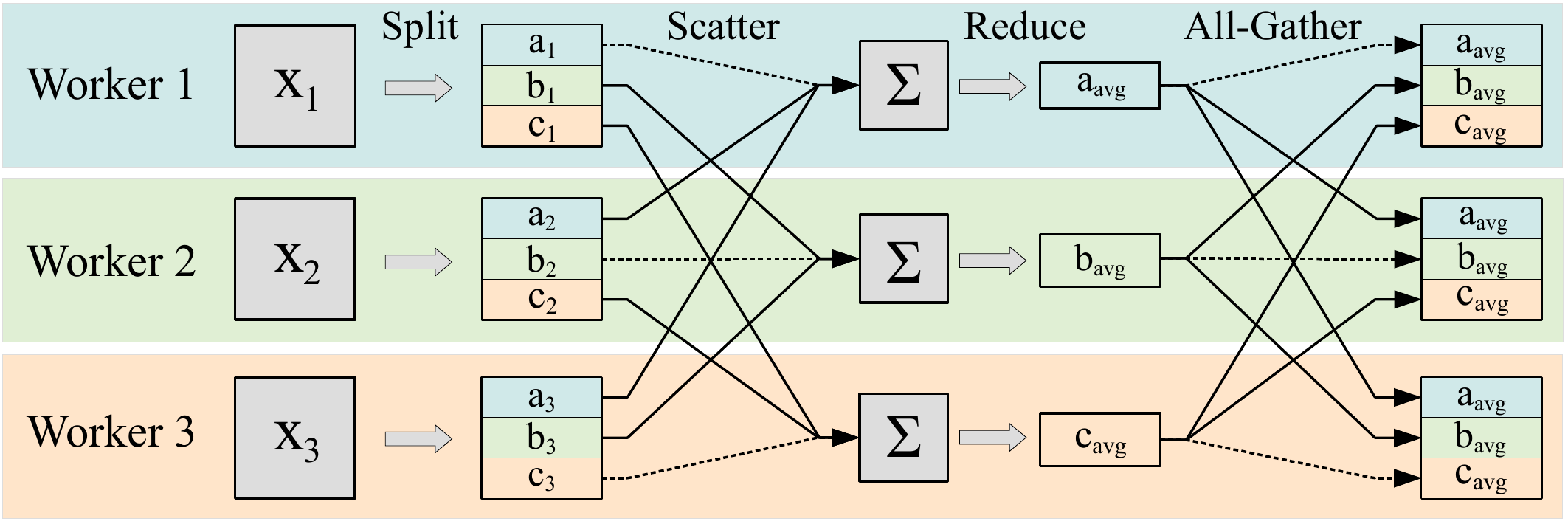}
    \caption{A schematic illustration of Butterfly All-Reduce.}
    \label{fig:butterfly_allreduce}
\end{figure}

As long as the vector size $s$ is greater than $N$, this protocol uses $\cO\left(s \times \frac{N - 1}{N}\right)$ total bandwidth on each worker. However, it requires all-to-all communication, which is not always practical for the HPC infrastructure due to network contention~\cite{bandwidth_optimal_allreduce}. As a result, real-world systems typically use Ring or Tree All-Reduce, where each worker only communicates with a small subset of its peers.

These protocols enable highly efficient and scalable averaging with $\cO(1)$ or $\cO(\log N)$ total communication per worker, but they also share a common drawback: they cannot tolerate node failures or network instability. If any single participant fails to execute its part or takes long to respond, this paralyzes all other workers.

\subsection{Distributed training in unstable conditions}\label{sect:related_unreliable}
Some distributed training applications must deal with unstable network bandwidth and/or unreliable workers. This issue is most prevalent in federated learning~\cite{mcmahan2017communication,secure_aggregation,federatedlearningatscale}. When dealing with privacy-sensitive data distributed across multiple actors, such as hospital servers~\cite{fed_intel,fed_nvidia} or mobile phones~\cite{fed_google1,fed_google2}, one must train the model using whichever hardware and network available to those actors.

Another important motivational factor is cost: HPC-grade infrastructure can be prohibitively expensive, pushing researchers and practitioners towards commodity servers or preemptible cloud VMs that are significantly cheaper (see Appendix~\ref{sect:cloud_costs}). Another solution is to use volunteer computing~\cite{volunteer_dl_async, learning_at_home} with abundant, but even less reliable, compute resources.

Training under these conditions requires specialized strategies. At a small scale, one can deploy one or a few reliable parameter servers to aggregate the updates from workers. This strategy can tolerate individual node failures~\cite{proteus}, but scales poorly due to the reasons discussed in Section~\ref{sect:related_data_parallel}.

\subsection{Decentralized training}\label{sect:related_decentralized_training}
If there are too many participants for PS, it can be advantageous to use decentralized SGD via \textbf{gossip-based} averaging \cite{boyd2006randomized,tsitsiklis1984problems,lian2017can}. In this scenario, participants form a sparse graph: each worker periodically downloads parameters from its neighbors and mixes them with local parameters.

In essence, gossip-based averaging removes the communication bottlenecks of PS at the cost of using different local parameters on each peer. That said, gossip-based optimization algorithms can match, and sometimes even outperform, their centralized counterparts in terms of training speed~\cite{scaman2017optimal,scaman2018optimal,scaman2019optimal,lian2017can,assran2019stochastic}. However, the convergence properties of gossip averaging and gossip-based optimization methods significantly depend on the communication graph through the spectral properties of the mixing matrix~\cite{xiao2004fast,scaman2019optimal} or the Laplacian matrix of the network~\cite{merris1994laplacian,uribe2020dual}. 

Consequently, as the number of peers increases, gossip-based averaging has to either increase the number of neighbors (hence more communication) or accept slower convergence speed. Because of this, gossip is less communication-efficient than all-reduce algorithms reviewed in Section~\ref{sect:related_allreduce}. However, gossip-based algorithms are more robust to changes, which makes them applicable to time-varying networks~\cite{nedic2014distributed,nedic2016stochastic,nedic2018network,rogozin2019projected} and federated learning~\cite{ram2009asynchronous,yan2012distributed,yuan2016convergence}.

\section{Moshpit SGD}\label{sect:method}

Large-scale training with unreliable participants requires a protocol that is both communication-efficient and fault-tolerant. Unfortunately, existing methods have only  provide one of these properties. To better address our conditions, we propose Moshpit All-Reduce --- a fully decentralized averaging protocol that combines the efficiency of all-reduce and the fault tolerance of gossip-based averaging. 

The rest of this section is organized as follows:
\begin{itemize}
    \item Section~\ref{sect:method_algorithm} describes the protocol and proves its correctness and communication efficiency;
    \item Section~\ref{sect:method_convergence} provides the analysis of the protocol and proves exponential convergence rate for averaging and the rate matching the one of centralized Local-SGD for optimization;
    \item Section~\ref{sect:method_implementation_details} contains implementation details for training with heterogeneous compute nodes.
\end{itemize}

\subsection{Moshpit All-Reduce}
\label{sect:method_algorithm}

The core idea of Moshpit All-Reduce is that workers perform averaging in small independent groups. That way, a single failed participant would only affect his current group. In turn, the composition of each group should be chosen dynamically to converge in the least number of steps.
Ideally, if there are 9 peers with local parameters $\theta$, we can average them in 2 rounds, as demonstrated in Figure~\ref{fig:square_allreduce}:

\vspace{-4pt}
\noindent
\begin{minipage}{0.45\textwidth}
\centering
\includegraphics[width=\textwidth]{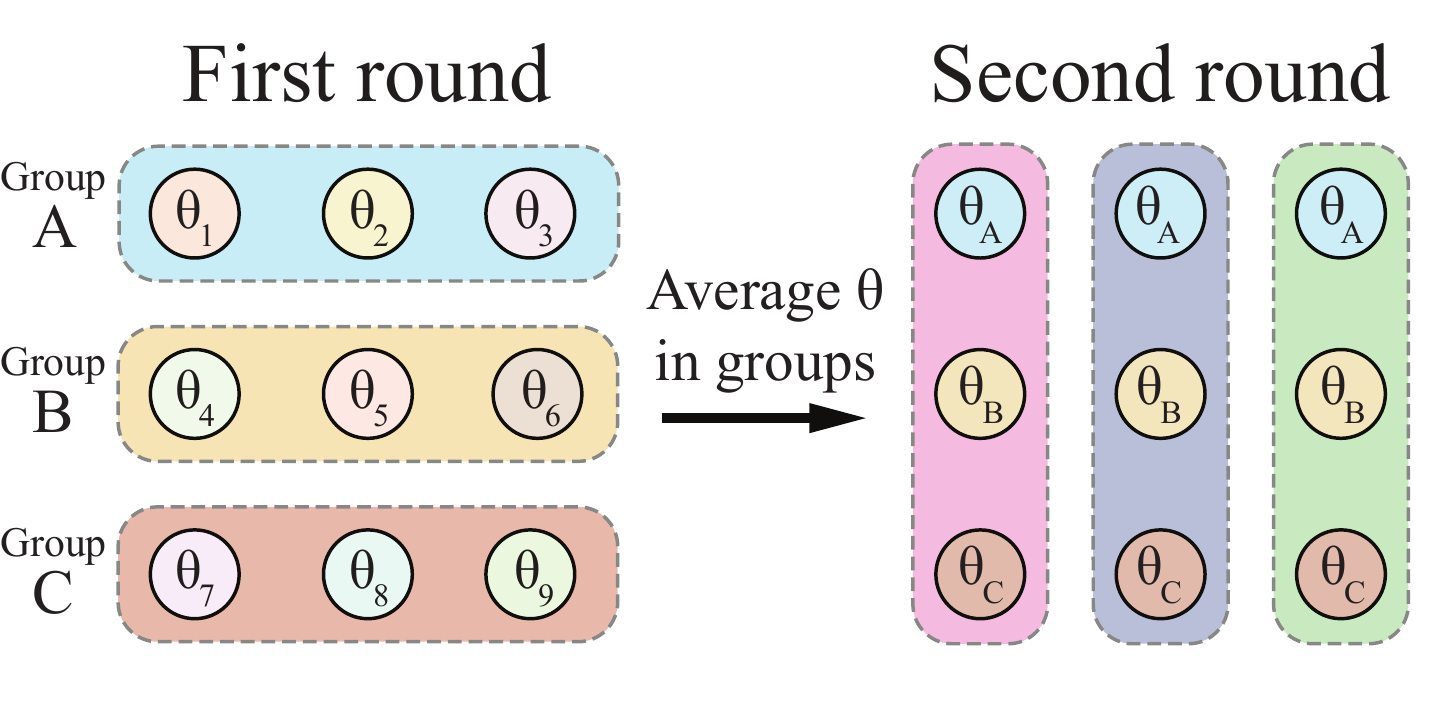}
\captionof{figure}{Example averaging order for 9 peers in 2 rounds. On each round, peers are split into 3 groups that run All-Reduce in parallel.}
\label{fig:square_allreduce}
\end{minipage}
\begin{minipage}{0.55\textwidth}
\begin{algorithm}[H]
\caption{Moshpit All-Reduce (for $i$-th peer)}
   \label{alg:moshpit}
\begin{algorithmic}[H]
   \STATE {\bfseries Input:} parameters $\{\theta_j\}_{j=1}^N$, number of peers $N$, $d$, $M$, number of iterations $T$, peer index $i$

   $\theta_{i}^0 := \theta_i$
   
   $C^0_i :=\texttt{get\_initial\_index(i)}$
   
   \FOR{$t \in 1 \dots T$}
     \STATE $\texttt{DHT}[C^{t-1}_i, t].\texttt{add}(\texttt{address}_i)$
     
     \STATE \texttt{Matchmaking()} // wait for peers to assemble
     
     \STATE $\texttt{peers}_t := \texttt{DHT}.\texttt{get}([C^{t-1}_i, t])$ 
     
     \STATE $\theta_{i}^t, c^t_i := \texttt{AllReduce}(\theta_{i}^{t - 1}, \texttt{peers}_t)$
     
     \STATE $C^t_i := (C^{t-1}_i\texttt{[1:]}, c^t_i)$  // same as eq. (1)
   \ENDFOR
   \STATE {\bfseries Return} $\theta^T_i$
\end{algorithmic}
\end{algorithm}
\end{minipage}

To achieve this in a decentralized system, we use Distributed Hash Tables (DHT) --- a decentralized key-value storage; \autoref{sect:post_related} contains its more detailed description. On each averaging round:
\begin{itemize}
    \item Each worker computes its group key $C_i$;
    \item Workers add their network addresses to the DHT key corresponding to $C_i$;
    \item Each worker can now fetch a full list of peers that have the same $C_i$ and run All-Reduce with those peers.
\end{itemize}

Unfortunately, the averaging structure from Figure~\ref{fig:square_allreduce} is impossible to maintain when participants are constantly joining, leaving, and failing. However, we can achieve equivalent results without global structure using a simple rule: \textit{if two peers were in the same group in round $t$, they must choose different groups in round $t {+} 1$.}

A natural way to enforce this rule is to take advantage of the chunk indices from Butterfly All-Reduce (see Figure~\ref{fig:butterfly_allreduce}). Recall that each worker accumulates a \textit{unique} chunk of parameters defined by an index $c_i$. By setting $C_i := c_i$, we can guarantee that any workers that were in the same group at a round $t$ will have different group indices in round $t {+} 1$.

This averaging scheme can be generalized to more than two dimensions in order to fit a larger number of peers or reduce the group size. For a $d$-dimensional hypercube, nodes should find groups of peers that they have not communicated with during $d {-} 1$ previous rounds. To that end, we define $C_i$ as tuples containing chunk indices from $d{-}1$ previous rounds ($t$ denotes the communication round):
\vspace{-2pt}
\begin{equation}
    C^t_i := (c^{t-d+1}_i, c^{t-d+2}_i, \ldots, c^{t }_i).
    \label{eq:group}
\end{equation}

The above intuition can be formalized with Algorithm \ref{alg:moshpit}.
Here, $N$ peers form a virtual $d$-dimensional grid with $M$ peers per row and average their parameters $\theta_i$ over $T$ rounds. $\texttt{DHT}[\cdot]$ is a shortcut for using the DHT to add or retrieve values for a given key. The  \texttt{Matchmaking} step corresponds to the decentralized matchmaking procedure that organizes active workers with the same index into groups, described in detail in ~\autoref{sect:matchmaking}. In turn, \texttt{AllReduce} denotes running all-reduce to compute the average $\theta$ in a given group. The \texttt{get\_initial\_index} function takes the peer index $i$ and returns $d{-}1$ integers
in range $[0, M)$ such that the size of initial groups does not exceed $M$.
This way, the groups formed on subsequent rounds will also have at most $M$ participants. One possible strategy is:

\vspace{-8pt}
\begin{equation}
    \texttt{get\_initial\_index}(i) = 
    \begin{pmatrix}
           \lfloor i / M^{d{-}1} \rfloor \Mod M \\
         \end{pmatrix}_{j\in \{1,\ \ldots,\ d\}}
    \label{eq:get_initial_index}
\end{equation}

If $N {=} M^d$ and there are no node/network failures, Algorithm~\ref{alg:moshpit} is equivalent to Torus All-Reduce~\cite{torus_allreduce}, achieving the exact average after $d$ rounds of communication (see Appendix~\ref{sect:equiv_to_torus}).
However, our typical use case is far from this perfect scenario; for example, some groups can have less than $M$ members. Furthermore, a peer might fail during all-reduce, causing its groupmates to skip a round of averaging. 
Still, Moshpit All-Reduce is applicable even in these conditions:
\begin{theorem}[Correctness]\label{thm:quality_of_avg_deterministic_vectors_0}
If all workers have a non-zero probability of successfully running a communication round and the order of $\texttt{peers}_t$ is random, then all local vectors $\theta^t_i$ converge to the global average with probability 1:
\vspace{-4px}
\begin{equation}
    \forall i, \Big|\Big|\theta^t_i - \frac1N \sum_i \theta^0_i\Big|\Big|^2_2 \xrightarrow[t\to\infty]{} 0.
\end{equation}
\end{theorem}\vspace{-16pt}
\begin{proof}[Proof (sketch, complete in Appendix~\ref{sect:correctness_proof})]
Running all-reduce with a subset of peers preserves the invariant $\frac1N \sum_i \theta^t_i=\frac1N \sum_i \theta^{t-1}_i$ and reduces the deviation of $\theta^t_i$ from the overall average.
\end{proof}\vspace{-6pt}

\textbf{Complexity.} The matchmaking protocol is implemented over Kademlia DHT~\cite{kademlia}, meaning that each read and write operation needs at most $\cO(\log N)$ requests and $\cO(M)$ bandwidth to load $\texttt{peers}_t$.

After the matchmaking is over, each group runs a single all-reduce round to compute the average. In principle, Moshpit Averaging can use any general-purpose all-reduce protocol. We opted for a butterfly-like version (Figure~\ref{fig:butterfly_allreduce}), as it is simpler than Ring All-Reduce while still being communication-efficient. The communication complexity of this algorithm is $\cO\left(\max(s, M) \times \frac{M - 1}{M}\right)$, where $s$ is the size of vector $\theta$. Thus, the total time complexity of Algorithm \ref{alg:moshpit} becomes:
\begin{equation}
    \cO\left(T \times \left[\log_2{N} + M + \max(s, M) \times {\frac{M - 1}{M}}\right]\right).
\end{equation}
This compares favorably to gossip, where network load grows linearly with the number of neighbors.

\vspace{-2pt}
\subsection{Convergence analysis}\label{sect:method_convergence}
\subsubsection{Mixing properties of Moshpit Averaging}\label{sect:theory_about_avg}
As stated in the previous section, Moshpit All-Reduce computes the exact average when $N = M^d$, which cannot be guaranteed in practice. Therefore, additional analysis is needed to establish how quickly Moshpit Averaging approximates the actual average of $N$ vectors stored on peers.

In the following theorem, we provide such analysis for a simplified version of Moshpit Averaging. One can find the full proof in Appendix~\ref{sec:proof_quality_of_avg_deterministic_vectors}.
\begin{theorem}\label{thm:quality_of_avg_deterministic_vectors}
    Consider a modification of Moshpit All-Reduce that works as follows: at each iteration $k\ge 1$, 1) peers are randomly split in $r$ disjoint groups of sizes $M_1^k,\ldots, M_r^k$ in such a way that $\sum_{i=1}^r M_i^k = N$ and $M_i^k \ge 1$ for all $i = 1,\ldots,r$ and 2) peers from each group compute their group average via All-Reduce. Let $\theta_1,\ldots,\theta_N$ be the input vectors of this procedure and $\theta_1^T,\ldots,\theta_N^T$ be the outputs after $T$ iterations. Also, let $\overline{\theta} = \frac{1}{N}\sum_{i=1}^N\theta_i$ Then,
    \begin{equation}
         \hspace{-0.1cm}\EE\left[\frac{1}{N}\sum\limits_{i=1}^N\|\theta_i^T - \overline{\theta}\|^2\right]= \left(\frac{r-1}{N} + \frac{r}{N^2}\right)^T\frac{1}{N}\sum\limits_{i=1}^N\|\theta_i - \overline{\theta}\|^2. \label{eq:determ_quality_of_avg}
    \end{equation}
\end{theorem}

\begin{algorithm}[h]
   \caption{Moshpit SGD}
   \label{alg:moshpit_local_sgd}
\begin{algorithmic}[1]
   \STATE {\bfseries Input:} starting point $\theta^0$, learning rate $\gamma > 0$, communication period $\tau \ge 1$
   \FOR{$k = 0, 1, \ldots$}
   \FOR{each peer $i\in P_{k+1}$ in parallel}
   \STATE Compute the stochastic gradient $g_i^k$ at the current point $\theta_i^k$
   \IF{$k+1 \mod \tau = 0$}
   \STATE $\theta_i^{k+1} = \text{Moshpit All-Reduce}_{j\in P_{k+1}}(\theta_j^k - \gamma g_j^k)$ for $i$-th peer (Algorithm~\ref{alg:moshpit})
   \ELSE
   \STATE $\theta_i^{k+1} = \theta_i^k - \gamma g_i^k$
   \ENDIF
   \ENDFOR
   \ENDFOR
\end{algorithmic}
\end{algorithm}\setlength{\textfloatsep}{12pt}

In particular, this result implies that even if workers are randomly split into pairs at each iteration, the simplified version of Moshpit Averaging makes the average distortion (the left-hand side of Equation~\ref{eq:determ_quality_of_avg}) less than $\varepsilon$ in expectation after $\cO\left(\log(\nicefrac{1}{\varepsilon})\right)$ iterations. That is, this algorithm finds $\varepsilon$-accurate average on each node with the rate that \textit{does not} depend on the spectral properties of the communication graph like gossip and its variants (see Section~\ref{sect:related_decentralized_training} and Appendix~\ref{sect:post_related_gossip}). Since Moshpit Averaging prevents two peers from participating in the same groups during successive iterations, the actual algorithm should find $\varepsilon$-accurate averages on participating peers even faster than Equation~\ref{eq:determ_quality_of_avg} predicts. Moreover, in Appendix~\ref{sec:proof_quality_of_avg_deterministic_vectors} we explain how this result can be generalized to the case when $\{M_i^k\}_{i=1}^N$ and $r$ depends on $k$ or even is random. In Appendix~\ref{sec:mix_rand_proof}, we also provide the guarantees measuring how fast Algorithm~\ref{alg:moshpit} reduces the variance when averaging random vectors.

\vspace{-4pt}
\subsubsection{Moshpit SGD}\label{sect:optim_theory}
We consider a classical distributed optimization problem
\vspace{-6pt}
\begin{equation}
    \min\limits_{\theta\in\R^n}\left\{f(\theta) = \frac{1}{N}\sum\limits_{i=1}^N f_i(\theta)\right\}, \label{eq:main_problem}
\end{equation}
\vspace{-6pt}
where $N$ is the number of workers and worker $i$ has access only to the function $f_i$.

We propose a new algorithm called Moshpit SGD to solve this problem (see Algorithm~\ref{alg:moshpit_local_sgd}). In this algorithm, workers perform independent local SGD steps and periodically synchronize their parameters $\theta_i^k$ with other peers using Moshpit All-Reduce. Moreover, we define the indices of participating nodes at iteration $k$ as $P_{k+1}$ ($P_0 = \{1,\ldots,N\}$) allowing peers to vanish.

First of all, we list the key assumptions that we use in the convergence analysis of Moshpit SGD.
\begin{assumption}[Bounded variance]\label{as:bounded_var}
    We assume that for all $k\ge 0$ and $i=1,\ldots, N$ stochastic gradients $g_i^k$ satisfy $\EE\left[g_i^k\mid \theta_i^k\right] = \nabla f_i(\theta_i^k)$ and
    \begin{eqnarray}
        \EE\left[\|g_i^k - \nabla f_i(\theta_i^k)\|^2\mid \theta_i^k\right] &\le& \sigma^2.\label{eq:bounded_variance}
    \end{eqnarray}
\end{assumption}\vspace{-6px}
This assumption is classical in the stochastic optimization literature \cite{nemirovski2009robust,ghadimi2013stochastic}. We notice that our analysis can be generalized to the settings when the stochastic gradients satisfy less restrictive assumptions such as expected smoothness \cite{gower2019sgd} or have more sophisticated structure similar to \cite{karimireddy2020scaffold} using the theoretical framework from \cite{gorbunov2020local}.

The following assumption controls the averaging properties and the effect of the peers' vanishing.
\begin{assumption}[Averaging quality \& peers' vanishing]\label{as:averaging_quality}
    We assume that the vanishing of peers does not change the global average of the iterates of Moshpit SGD too much, i.e., $P_{k+1}\subseteq P_{k}$ and $|P_k| \ge N_{\min}$ for all $k\ge 0$, $|P_{a\tau}| \le 2|P_{a(\tau+1)}|$ for all non-negative integers $a\ge 0$, and there exist such $\widetilde{\theta}\in \R^n$ and a sequence of non-negative numbers $\{\Delta_{pv}^k\}_{k\ge 0}$ that $\forall k \ge 0$
    \begin{align}
        \EE\left[\langle\theta^{k+1} - \widehat{\theta}^{k+1}, \theta^{k+1}+\widehat{\theta}^{k+1} - 2\widetilde\theta\rangle\right] \!\le\! \Delta_{pv}^k\label{eq:stationary_avg_almost}&,f\text{ convex;}\\
        \EE\!\left[\langle\nabla f(\theta^k), \theta^{k+1}-\widehat{\theta}^{k+1}\rangle + L\|\widehat{\theta}^{k+1} - \theta^{k+1}\|^2\right] \!\le\! \Delta_{pv}^k\label{eq:stationary_avg_almost_2}&,f\text{ non-convex, $L$-smooth, (Def.~\ref{def:L_smoothness})}
    \end{align}
    where $N_k = |P_k|$, $\theta^{k+1} = \frac{1}{N_{k+1}}\sum_{i\in P_{k+1}}\theta_i^{k+1}$, and $\widehat \theta^{k+1} = \frac{1}{N_{k}}\sum_{i\in P_{k}}(\theta_i^{k}-\gamma g_i^k)$ for $k\ge 0$. 
    
    Moreover, we assume that for some $\delta_{aq} \ge 0$ and for all non-negative integers $a\ge 0$,
    \begin{eqnarray}
        \EE\left[\frac{1}{N_{a\tau}}\sum\limits_{i\in P_{a\tau}}\|\theta_i^{a\tau} - \theta^{a\tau}\|^2\right] &\le& \gamma^2\delta_{aq}^2.\label{eq:quality_of_avg}
    \end{eqnarray}
\end{assumption}
If $P_k = P_{k+1} = \{1,\ldots,N\}$ for all $k\ge 0$, i.e., peers do not vanish, then $\theta^{k} = \widehat{\theta}^{k}$ and properties (\ref{eq:stationary_avg_almost}, \ref{eq:stationary_avg_almost_2}) hold with $\Delta_{pv}^k \equiv 0$ for all $k\ge 0$. Moreover, according to the mixing properties of Moshpit Averaging established in Theorem~\ref{thm:quality_of_avg_deterministic_vectors}, inequality \ref{eq:quality_of_avg} holds after $\cO\left(\log\left(\nicefrac{1}{\gamma^2\delta_{aq}^2}\right)\right)$ iterations of Algorithm~\ref{alg:moshpit}. Therefore, the assumption above is natural and well-motivated.

Under these assumptions, we derive the convergence rates both for convex and non-convex problems. The full statements and complete proofs are deferred to Appendix~\ref{sect:missing_proofs_local_sgd}.
\begin{theorem}[Convex case]\label{thm:cvx_convergence}
    Let $f_1 = \ldots = f_N = f$, function $f$ be $\mu$-strongly convex (Def.~\ref{def:str_cvx}) and $L$-smooth (see Def.~\ref{def:L_smoothness}), and Assumptions~\ref{as:bounded_var}~and~\ref{as:averaging_quality} hold with $\Delta_{pv}^k = \delta_{pv,1}\gamma\mu\EE[\|\theta^k-\theta^*\|^2] + \gamma^2\delta_{pv,2}^2$ and $\widetilde{\theta} = \theta^*$, where $\theta^* \in \argmin_{\theta\in\R^n} f(\theta)$ and $\delta_{pv,1}\in [0,1)$, $\delta_{pv,2}\ge 0$. Then there exists a choice of $\gamma$ such that $\EE\left[f(\overline{\theta}^K) - f(\theta^*)\right]\le \varepsilon$ after $K$ iterations of Moshpit SGD, where $K$ equals
    \vspace{-2pt}
    \begin{align}
        \widetilde{\cO}\!\left(\!\frac{L}{(1\!-\!\delta_{pv,1})\mu}\! +\! \frac{\delta_{pv,2}^2\!+\!\nicefrac{\sigma^2}{N_{\min}}}{(1-\delta_{pv,1})\mu\varepsilon}\! +\! \sqrt{\frac{L((\tau\!-\!1)\sigma^2\!+\!\delta_{aq}^2)}{(1\!-\!\delta_{pv,1})^2\mu^2\varepsilon}}\!\right)&,\ \mu>0;\\
        \cO\!\left(\!\frac{LR_0^2}{\varepsilon}\!+\! \frac{R_0^2(\delta_{pv,2}^2\!+\!\nicefrac{\sigma^2}{N_{\min}})}{\varepsilon^2}\!+\! \frac{R_0^2\!\sqrt{L\!(\!(\tau\!-\!1)\!\sigma^2\!+\!\delta_{aq}^2)}}{\varepsilon^{\nicefrac{3}{2}}}\!\right)&,\ \mu=0,
    \end{align}
    where $\overline{\theta}^K = \frac{1}{W_K}\sum\limits_{k=0}^K\frac{1}{N_k}\sum\limits_{i\in P_k} w_k \theta_i^k$, $w_k = (1-\gamma\mu)^{-(k+1)}$, $W_K = \sum_{k=0}^Kw_k$, $R_0 = \|\theta^0 - \theta^*\|$ and $\widetilde{\cO}(\cdot)$ hides constant and $\log(\nicefrac{1}{\varepsilon})$ factors.
\end{theorem}
That is, if $\delta_{pv,1} \le \nicefrac{1}{2}$, $N_{\min} = \Omega(N)$, $\delta_{pv,2}^2 = \cO(\nicefrac{\sigma^2}{N_{\min}})$, and $\delta_{aq}^2 = \cO((\tau-1)\sigma^2)$, then Moshpit SGD has the same iteration complexity as Local-SGD in the homogeneous case \cite{khaled2020tighter,woodworth2020local}. However, the averaging steps of Moshpit SGD are much faster than those of the parameter-server architecture when the number of peers is large. Also, unlike the state-of-the-art convergence guarantees for Decentralized Local-SGD \cite{koloskova2020unified}, our bounds do not depend on the spectral properties of the communication graph (see Appendix~\ref{sect:post_related_gossip} for the details).

\begin{theorem}[Non-convex case]\label{thm:non_cvx_convergence}
    Let $f_1 = \ldots = f_N = f$, function $f$ be $L$-smooth and bounded from below by $f_*$, and Assumptions~\ref{as:bounded_var}~and~\ref{as:averaging_quality} hold with $\Delta_{pv}^k = \delta_{pv,1}\gamma\EE[\|\nabla f(\theta^k)\|^2] + L\gamma^2\delta_{pv,2}^2$, $\delta_{pv,1}\in [0,\nicefrac{1}{2})$, $\delta_{pv,2}\ge 0$. Then there exists such choice of $\gamma$ that $\EE\left[\|\nabla f(\theta_{\text{rand}}^K)\|^2\right]\le \varepsilon^2$ after $K$ iterations of Moshpit SGD, where $K$ equals
    {\begin{eqnarray*}
        \cO\Bigg(\tfrac{L\Delta_0}{(\!1\!-\!2\delta_{pv,1}\!)^2\varepsilon^2}\!\Bigg[\!1\! +\!\tau\sqrt{1\!-\!2\delta_{pv,1}}\! +\! \tfrac{\delta_{pv,2}^2 + \nicefrac{\sigma^2}{N_{\min}}}{\varepsilon^2}\!+\! \tfrac{\sqrt{(1-2\delta_{pv,1})(\delta_{aq}^2+(\tau-1)\sigma^2)}}{\varepsilon}\!\Bigg]\!\Bigg),
    \end{eqnarray*}}
    $\Delta_0 = f(\theta^0) - f(\theta^*)$ and $\theta_{\text{rand}}^K$ is chosen uniformly from $\{\theta^0,\theta^1,\ldots,\theta^{K-1}\}$ defined in As.~\ref{as:averaging_quality}.
\end{theorem}
Again, if $\delta_{pv,1} \le \nicefrac{1}{3}$, $N_{\min} = \Omega(N)$, $\delta_{pv,2}^2 = \cO(\nicefrac{\sigma^2}{N_{\min}})$, and $\delta_{aq}^2 = \cO((\tau-1)\sigma^2)$, then the above theorem recovers the state-of-the-art results in the non-convex case for Local-SGD \cite{li2019communication,koloskova2020unified}. 

\subsection{Implementation details}
\label{sect:method_implementation_details}

Training on heterogeneous unreliable hardware also poses a number of engineering challenges. The most obvious one is that the system must be able to recover from node failures. To address this challenge, we use a fully decentralized infrastructure where all information is replicated in a Distributed Hash Table; see Appendix~\ref{sect:related_dht} for details. When a new worker joins midway through training, it can download the latest model parameters and metadata from any other peer (see \autoref{sect:load_state_from_peers}). Another challenge arises when devices in a group have uneven network bandwidth. In that case, we dynamically adjust the communication load of each peer to avoid being bottlenecked. More information on this procedure can be found in \autoref{sect:load_balancing}.

\vspace{-10pt}
\section{Experiments}\label{sect:experiments}
\vspace{-2pt}
In this section, we
conduct empirical evaluation of the proposed averaging protocol and its corresponding optimization algorithm. 
First, we check the theoretical properties of Moshpit All-Reduce in a controlled setup (Section~\ref{sect:experiments_averaging}). Then, we compare Moshpit SGD with other distributed methods on practical tasks of image classification and masked language model pretraining (Sections~\ref{sect:experiments_vision} and~\ref{sect:experiments_nlp}).

\vspace{-4pt}
\subsection{Decentralized averaging}
\label{sect:experiments_averaging}
In this series of experiments, we aim to empirically verify the convergence and fault tolerance properties proven in Section~\ref{sect:method_convergence}.
To measure this in a controlled setting, we create peers with parameters that are scalar values drawn from the standard Gaussian distribution. We study the convergence of different distributed methods with respect to the number of workers $N$ and their individual failure rate for a single iteration of averaging $p$  (failed peers return in the next round). 

We compare Moshpit Averaging with the following algorithms from prior work: All-Reduce (with restarts in case of node failures), Gossip, PushSum (equivalent to the method described in~\cite{sgpush}). Also, we provide the results of averaging in random groups as a simpler version of our approach. However, the implementation of group averaging maintains approximately the same group size across all iterations: this property might be hard to achieve in a decentralized setting, and as a result, the estimate of this method's performance should be considered highly optimistic.

We report the average squared difference between the worker parameters and the actual average of all values; the results are averaged across 100 restarts from different random initializations.
We compare the convergence for 512--1024 peers and consider failure probabilities ranging from 0 to 0.01. For Moshpit Averaging and random group averaging, we use groups of size 32, which corresponds to $M=32$ and $d=2$ for Algorithm~\ref{alg:moshpit}.

\vspace{-4pt}
\begin{figure}[h]
\noindent
\centering
\includegraphics[width=\textwidth]{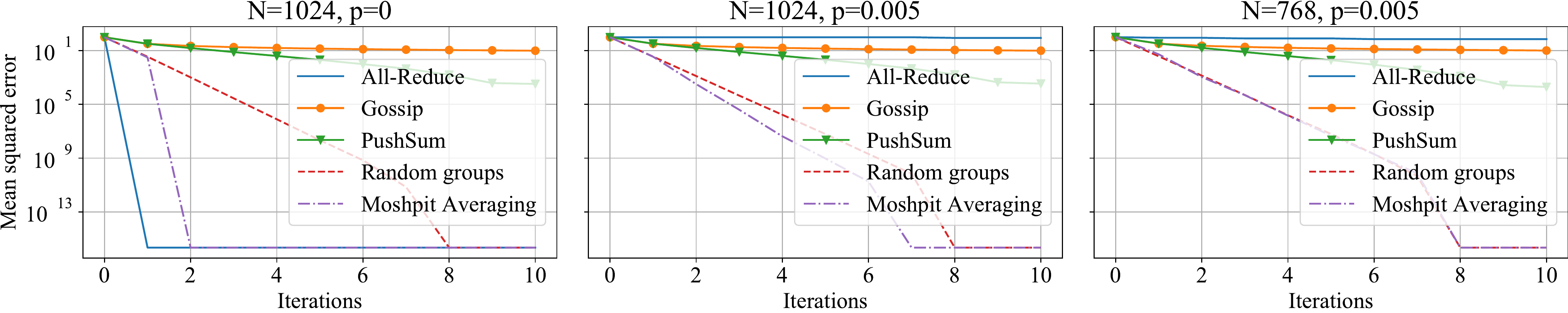}
\caption{Convergence of averaging algorithms in different configurations.}
\label{fig:averaging}
\end{figure}
\vspace{-8pt}

Figure~\ref{fig:averaging} displays the results of experiments for several combinations of $N$ and $p$; the complete results with additional grid configurations are available in Appendix~\ref{sect:extra_averaging}. We make several key observations: 
\begin{enumerate}[leftmargin=*]
    \vspace{-2pt}\item When the failure rate of each peer is zero, standard All-Reduce predictably computes the average faster than all other methods. However, as soon as $p$ reaches a value of at least 0.005, the number of retries needed for the success becomes prohibitively high.
    \vspace{-2pt}\item Previous decentralized averaging methods, such as Gossip or PushSum, require significantly more iterations for convergence to the global average than Moshpit All-Reduce, likely due to the structure of their communication graphs.
    \vspace{-2pt}\item As discussed in Section~\ref{sect:method_algorithm}, when the total number of peers is equal to the grid capacity and there are no failures, Moshpit All-Reduce matches the result of regular All-Reduce with the number of steps equal to the number of grid dimensions (2 in this case).
    \vspace{-2pt}\item Averaging in random groups can perform comparably to Moshpit Averaging when the number of peers is less than half of the grid capacity. The reason for this behavior is that when the workers do not fully occupy the grid, the group sizes are no longer guaranteed to be equal across groups and across iterations. In the worst case, there can be groups of only one peer for certain grid coordinates, which may significantly affect the convergence. However, as the grid utilization grows, Moshpit Averaging starts to outperform random group averaging. Moreover, even if we use 512 peers, arranging them in a proper 8x8x8 grid leads to faster convergence.
\end{enumerate}

\pagebreak[4]

\subsection{ImageNet training}\label{sect:experiments_vision}
Here, we evaluate the performance of Moshpit SGD in distributed training. More specifically, we train ResNet-50~\cite{resnet} on the ILSVRC~\cite{imagenet_cvpr09} dataset, following the training protocol of~\cite{goyal2017accurate}. Trainers use SGD with Nesterov momentum with a batch size of 256 and 32-bit precision regardless of the GPU type\footnote{For GPUs that cannot fit this into memory, we accumulate gradients over 2 batches of 128 examples.}. We evaluate the following training strategies:
\begin{itemize}[leftmargin=*]\vspace{-2px}
    \item \textbf{All-Reduce SGD (AR-SGD)} --- traditional distributed training with all-reduce gradient averaging;
    \item \textbf{Asynchronous Decentralized Parallel SGD (AD-PSGD)} --- parallel SGD that runs gossip communication in a cycle: each worker averages parameters with 2 neighbors~\cite{ad_psgd}. Communication rounds are overlapped with computation;
    \item \textbf{Stochastic Gradient Push (SGP)} --- a more advanced algorithm with an exponential communication graph and push-based communication~\cite{sgpush};
    \item \textbf{Moshpit SGD} --- similar to \textbf{SGP}, but with 1 round of Moshpit Averaging instead of PushSum.
\end{itemize}\vspace{-2px}

We report top-1 validation accuracy as a function of training time in two experimental setups:
\begin{itemize}[leftmargin=*]\vspace{-4px}
    \item \textbf{Homogeneous}: 16 servers with a single Tesla V100-PCIe GPU, 6 CPU cores, and 64GB RAM.
    \item \textbf{Heterogeneous}: a total of 81 GPUs (V100, 1080Ti, and P40) across 64 servers and workstations.\footnote{We provide a detailed configuration in Appendix~\ref{sect:detailed_setup}.}
\end{itemize}\vspace{-4px}

All servers and workstations communicate over the network with 1Gb/s Ethernet (non-dedicated symmetric bandwidth). The machines are located in two data centers and one office within 300 km of one another. The communication latency is 1--6ms depending on the location. To simulate shared usage, at the beginning of each communication round we inject additional latency sampled from the exponential distribution~\cite{sukhov2016generating} with the mean of 100ms.

For Moshpit SGD, we use a two-dimensional ``grid'' with 4 and 8 groups for homogeneous and heterogeneous setups respectively. For AD-PSGD, we attempt to compensate for slow convergence by training for 60 more epochs without changing the learning rate schedule. Finally, we only report AR-SGD in the first setup, as it is unsuitable for heterogeneous hardware.%

The results in Figure~\ref{fig:all} (Left) demonstrate that the two most efficient strategies for our setting are Moshpit SGD and SGP. In the \textbf{homogeneous} setup, Moshpit is only slightly more efficient than SGP, likely due to higher efficiency of all-reduce. This advantage increases to over 30\% for the \textbf{heterogeneous} setup with 64 servers. In turn, AR-SGD demonstrates the best performance per iteration, but its training time is by far the longest due to network latency ($1.5{\times}$ of Moshpit SGD). Finally, AD-PSGD predictably shows the best throughput (time per epoch), but achieves lower accuracy even after training for 150 epochs. We report results for smaller setups in Appendix~\ref{sect:extra_classification}. %

\subsection{Masked Language Model training}
\label{sect:experiments_nlp}
Finally, we evaluate Moshpit All-Reduce training performance in the wild with preemptible cloud instances. For this experiment, we perform one of the most resource-demanding tasks in modern deep learning --- unsupervised pretraining of Transformers~\cite{bert,roberta,radford2019language,gpt3}.
We opt for the ALBERT model~\cite{albert} to make better use of communication-constrained devices. This model has fewer trainable parameters due to layer-wise weight sharing.

\begin{figure*}[t]
    \noindent
    \centering
    \vspace{-10pt}
    \includegraphics[width=\textwidth]{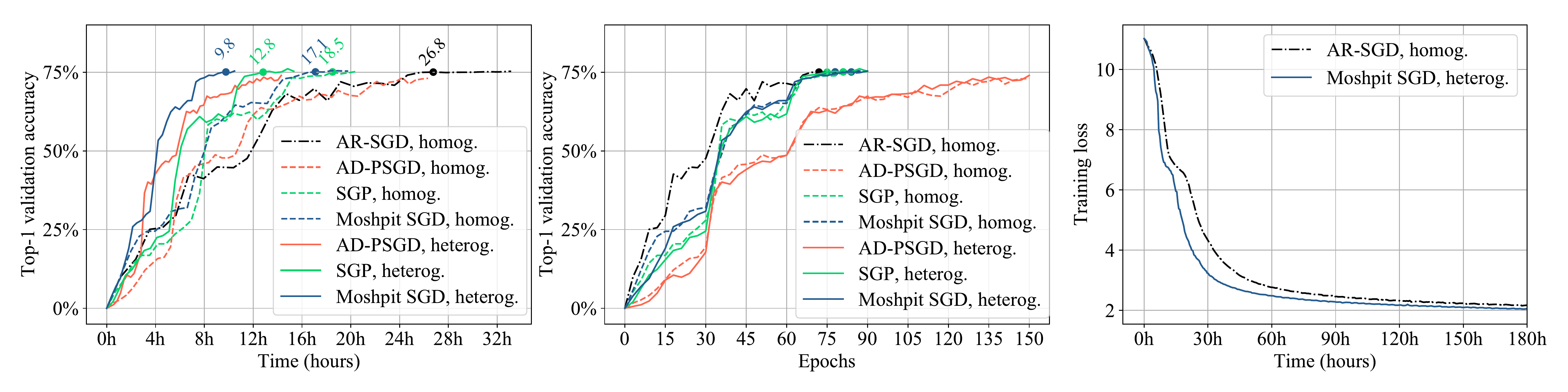}
    \vspace{-16pt}
    \caption{\textbf{(Left, Middle)} ResNet-50 top-1 validation accuracy for ImageNet as a function of training time (left) and epochs (middle). \textbf{(Right)} Full training objective (MLM + SOP) of ALBERT-large on BookCorpus as a function of training time.}
    \label{fig:all}\vspace{-6pt}
\end{figure*}

Specifically, we train ALBERT-large (18M parameters) on the BookCorpus~\cite{bookcorpus} dataset, following the training setup from the original paper. We minimize the masked language modeling loss (MLM) along with the sentence order prediction loss (SOP) using the LAMB optimizer~\cite{You2020Large} with a global batch size of 4096 and sequence length 512. We measure convergence in terms of full training loss~\cite{lin2020multinode,fedus2021switch}. Similarly to Section~\ref{sect:experiments_vision}, we use two training setups:
\vspace{-4pt}\begin{itemize}[leftmargin=*]
    \item \textbf{Homogeneous:} a single cloud instance with $8$ Tesla V100-PCIe GPUs and 56 vCPUs;
    \item \textbf{Heterogeneous:} a total of 66 preemptible GPUs, 32 of which are cloud T4, and the remaining 34 are various devices rented on a public marketplace.
\end{itemize}\vspace{-4pt}

Despite the fact that the latter setup has almost $3{\times}$ more raw compute\footnote{Based on official performance benchmarks~\cite{nvidia_perf}.}, its hourly rent costs less than the homogeneous setup due to relying on preemptible instances\footnote{Please refer to Appendix~\ref{sect:detailed_setup} for full experimental setups.}. This instance type is much cheaper than regular cloud instances, but it can be interrupted at any time. As a side-effect, the participants in \textbf{heterogeneous} setup are also spread across 3 continents with uneven network bandwidth, ranging from 100Mb/s to 1500Mb/s per worker. These limitations make it impractical to deploy conventional all-reduce protocols. By contrast, the fully decentralized nature of Moshpit SGD allows it to operate on unreliable nodes.

In this setup, the participants accumulate gradients over multiple local batches and use DHT to track the global batch size. Once the swarm collectively accumulates gradients over 4096 training samples, it runs 2 rounds of Moshpit All-Reduce with $M{=}8$ and $d{=}2$. Unfortunately, training with simple parameter averaging does not converge, likely due to diverging LAMB statistics. To mitigate this issue, workers recover ``pseudo-gradients''~\cite{reddi2021adaptive,chen2020toward} after averaging to update the optimizer statistics.

Figure~\ref{fig:all} (right) demonstrates that Moshpit SGD with a fully preemptible fleet of machines trains 1.5 times faster than the traditional data-parallel setup.
The final loss achieved by two training strategies is the same within the margin of error.
A closer investigation reveals that this speedup is entirely explained by the reduced iteration time.
An interesting observation is that the iteration time of Moshpit SGD varies between {10--22} seconds, while AR-SGD consistently spends {25}s per step. This can be explained by natural variation in the preemptible fleet size: there were 30--66 active participants depending on the resource availability.

\vspace{-6pt}
\section{Conclusion and future work}
\vspace{-4pt}
In this work, we propose Moshpit All-Reduce, a decentralized averaging protocol intended for distributed optimization in unstable and network-constrained environments. It has favorable theoretical properties when compared to gossip-based approaches and achieves considerable speedups in distributed training for image classification and masked language modeling.

Our approach was primarily designed for cloud-based training and federated learning, as well as for distributed training on unreliable instances; future work might explore additional settings, such as collaborative training of neural networks.
Another potential research direction is to study the interactions of Moshpit All-Reduce with other methods that improve communication efficiency of distributed optimization, such as gradient compression.
Finally, the idea of arranging All-Reduce nodes into groups can be improved to address specific issues that may arise in practice, such as the varying number of workers and their geographical distribution. 

\vspace{-6pt}
\section*{Acknowledgements}
\vspace{-4pt}
We would like to thank Anastasia Koloskova, Liudmila Prokhorenkova and Anton Osokin for helpful feedback and discussions. We are also grateful to the anonymous reviewers for their suggestions on improving the paper. Finally, we would like to thank Dmitry Afanasiev, Vladimir Aliev, Anand Jayarajan and Michael Solotky for their suggestions on the technical aspects of our study. 
This project was supported in
part by the Canada Foundation for Innovation JELF grant,
NSERC Discovery grant, AWS Machine Learning Research
Award, and Facebook Faculty Research Award. The paper was also partially supported by by a grant for research centers in the field of artificial intelligence, provided by the Analytical Center for the Government of the Russian Federation in accordance with the subsidy agreement (agreement identifier 000000D730321P5Q0002) and the agreement with the Moscow Institute of Physics and Technology dated November 1, 2021 No. 70-2021-00138. The computational resources for the experiments were provided by the Amazon Research Awards program and Yandex.

\bibliographystyle{unsrt}
\bibliography{bibliography}

\clearpage
\part*{Supplementary Material}
\appendix

\section{GPU instance costs}
\label{sect:cloud_costs}

This section provides a brief cost analysis of typical deep learning compute resources both in the cloud and on-premises.
For brevity, we limit this analysis to the popular GPUs available at the time of submission. Note that the exact costs will depend on a variety of factors such as the cloud provider, the region, electricity costs, and market fluctuations. Therefore, we warn the reader to consider this analysis only as a rough estimate. 

Specifically, we estimate the compute costs for the occasional usage scenario: running a single set of experiments over several weeks or conducting infrequent experiments. This scenario covers most research scientists and small organizations. The most straightforward way to provision a GPU server in such a scenario is to rent it from a cloud provider (e.g., GCP or AWS) or a public marketplace (e.g., Vast.ai or Golem).

While the exact server specifications vary from one provider to another, there are two broad categories of GPU machines: regular and preemptible. Regular instance types typically offer 1--8 GPUs per node with tight uptime guarantees (typically $99.99\%$) and a high-bandwidth network (tens of Gb/s). In turn, preemptible instances provide the same resource type at a significant discount with the condition that the machine can be terminated at any time after short notice.

To account for individual variations, we report the average rent price over three popular cloud providers.
We consider three popular instance types: two high-end instances with 8 Tesla V100 or A100 GPUs and a low-end instance with a single Tesla T4 GPU.
We also describe several low-end servers and workstations available on a public marketplace. Unlike cloud VMs, these instances are hosted on non-curated hardware with less uptime guarantees (typically 95\% -- 99.9\%), slower network and significant variation in performance. However, marketplace instances are the cheapest in terms of cost per TFLOPS. To quantify this, we report the average over three most affordable instances that fit the chosen minimum requirements.

As a point of comparison, we also measure each system's training performance for BERT-Large~\cite{bert} fine-tuning on SQuAD v1.1~\cite{squad} in PyTorch with mixed precision. We follow the official benchmarking protocol by~\cite{nvidia_perf} and reuse the official performance results for V100, A100, and T4 instances. The only exception is GTX 1080Ti, where we use full 32-bit precision because that device does not support efficient half-precision operations.

\begin{table}[h]
\small
\setlength{\tabcolsep}{2pt}
\renewcommand{\arraystretch}{1}
\centering
\caption{Cloud and marketplace GPU instance pricing for short-term usage.}
\label{fig:cloud_costs}
\begin{tabular}{@{}ccccccc@{}}
\toprule
\multicolumn{4}{c}{Minimum system specifications} & \multicolumn{2}{c}{Average cost, \$/hour} & \multirow{2}[2]{*}{\shortstack{BERT-Large\\ training samples/s}} \\
\cmidrule(lr){1-4}\cmidrule(lr){5-6}
GPU & CPU cores & CPU type & RAM, GB & Regular & Preemptible &  \\ \midrule
\multicolumn{7}{c}{Cloud instances} \\ \midrule
8$\times$ V100 & 64 & Intel Xeon Broadwell & 480 & 23.47 & 7.13 & 354 \\
8$\times$  A100 & 96 & AMD Epyc ROME & 960 & 30.65 & 10.18 & 755 \\
1$\times$  T4 & 4 & Intel Xeon Cascade Lake & 16 & 0.46 & 0.18 & 18 \\ \midrule
\multicolumn{7}{c}{Marketplace instances} \\ \midrule
6$\times$ 3090 & 32 & AMD Epyc Rome & 480 & 5.04 & 4.17 & 154 \\
4$\times$  2080Ti & 16 & Intel Xeon Haswell & 240 & 0.96 & 0.84 & 83.4 \\
1$\times$  RTX 1080Ti & 8 & Intel Xeon Haswell & 16 & 0.22 & 0.16 & 12 \\ \bottomrule
\end{tabular}
\end{table}

Table~\ref{fig:cloud_costs} shows two main tendencies. First, preemptible \textit{cloud} instances are, on average, three times cheaper than their non-preemptible counterparts\footnote{The cost can be up to $11{\times}$ cheaper for some instance types, e.g. Azure V100 instances in the central US region at the time of writing.}. Second, the high-end HPC-grade servers that offer the highest raw performance are less cost-effective than lower-tier servers and marketplace instances. In theory, one could match the raw floating-point performance of a $8{\times}$V100 instance at a fraction of its cost using multiple lower-tier workstations, such as $4{\times}$ RTX 2080Ti, with a smaller total cost.
However, in practice, running distributed training with these workstations is challenging due to their unreliability and slow network connection.

Note that this analysis does not represent the cloud costs for sustained GPU usage. If an organization plans to constantly use GPU resources over a period of multiple years, they can reduce the costs by deploying their own compute infrastructure or relying on the sustained usage discounts reaching up to 60--70\%. Thus, the long-term compute costs are much harder to analyze and depend on a number of additional factors, such as local electricity prices for on-premise infrastructure. However, this scenario offers similar trade-offs: HPC-grade infrastructure offers greater interconnectivity, but requires expensive network interface cards, high-end switches and a more complex setup process.

\section{Additional Related Work}
\label{sect:post_related}

In this section, we review some of the papers relevant to our work, but omitted from the main part due to space constraints. 

\subsection{Decentralized training}\label{sect:post_related_gossip}
In this subsection, we give additional details about the dependence of gossip-based optimization methods on the spectral properties on the communication graph through the spectral properties of the mixing matrix~\cite{xiao2004fast,scaman2019optimal} or the Laplacian matrix~\cite{merris1994laplacian,uribe2020dual} of the network. 
That is, gossip finds approximate average on nodes with accuracy $\varepsilon$ after $\cO\left((1-\lambda_2(\mM))^{-1}\log(\varepsilon^{-1})\right)$ iterations, where $\mM$ is the mixing matrix and $\lambda_2(\mM)$ is the second largest eigenvalue of $\mM$ when sorted by absolute value. 
The quantity $\eta = 1-\lambda_2(\mM)$ is called the spectral gap of the mixing matrix $\mM$, and $\eta^{-1}$ is typically a polynomial of the total number of nodes $N$ when the maximal degree of the node is $\cO(1)$. For example, for uniformly averaging $\mM$ one can show that $\eta^{-1} = \cO(N^2)$ for the ring topology (node degree $2$), $\eta^{-1} = \cO(N)$ for the two-dimensional torus topology (node degree  $2$), and $\eta^{-1} = \cO(1)$ for the fully connected graph (node degree $N-1$); one can find more examples in~\cite{aldous2002reversible}. Similarly, the communication complexity of decentralized optimization methods often has multiplicative dependence on either $\cO(\eta^{-1})$ (see~\cite{xu2020distributed} and references therein) or $\cO(\eta^{-\nicefrac{1}{2}})$~\cite{scaman2019optimal,uribe2020dual,fallah2019robust,kovalev2020optimal}, which is not improvable for gossip-based methods~\cite{arjevani2015communication,scaman2017optimal}.

Contrary to this, Moshpit All-Reduce does not depend on a fixed communication graph and the properties of its mixing matrix.
However, it depends on the number of averaging groups and the total number of peers (see Theorem~\ref{thm:quality_of_avg_deterministic_vectors}), which can be viewed as properties of a time-varying random communication graph. Fortunately, this dependence is often much better than in gossip: as we mentioned in the main part of the paper, even if workers are randomly split into pairs at each iteration, the simplified version of Moshpit All-Reduce makes the average distortion (the left-hand side of Equation~\ref{eq:determ_quality_of_avg}) at least $2$ times smaller after each round on average.

\subsection{Compressed communication}
Another popular approach to address the communication bottleneck is communication compression~\cite{seide20141,alistarh2017qsgd,suresh2017distributed, ramezani2021nuqsgd, faghri2020adaptive}: before sending any information (e.g., iterates, gradients, Hessians or more sophisticated data) over the network, peers compress this information by applying a possibly random transformation. As the result, peers send fewer bits for each communication round, but the total number of communication rounds needed to achieve the predefined accuracy of the solution increases. However, compression can be useful in situations when the reduction in communication costs of one round is more important than the increase in the number of these rounds~\cite{horvath2019natural}.

There are two distinct groups of works on distributed training with compressed communication: ones that focus on unbiased compression operators (e.g., Rand-K, $\ell_p$-quantization) and ones studying algorithms with biased compressors (e.g., Top-K); see a detailed summary of  popular compression operators in~\cite{beznosikov2020biased}). 
Quantized SGD (QSGD)~\cite{alistarh2017qsgd} and TernGrad~\cite{wen2017terngrad} were among the first compression methods with convergence guarantees. Next, the convergence analysis of these methods was generalized and tightened in the (strongly) convex case in~\cite{mishchenko2019distributed}. Moreover, the authors of \cite{mishchenko2019distributed} proposed a modification of QSGD called DIANA: this algorithm is based on the quantization of gradients' differences, which helps it achieve linear convergence in the strongly convex case when peers compute full gradients. Next, DIANA was generalized to arbitrary unbiased compression in~\cite{horvath2019stochastic}, where authors also developed and analyzed the variance-reduced version of DIANA. After that, several further modifications, such as Accelerated DIANA~\cite{li2020acceleration} and DIANA with bidirectional compression~\cite{gorbunov2020linearly,philippenko2020artemis}, were proposed. Finally, we refer the reader to~\cite{li2020unified,haddadpour2020federated,das2020improved, pmlr-v139-gorbunov21a} for state-of-the-art results for distributed methods with unbiased compression in the non-convex case.

However, naïve application of biased compression operators can lead to significantly worse performance in practice. For instance, as it was shown recently in~\cite{beznosikov2020biased}, parallel SGD with Top-1 compression can diverge exponentially fast. Therefore, biased compressors are used jointly with so-called error-compensation~\cite{seide20141}. The first analysis of Error-Compensated SGD (EC-SGD) was proposed in~\cite{stich2018sparsified,karimireddy2019error} which then was generalized and tightened in~\cite{beznosikov2020biased}. Next, several further improvements, such as an accelerated version of EC-SGD~\cite{qian2020error} and linearly converging EC-SGD~\cite{gorbunov2020linearly}, were recently proposed. However, current theory does not show any superiority of distributed methods with biased compressors to the ones with unbiased compression operators.
In addition, one can combine decentralized communication with compression. Such combinations with unbiased compression operators were studied in~\cite{reisizadeh2019exact,kovalev2020linearly} and with biased operators in~\cite{pmlr-v97-koloskova19a,Koloskova2020Decentralized}.
In this paper, we do not study the interaction of different compression methods and Moshpit Averaging, leaving this promising direction to future work.

\subsection{Multiple local steps}
Alternatively, to reduce the impact of the communication bottleneck, it is possible to perform several local optimization steps on each peer between the communication rounds.
This approach is based on the idea that the increased computational load of peers will decrease the number of communication rounds required to obtain the optimal parameters; it is frequently used in federated learning~\cite{konevcny2016federated,kairouz2019advances}. In particular, one of the most popular methods with multiple local steps is called Local-SGD or Federated Averaging~\cite{konevcny2016federated,Stich18local}. The first results on its convergence were given in \cite{Stich18local,LinSPJ2018local}, and later they were tightened and generalized both for homogeneous~\cite{khaled2020tighter,woodworth2020local} and heterogeneous  cases~\cite{khaled2020tighter,woodworth2020minibatch}. Recently, further modifications of Local-SGD were proposed and analyzed: these modifications include acceleration \cite{yuan2020federated}, variance reduction \cite{gorbunov2020local}, communication compression \cite{basu2019qsparse,haddadpour2020federated,das2020improved}, decentralization \cite{li2019communication,koloskova2020unified}, adaptive and proximal methods \cite{reddi2021adaptive,yuan2020federated_comp}, and resistance to client drift \cite{karimireddy2020scaffold}.
Moshpit SGD can perform multiple local gradient steps before synchronization by design, as shown in Algorithm~\ref{alg:moshpit_local_sgd}.

\subsection{Asynchronous methods}
In the previous subsections, we mostly discussed synchronous distributed methods, since they are more widespread and better studied than asynchronous ones. Mainly, this is because asynchronous methods are more difficult to implement, debug and analyze under general assumptions. However, such methods can be more efficient in terms of using computational resources, which leads to faster wall-clock convergence \cite{assran2020advances}. In recent years, several asynchronous stochastic methods~\cite{recht2011hogwild,zhao2016fast,leblond2017asaga}, methods with no shared memory~\cite{peng2016arock,mishchenko2018delay}, and methods with delayed updates~\cite{agarwal2011distributed,feyzmahdavian2016asynchronous,arjevani2020tight,gorbunov2020linearly} were proposed and analyzed: one can find more details in a recent survey~\cite{assran2020advances}.
Moshpit SGD belongs to this family of asynchronous approaches as well, because the averaging steps happen in smaller groups and can be interleaved with local parameter updates.

\subsection{Distributed Hash Tables}
\label{sect:related_dht}

In this work, we set out to improve distributed averaging with a dynamic matchmaking protocol. Without a central server, this protocol relies on decentralized data structures to organize peers. The main data structure we use is the Distributed Hash Table, or DHT. On a high level, DHT is a distributed fault-tolerant ``dictionary'' that can be accessed by every participant. Each key-value pair is stored on a subset of peers determined by the $\mathrm{hash}$ function of the key.

Each participant has a unique identifier (ID) sampled uniformly from the $\mathrm{hash}$ function output range. When storing a $(key,\ value)$ pair, one must find $k$ peers whose IDs are nearest to $\mathrm{hash}(key)$ according to a chosen metric. After that, the participant requests each of those peers to store $(key,\ value)$. When retrieving a value for a key, one should compute $\mathrm{hash}(key)$, search for peers with IDs nearest to that $\mathrm{hash}$ value and request the value from those peers.

Specific DHT versions, such as Chord~\cite{chord} or Kademlia~\cite{kademlia}, employ different hash types and algorithms for finding nearest peers. For instance, Kademlia DHT sorts peers based on the XOR distance function: $d(x, y) = \mathrm{int}(x \oplus y)$.

In DHT, each participant is directly aware of only a small subset of peers. When storing or retrieving a key, the participant requests additional peers from its neighbors in a semi-greedy search, minimizing the XOR distance until it finds $k$ nearest peers. In Kademlia, nodes form a special navigable graph structure that lets them find nearest peers in at most $\cO(k + \log N)$ requests to other peers, where $N$ is the total number of participants. Due to their scalability and fault-tolerance, DHTs found numerous applications including BitTorrent, Ethereum, I2P and decentralized deep learning~\cite{learning_at_home}.

\section{Proofs of Mixing Properties of Moshpit All-Reduce}\label{sect:missing_proofs}

\textbf{Notation.} Throughout the following sections, we use the standard notation from the literature on stochastic optimization. That is, for any $n$-dimensional vectors $x = (x_1,\ldots,x_n)^\top,y = (y_1,\ldots,y_n)^\top\in\R^n$ we use $\langle x,y\rangle$ to denote the standard inner product: $\langle x, y\rangle = x_1y_1 + \ldots + x_ny_n$. Next, we use $\|x\|$ to denote the $\ell_2$=norm of $x$ ($\|x\| = \sqrt{\langle x, x\rangle}$), $\EE[\xi]$ to denote an expectation of a random variable $\xi$, $\EE[\xi\mid \eta]$ is used for the conditional expectation of $\xi$ given $\eta$, and $\PP\{E\}$ denotes the probability of an event $E$.

\subsection{Computing exact average in a full grid}\label{sect:equiv_to_torus}
As discussed in Section~\ref{sect:method_algorithm}, Moshpit All-Reduce obtains the exact average of parameter vectors from $N$ peers arranged in a grid with $d$ coordinates and $M$ positions per coordinate when $N\equiv M^d$. That is, when the grid is full and each step averages $M$ parameter values along a single grid coordinate without repetitions, the algorithm needs only $d$ steps to compute the actual average across all nodes. In this section, we give a proof of this fact.

First, let us formally define the setting and the averaging steps of Moshpit All-Reduce in this specific case. Let $\theta_{i_1 i_2\ldots i_d}$ be the parameter vector of the worker with coordinates $i_1, i_2,\ldots, i_d$; each coordinate $i_k$ takes values from $1$ to $M$, because the hypercube of peers is completely full (thus, due to the pigeonhole principle, there are no unoccupied coordinates). Next, arrange the coordinates of these vector according to the order of averaging iterations: namely, at iteration 1
\begin{equation}
    \overline{\theta}_{i_1 i_2\ldots  i_d}^1=\frac{1}{M}\sum_{j_1=1}^M \theta_{j_1 i_2\ldots i_d},\quad  i_1\in\{1,\ldots,M\},
\end{equation}
which means that for the first iteration, we take the average across the first axis $\overline{\theta}^1$ and replicate it across all $M$ resulting vectors regardless of their index $i_1$. The next averaging steps can be expressed similarly with a simple recurrence relation:
\begin{equation}
\label{eqn:avg_recurrence}
    \overline{\theta}_{i_1 i_2 \ldots i_d}^t=\frac{1}{M}\sum_{j_t=1}^M \overline{\theta}_{i_1\ldots i_{t-1} j_t i_{t+1}\ldots i_d}^{t-1}.
\end{equation}
Given this formal definition, we can now state and prove the exact averaging result:
\begin{theorem}[Exact average in a full $d$-dimensional hypercube after $d$ steps]
Assume that $M^d$ peers are arranged in a $d$-dimensional hypercube with $M$ positions in each dimension. Also, assume that each peer fully participates in every averaging step and $M$-sized groups for each averaging iteration are determined based on the hypercube coordinates. Then, if Moshpit All-Reduce is ran in the above setup for $d$ iterations without repeating groups (i.e. averaging across each dimension exactly once), its result for each participant is the average value of $\theta$ across all $M^d$ peers.
\end{theorem}
\begin{proof}
We can directly obtain the expression for the average by expanding the recurrence and rearranging the sums:
\begin{eqnarray*}
    \overline{\theta}_{i_1 i_2\ldots i_d}^d &=& \frac{1}{M}\sum_{j_d=1}^M\overline{\theta}_{i_1\ldots i_{d-1} j_d}^{d-1}=\frac{1}{M}\sum_{j_d=1}^M\left(\frac{1}{M}\sum_{j_{d-1}=1}^M \overline{\theta}_{i_1 i_2\ldots j_{d-1}j_d}\right)=\ldots\\
    &=& \frac{1}{M}\Bigg(\underbrace{\sum_{j_d=1}^M\Bigg(\frac{1}{M}\sum_{j_{d-1}=1}^M\ldots\sum_{j_2=1}^M\Bigg(\frac{1}{M}\sum_{j_1=1}^M}_{d\textrm{ summations}} \theta_{j_1 \ldots j_d}\Bigg)\Bigg)\Bigg)\\
    &=& \frac{1}{M^d}\sum_{j_d=1}^M\sum_{j_{d-1}=1}^M\ldots\sum_{j_2=1}^M\sum_{j_1=1}^M \theta_{j_1 \ldots j_d} =\frac{1}{M^d}\sum_{j_1, \ldots, j_d=1}^M  \theta_{j_1 \ldots j_d}.
\end{eqnarray*}
But this is exactly the global average of all $\theta$, since there are $M^d$ participants and each vector is represented in the sum because of summation over all possible indices.
\end{proof}

Notice that for a given grid of peers, if some of its indices do not have corresponding parameter vectors, Equation~\ref{eqn:avg_recurrence} may result in different average vectors on different workers due to different numbers of peers along a coordinate for different indices. For example, running two iterations of Moshpit Averaging with $d=2,\ M=2$ and three parameter vectors $\theta_{11},\ \theta_{21},\ \theta_{22}$ results in $\frac{\theta_{11}+\theta_{21}}{2}$ on the first worker and $\frac{\theta_{11}+\theta_{21}}{4}+\theta_{22}$ on other workers, with neither equal to the global average. However, the variance of the averaged vectors does decrease, which is formally proven in Section~\ref{sec:proof_quality_of_avg_deterministic_vectors}.

\subsection{Proof of Theorem~\ref{thm:quality_of_avg_deterministic_vectors_0}}\label{sect:correctness_proof}
Below we provide the complete proof of Theorem~\ref{thm:quality_of_avg_deterministic_vectors_0}. For the readers' convenience, we restate the theorem.
\begin{theorem}[Theorem~\ref{thm:quality_of_avg_deterministic_vectors_0}]\label{thm:quality_of_avg_deterministic_vectors_0_supp}
If all workers have non-zero probability of successfully running a communication round in Moshpit Averaging and the order of $\texttt{peers}_t$ is random, then all local vectors $\theta^t_i$ converge to the global average with probability $1$:
\begin{equation}
    \forall i = 1,\ldots, N\quad \left\|\theta^t_i - \frac1N \sum_{i=1}^N \theta^0_i\right\|^2 \xrightarrow[t\to\infty]{} 0.\label{eq:quality_of_avg_deterministic_vectors_0_supp}
\end{equation}
\end{theorem}
\begin{proof}[Proof of Theorem~\ref{thm:quality_of_avg_deterministic_vectors_0}]
    First of all, we notice that \eqref{eq:quality_of_avg_deterministic_vectors_0_supp} is equivalent to
    \begin{equation}
    \forall i = 1,\ldots, N,\;\forall j=1,\ldots,n\quad \left(\theta^t_i(j) - \frac1N \sum_{i=1}^N \theta^0_i(j)\right)^2 \xrightarrow[t\to\infty]{} 0,\label{eq:quality_of_avg_deterministic_vectors_0_supp_tech_1}
    \end{equation}
    where $\theta_i^t(j)$ denotes $j$-th component of $\theta_i^t$. Consider an arbitrary component $j \in \{1,\ldots,n\}$ and the sequence of intervals $\{I_{j,t}\}_{t\ge 0}$ where $I_{j,t} = \text{conv}\{\theta_1^t(j),\theta_2^t(j),\ldots, \theta_N^t(j)\}$. Then, $\{I_{j,t}\}_{t\ge 0}$ is a sequence of nested intervals ($I_{j,t+1} \subseteq I_{j,t} \forall t\ge 0$), since averaging in groups does not expand the convex hull of $\{\theta_1^t,\theta_2^t,\ldots, \theta_N^t\}$. For convenience, we specify the bounds of the intervals: $I_{j,t} = [a_{j,t}, b_{j,t}]$. Using the Cantor's intersection theorem, we conclude that
    \begin{equation*}
        \bigcap\limits_{t=0}^\infty I_{j,t} = I_j = [a_j, b_j],
    \end{equation*}
    where $\overline{\theta}(j) = \frac{1}{N}\sum_{i=1}^n\theta_i^0(j) \in [a_j, b_j]$. If $[a_j, b_j] = \{\overline{\theta}(j)\}$ with probability $1$, then \eqref{eq:quality_of_avg_deterministic_vectors_0_supp_tech_1} holds with probability $1$ as well. Suppose the opposite: there exist such $j \in \{1,\ldots,n\}$, $[a,b]$ and $\delta,\Delta > 0$ that $\overline{\theta}(j) \in [a,b]$, $b-a = \Delta$ and
    \begin{equation*}
        \PP\Bigg\{\underbrace{[a,b] \subseteq \bigcap\limits_{t=0}^\infty I_{j,t}}_{E}\Bigg\} = \delta > 0\quad \text{ and }\quad \forall \varepsilon > 0\; \PP\Bigg\{\underbrace{[a-\varepsilon,b+\varepsilon] \subseteq \bigcap\limits_{t=0}^\infty I_{j,t}}_{E_{\varepsilon}}\Bigg\} < \delta.
    \end{equation*}
    This implies that for all $\varepsilon > 0$ there exists such $T_{\varepsilon} > 0$ that
    \begin{equation*}
        \PP\Big\{\underbrace{\forall t \ge T_{\varepsilon}\;\; a_{j,t}\in [a-\varepsilon,a], b_{j,t}\in[b,b+\varepsilon]}_{E_{\varepsilon}'}\Big\} = \delta_{\varepsilon} > 0.
    \end{equation*}
    Consider $\varepsilon = \frac{\Delta}{(2N+100)^{2N}}$ and assume that the event $E_{\varepsilon}'$ holds. Next, we introduce new notation: $J_{\text{left}}^t = \{i \in \{1,\ldots, n\}\mid \theta_{i}^t(j) \in [a-\varepsilon,a]\}$ and $J_{\text{right}}^t = \{i \in \{1,\ldots, n\}\mid \theta_{i}^t(j) \in [b,b+\varepsilon]\}$. Since $E_{\varepsilon}'$ holds the sets $J_{\text{left}}^t$ and $J_{\text{right}}^t$ are non-empty for all $t\ge T_{\varepsilon}$ with probability $\delta_{\varepsilon} > 0$:
    \begin{equation}
        \PP\left\{\forall t \ge T_{\varepsilon}\;\; J_{\text{left}}^t \neq \varnothing\text{ and }  J_{\text{right}}^t \neq \varnothing\right\} = \delta_{\varepsilon} > 0. \label{eq:quality_of_avg_deterministic_vectors_0_supp_tech_2}
    \end{equation}
    We notice that every pair of workers $i_1,i_2$ has a non-zero probability of taking part in the averaging inside the common group at each iteration since all workers have a non-zero probability of successfully running a communication round and the order of $\texttt{peers}_t$ is random. This implies that every pair of workers $i_1,i_2$ with probability $1$ take part in the averaging inside the common group infinitely many times when $t$ goes to the infinity.
    
    Next, we choose some $t_0 \ge T_{\varepsilon}$. Let $J_{\text{left}}^{t_0} = \{i_{l,1},\ldots, i_{l,q_l}\}$ and $J_{\text{right}}^{t_0} = \{i_{r,1},\ldots, i_{r,q_r}\}$. Consider the event $E_{\varepsilon,0}' \subseteq E_{\varepsilon}'$ such that in $E_{\varepsilon,0}'$ peer $i_{l,1}$ computes an average in the group containing any peer from $J_{\text{right}}^{t_0}$ at some iteration $t_1 > t_0$. Our observations above imply that $\PP\{E_{\varepsilon,0}'\} = \PP\{E_{\varepsilon}'\} = \delta_{\varepsilon} > 0$. Then, $\theta_{i_{l,1}}^{t_1}(j) \ge \frac{N-1}{N}(a-\varepsilon) + \frac{1}{N}b = a-\varepsilon + \frac{1}{N}(\Delta + \varepsilon) = a - \frac{\Delta}{(2N+100)^{2N}} + \frac{1}{N}\left(\Delta + \frac{\Delta}{(2N+100)^{2N}}\right) > a + \frac{\Delta}{2N}$, i.e., $\theta_{i_{l,1}}^{t_1}(j) \in (a,b]$ meaning that $i_{l,1} \not\in J_{\text{left}}^{t_1}$. The last part of the proof shows that for any $t\ge t_1$, the peer $i_{l,1}$ will never be the part of $J_{\text{left}}^t$ and after a finite number of iterations $J_{\text{left}}^t = \varnothing$ with probability $\delta_{\varepsilon} > 0$ when $E_{\varepsilon,0}'$ holds, implying the contradiction with \eqref{eq:quality_of_avg_deterministic_vectors_0_supp_tech_2}.
    
    To show that, we consider the following set of peers: $\widehat{J}_{\text{left}}^{t_1} = \{i\in\{1,\ldots,n\}\mid \exists t \ge t_1:\; \theta_i^{t}(j)\in [a-\varepsilon, a+\frac{\Delta}{2N})\}$. Next, we consider the event $E_{\varepsilon,1}'\subseteq E_{\varepsilon,0}'$ such that in $E_{\varepsilon,1}'$ peer $i_{l,1}$ computes an average in the group containing some peer $i_{l,avg,1}$ from $\widehat{J}_{\text{left}}^{t_1}$ at some iteration $t_2 > t_1$ (and $t_2$ is the first such moment after $t_1$). Again, our observations imply $\PP\{E_{\varepsilon,1}'\} = \PP\{E_{\varepsilon,0}'\} = \delta_{\varepsilon}>0$. Then, $\theta_{i_{l,1}}^{t_2}(j) = \theta_{i_{l,avg,1}}^{t_2}(j) > \frac{N-1}{N}(a-\varepsilon) + \frac{1}{N}\left(a+\frac{\Delta}{2N}\right) = a + \frac{\Delta}{2N^2} - \frac{(N-1)\Delta}{N(2N+100)^{2N}} > a + \frac{\Delta}{4N^2}$. After that, we consider the event $E_{\varepsilon,2}'\subseteq E_{\varepsilon,1}'$ such that in $E_{\varepsilon,2}'$ peer $i_{l,1}$ or $i_{l,avg,1}$ computes an average in the group containing a peer $i_{l,avg,2}\neq i_{l,avg,1}$ from $\widehat{J}_{\text{left}}^{t_1}$ at an iteration $t_3 > t_2$ (and $t_3$ is the first such moment after $t_2$). Then, $\theta_{i_{l,1}}^{t_3}(j), \theta_{i_{l,avg,1}}^{t_3}(j)$ and $\theta_{i_{l,avg,2}}^{t_3}(j)$ are greater than $\frac{N-1}{N}(a-\varepsilon) + \frac{1}{N}\left(a + \frac{\Delta}{4N^2}\right) = a + \frac{\Delta}{4N^3} - \frac{(N-1)\Delta}{N(2N+100)^{2N}} > a + \frac{\Delta}{8N^3}$.
    
    Therefore, after at least $N-1$ of such averaging iterations, with probability $\delta_\varepsilon$ all $\theta_i^t(j)$ will be greater than $a + \frac{\Delta}{(2N)^N} > a$ while $E_{\varepsilon}'$ holds. This contradicts \eqref{eq:quality_of_avg_deterministic_vectors_0_supp_tech_2}. Therefore, 
    \begin{equation*}
        \bigcap\limits_{t=0}^\infty I_{j,t} = \{\overline{\theta}(j)\}
    \end{equation*}
    with probability $1$, which concludes the proof.
\end{proof}

\subsection{Proof of Theorem~\ref{thm:quality_of_avg_deterministic_vectors}}\label{sec:proof_quality_of_avg_deterministic_vectors}
In this section, we provide the complete proof of Theorem~\ref{thm:quality_of_avg_deterministic_vectors}. For convenience, we restate the theorem below.
\begin{theorem}[Theorem~\ref{thm:quality_of_avg_deterministic_vectors}, averaging convergence rate]\label{thm:quality_of_avg_deterministic_vectors_supp}
    Consider the modification of Moshpit All-Reduce that works as follows: at each iteration $k\geq 1$ 1) peers are randomly split into $r$ disjoint groups of sizes $M_1^k,\ldots, M_r^k$ in such a way that $\sum_{i=1}^r M_i^k = N$ and $M_i^k \ge 1\  \forall i = 1,\ldots,r$ and 2) peers from each group compute their group average via All-Reduce. Let $\theta_1,\ldots,\theta_N$ be the input vectors of this procedure and $\theta_1^T,\ldots,\theta_N^T$ be the outputs after $T$ iterations. Then,
    \begin{eqnarray}
         \EE\left[\frac{1}{N}\sum\limits_{i=1}^N\|\theta_i^T - \overline{\theta}\|^2\right] = \left(\frac{r-1}{N} + \frac{r}{N^2}\right)^T\cdot\frac{1}{N}\sum\limits_{i=1}^N\|\theta_i - \overline{\theta}\|^2, \label{eq:determ_quality_of_avg_supp}
    \end{eqnarray}
    where $\overline{\theta} = \frac{1}{N}\sum_{i=1}^N\theta_i$.
\end{theorem}
\begin{proof}
First of all, let us clarify the procedure of random splitting of peers in $r$ groups. We assume that at iteration $k$ of the modified algorithm we generate a random permutation $\pi^k = (\pi_1^k,\ldots,\pi_N^k)$ of $1,\ldots, N$. Next, $J_1^k = \{\pi_1^k,\ldots,\pi_{M_1^k}^k\}$ form the indices of the first group of workers, $J_2^k = \{\pi_{M_1^k+1}^k,\ldots,\pi_{M_2^k}^k\}$ are the indices of the second group, and $J_r^k = \{\pi_{M_1^k+M_2^k+\ldots+M_{r-1}^k+1}^k,\ldots,\pi_{N}^k\}$ are the indices of group $r$. In other words, we generate a random permutation and take contiguous subgroups of indices corresponding to predefined group sizes $M_i^k$, starting from the first group.

By definition, we have $\bigsqcup_{i=1}^r J_i^k = \{1,2,\ldots,N\}$, where $\sqcup$ defines the disjoint union operator. Moreover, notice that group sizes $M_1^k,\ldots,M_r^k$ can depend on $k$ and even be random: for our analysis, it is sufficient that the randomness defining the permutation is independent from $M_1^k,\ldots,M_r^k$. Next, vectors $\theta_1^k,\ldots,\theta_N^k$ are obtained by the following formula:
\begin{equation*}
    \forall j=1,\ldots,N,\quad \theta_j^k = \frac{1}{M_i^k}\sum\limits_{t\in J_i^k}\theta_t^{k-1},\quad \text{where } J_i^k \text{ is the group for which } j\in J_i^k.
\end{equation*}
Using this, we show that the average of vectors $\{\theta_i^k\}_{i=1}^n$ remains the same throughout the iterations of Moshpit All-Reduce:
\begin{equation*}
    \frac{1}{N}\sum\limits_{j=1}^N\theta_j^k = \frac{1}{N}\sum\limits_{i=1}^rM_i^k\cdot\frac{1}{M_i^k}\sum\limits_{t\in J_i^k}\theta_t^{k-1} = \frac{1}{N}\sum\limits_{i=1}^r\sum\limits_{t\in J_i^k}\theta_t^{k-1} = \frac{1}{N}\sum\limits_{j=1}^N\theta_j^{k-1}.
\end{equation*}
Therefore, the quantity $\frac{1}{N}\sum_{j=1}^N\|\theta_j^k - \overline{\theta}\|^2$ (average distortion) measures the quality of averaging. For this quantity, we can derive the following expression:
\begin{eqnarray}
    \frac{1}{N}\sum\limits_{j=1}^N\|\theta_j^k - \overline{\theta}\|^2 &=& \frac{1}{N}\sum\limits_{i=1}^r M_i^k\left\|\frac{1}{M_i^k}\sum\limits_{t\in J_i^k}\theta_t^{k-1} - \overline{\theta}\right\|^2\notag\\
    &=& \frac{1}{N}\sum\limits_{i=1}^r\frac{1}{M_i^k}\left(\sum\limits_{t\in J_i^k}\|\theta_t^{k-1} - \overline{\theta}\|^2 + 2\sum\limits_{t,l\in J_i^k, t < l}\langle \theta_t^{k-1} - \overline{\theta}, \theta_l^{k-1} - \overline{\theta} \rangle\right).\notag
\end{eqnarray}
Taking the expectation $\EE_{\pi^k}[\cdot]$ with respect to the randomness coming from the choice of $\pi^k$ we get
\begin{eqnarray}
    \EE_{\pi^k}\left[\frac{1}{N}\sum\limits_{j=1}^N\|\theta_j^k - \overline{\theta}\|^2\right] &\notag\\
    &\hspace{-2.5cm}= \frac{1}{N}\sum\limits_{i=1}^r\frac{1}{M_i^k}\left(\EE_{\pi^k}\left[\sum\limits_{t\in J_i^k}\|\theta_t^{k-1} - \overline{\theta}\|^2\!\right] \!+\! 2\EE_{\pi^k}\!\left[\sum\limits_{t,l\in J_i^k, t < l}\langle \theta_t^{k-1} - \overline{\theta}, \theta_l^{k-1} - \overline{\theta} \rangle\right]\right).\notag
\end{eqnarray}
Since $\forall j,j_1,j_2 \in\{1,\ldots,N\},j_1\neq j_2$ and for all $i=1,\ldots,r$
\begin{equation*}
    \PP\left\{j\in J_i^k\right\} = \frac{M_i^k}{N},\quad \PP\left\{j_1,j_2 \in J_i^k\right\} = \frac{M_{i}^k(M_i^k - 1)}{N^2},
\end{equation*}
we have
\begin{eqnarray*}
    \EE_{\pi^k}\left[\frac{1}{N}\sum\limits_{j=1}^N\|\theta_j^k - \overline{\theta}\|^2\right] &=& \frac{1}{N}\sum\limits_{i=1}^r\frac{1}{N}\sum\limits_{j=1}^N\|\theta_j^{k-1} - \overline{\theta}\|^2\\
    &&\quad +\frac{1}{N}\sum\limits_{i=1}^r2\frac{M_i^k - 1}{N^2}\sum\limits_{1 \le j_1 < j_2 \le N}\langle \theta_{j_1}^{k-1} - \overline{\theta}, \theta_{j_2}^{k-1} - \overline{\theta}\rangle\\
    &=& \frac{r}{N^2}\sum\limits_{j=1}^N\|\theta_j^{k-1} - \overline{\theta}\|^2 + 2\frac{N-r}{N^3}\sum\limits_{1 \le j_1 < j_2 \le N}\langle \theta_{j_1}^{k-1} - \overline{\theta}, \theta_{j_2}^{k-1} - \overline{\theta}\rangle\\
    &=& \left(\frac{r}{N^2} - \frac{N-r}{N^3}\right)\sum\limits_{j=1}^N\|\theta_j^{k-1} - \overline{\theta}\|^2 +\frac{N-r}{N^3}\sum\limits_{j=1}^N\|\theta_j^{k-1} - \overline{\theta}\|^2\\
    &&\quad +2\frac{N-r}{N^3}\sum\limits_{1 \le j_1 < j_2 \le N}\langle \theta_{j_1}^{k-1} - \overline{\theta}, \theta_{j_2}^{k-1} - \overline{\theta}\rangle\\
    &=& \frac{N(r-1)+r}{N^3}\sum\limits_{j=1}^N\|\theta_j^{k-1} - \overline{\theta}\|^2 + \frac{N-r}{N^3}\underbrace{\left\|\sum\limits_{j=1}^N(\theta_j^{k-1} - \overline{\theta})\right\|^2}_{\|N\overline{\theta} - N\overline{\theta}\|^2 = 0}\\
    &=& \left(\frac{r-1}{N} + \frac{r}{N^2}\right)\cdot\frac{1}{N}\sum\limits_{j=1}^N\|\theta_j^{k-1} - \overline{\theta}\|^2.
\end{eqnarray*}
Finally, we take the full expectation from the both sides of the above equation and apply the tower property $\EE\left[\EE_{\pi^k}\left[\cdot\right]\right] = \EE\left[\cdot\right]$:
\begin{equation*}
    \EE\left[\frac{1}{N}\sum\limits_{j=1}^N\|\theta_j^k - \overline{\theta}\|^2\right] = \left(\frac{r-1}{N} + \frac{r}{N^2}\right)\EE\left[\frac{1}{N}\sum\limits_{j=1}^N\|\theta_j^{k-1} - \overline{\theta}\|^2\right].
\end{equation*}
Unrolling the recurrence for $k=T$, we establish \eqref{eq:determ_quality_of_avg_supp}.
\end{proof}

\begin{remark}
    The result implies that increasing the group size $\alpha > 1$ times implies almost $\alpha$ times faster convergence to the average.
\end{remark}

\begin{remark}
    Our analysis can be easily generalized to the case when number of groups $r$ can depend on $k$ and be a random variable independent from the choice of permutations and the number of groups at previous steps. In this case, \eqref{eq:determ_quality_of_avg_supp} transforms into
    \begin{equation}
        \EE\left[\frac{1}{N}\sum\limits_{i=1}^N\|\theta_i^T - \overline{\theta}\|^2\right] = \frac{1}{N}\sum\limits_{i=1}^N\|\theta_i - \overline{\theta}\|^2\cdot\prod_{k=1}^T\left(\frac{\EE[r_k]-1}{N} + \frac{\EE[r_k]}{N^2}\right), \label{eq:determ_quality_of_avg_generalized_supp}
    \end{equation}
    where $r_k$ is the number of groups at iteration $k$.
\end{remark}

\subsection{Additional Guarantees For Moshpit Averaging}\label{sec:mix_rand_proof}
In this section,  we derive the result measuring the rate of variance reduction when averaging random vectors with Algorithm~\ref{alg:moshpit}. We start with the following technical lemma:
\begin{lemma}\label{lem:ode_lemma}
    Let $\xi \sim \text{Binom}(M,p)$ have a binomial distribution with parameters $M$ (number of trials) and $p$ (probability of success for each trial). Then
    \begin{eqnarray}
        m_1(M,p) := \EE\left[\min\left\{\frac{1}{\xi},1\right\}\right] &=& (1-p)^M + \sum\limits_{i=1}^M\frac{(1-p)^{M-i} - (1-p)^M}{i}, \label{eq:binom_first_inverse_moment}\\
        m_2(M,p) := \EE\left[\min\left\{\frac{1}{\xi^2},1\right\}\right] &=& (1-p)^M + \sum\limits_{i=1}^M\frac{(1-p)^{M-i} - (1-p)^M}{i}\sum\limits_{j=i}^M\frac{1}{j}. \label{eq:binom_second_inverse_moment}
    \end{eqnarray}
\end{lemma}
\begin{proof}
    We start with the proof of \eqref{eq:binom_first_inverse_moment}. By definition of the expectation, we have
    \begin{eqnarray*}
        \EE\left[\min\left\{\frac{1}{\xi},1\right\}\right] &=& (1-p)^M + \sum\limits_{i=1}^M \frac{1}{i}p^i(1-p)^{M-i}\binom{M}{i}.
    \end{eqnarray*}
    For simplicity of further derivations, we introduce the following notation: $m_1(M,p) = \EE\left[\min\left\{\frac{1}{\xi},1\right\}\right]$ and $m_2(M,p) = \EE\left[\min\left\{\frac{1}{\xi^2},1\right\}\right]$. Taking the derivative of $m_1(M,p)$ by $p$, we obtain
    \begin{eqnarray*}
        m_1'(M,p) &=& -M(1-p)^{M-1} + \sum\limits_{i=1}^Mp^{i-1}(1-p)^{M-i}\binom{M}{i} \\
        &&\quad - \sum\limits_{i=1}^M\frac{M-i}{i}p^i(1-p)^{M-i-1}\binom{M}{i}\\
        &=& -M(1-p)^{M-1} + \frac{1}{p}\left(-(1-p)^M + \sum\limits_{i=0}^Mp^{i}(1-p)^{M-i}\binom{M}{i}\right)\\
        && - \frac{M}{1-p}\sum\limits_{i=1}^M\frac{1}{i}p^i(1-p)^{M-i}\binom{M}{i}\\
        &&\quad + \frac{1}{1-p}\left(-(1-p)^M + \sum\limits_{i=0}^Mp^i(1-p)^{M-i}\binom{M}{i}\right)\\
        &=& -M(1-p)^{M-1} + \frac{1}{p}\left(1 - (1-p)^M\right) - \frac{M}{1-p}\left(m_1(M,p) - (1-p)^M\right)\\
        &&\quad+ \frac{1}{1-p}\left(1- (1-p)^M\right)\\
        &=& \frac{1}{p(1-p)} - \frac{(1-p)^{M-1}}{p} - \frac{M}{1-p}m_1(M,p).
    \end{eqnarray*}
    Rearranging the terms, we get the following linear first-order ODE
    \begin{equation}
        m_1'(M,p) + \frac{M}{1-p}m_1(M,p) = \frac{1}{p(1-p)} - \frac{(1-p)^{M-1}}{p}. \label{eq:first_moment_ODE}
    \end{equation}
    To solve it, we consider the following homogeneous ODE:
    \begin{equation*}
        m_1'(M,p) + \frac{M}{1-p}m_1(M,p) = 0.
    \end{equation*}
    The solution of this ODE is $m_1(M,p) = C(1-p)^M$, where $C\in\R$ is an arbitrary real constant. Next, we go back to the initial ODE \eqref{eq:first_moment_ODE} and try to find a solution of the form $m_1(M,p) = C(p)(1-p)^M$, where $C(p):\R \to \R$ is a differentiable function:
    \begin{eqnarray*}
        \left(C(p)(1-p)^M\right)' + \frac{M}{1-p}C(p)(1-p)^M &=& \frac{1}{p(1-p)} - \frac{(1-p)^{M-1}}{p}\\
        &\Downarrow&\\
        C'(p)(1-p)^M &=& \frac{1}{p(1-p)} - \frac{(1-p)^{M-1}}{p}\\
        &\Downarrow&\\
        C'(p) &=& \frac{1}{p(1-p)^{M+1}} - \frac{1}{p(1-p)}.
    \end{eqnarray*}
    Since 
    \begin{equation}
        \frac{1}{x(1-x)^{k+1}} = \frac{1}{x(1-x)^{k}} + \frac{1}{(1-x)^{k+1}}\label{eq:technical_expansion}
    \end{equation}
    for all $x\not\in \{0,1\}$ and all non-negative integers $k$, we have
    \begin{eqnarray*}
        C'(p) &=& \frac{1}{p} + \frac{1}{1-p} + \frac{1}{(1-p)^2} + \ldots + \frac{1}{(1-p)^{M+1}} - \frac{1}{p} - \frac{1}{1-p}\\
        &\Downarrow&\\
        C'(p) &=& \sum\limits_{i=1}^M(1-p)^{-i-1},
    \end{eqnarray*}
    hence
    \begin{eqnarray*}
        C(p) = \hat{C} + \sum\limits_{i=1}^M\frac{1}{i}(1-p)^{-i},
    \end{eqnarray*}
    where $\hat{C}$ is a real constant. Putting all together, we obtain
    \begin{eqnarray*}
        m_1(M,p) &=& C(p)(1-p)^M = \hat{C}(1-p)^M + \sum\limits_{i=1}^M\frac{1}{i}(1-p)^{M-i}.
    \end{eqnarray*}
    Taking $m_1(M,0) = 1$ into account, we conclude that $\hat{C} = 1 - \sum_{i=1}^M\frac{1}{i}$ and obtain \eqref{eq:binom_first_inverse_moment}.
    
    Using a similar technique, we derive \eqref{eq:binom_second_inverse_moment}. By definition of the expectation, we have
    \begin{eqnarray*}
        m_2(M,p) &=& (1-p)^M + \sum\limits_{i=1}^M \frac{1}{i^2}p^i(1-p)^{M-i}\binom{M}{i}.
    \end{eqnarray*}
    Taking the derivative of $m_2(M,p)$ by $p$, we obtain
    \begin{eqnarray*}
        m_2'(M,p) &=& -M(1-p)^{M-1} + \sum\limits_{i=1}^M\frac{1}{i}p^{i-1}(1-p)^{M-i}\binom{M}{i}\\
        &&\quad - \sum\limits_{i=1}^M\frac{M-i}{i^2}p^i(1-p)^{M-i-1}\binom{M}{i}\\
        &=& -M(1-p)^{M-1} + \frac{1}{p} \sum\limits_{i=1}^M\frac{1}{i}p^{i}(1-p)^{M-i}\binom{M}{i}\\
        && - \frac{M}{1-p}\sum\limits_{i=1}^M\frac{1}{i^2}p^i(1-p)^{M-i}\binom{M}{i} + \frac{1}{1-p}\sum\limits_{i=1}^M\frac{1}{i}p^i(1-p)^{M-i}\binom{M}{i}\\
        &=& -M(1-p)^{M-1} + \frac{1}{p}\left(m_1(M,p) - (1-p)^M\right) \\
        &&\quad + \frac{1}{1-p}\left(-M m_2(M,p) + M(1-p)^M + m_1(M,p) - (1-p)^M\right)\\
        &=& \frac{m_1(M,p)}{p(1-p)} - \frac{(1-p)^{M-1}}{p} - \frac{M}{1-p}m_2(M,p).
    \end{eqnarray*}
    Rearranging the terms, we get the following linear first-order ODE
    \begin{equation}
        m_2'(M,p) + \frac{M}{1-p}m_2(M,p) = \frac{m_1(M,p)}{p(1-p)} - \frac{(1-p)^{M-1}}{p}. \label{eq:second_moment_ODE}
    \end{equation}
    To solve this ODE, we consider the homogeneous ODE:
    \begin{equation*}
        m_2'(M,p) + \frac{M}{1-p}m_2(M,p) = 0.
    \end{equation*}
    The solution of this ODE is $m_2(M,p) = C(1-p)^M$, where $C\in\R$ is an arbitrary real constant. Next, we go back to the initial ODE \eqref{eq:second_moment_ODE} and try to find a solution of the form $m_2(M,p) = C(p)(1-p)^M$, where $C(p):\R \to \R$ is a differentiable function:
    \begin{eqnarray*}
        \left(C(p)(1-p)^M\right)' + \frac{M}{1-p}C(p)(1-p)^M &=& \frac{m_1(M,p)}{p(1-p)} - \frac{(1-p)^{M-1}}{p}\\
        &\Downarrow&\\
        C'(p)(1-p)^M &=& \frac{m_1(M,p)}{p(1-p)} - \frac{(1-p)^{M-1}}{p}\\
        &\Downarrow&\\
        C'(p) &=& \frac{m_1(M,p)}{p(1-p)^{M+1}} - \frac{1}{p(1-p)}.
    \end{eqnarray*}
    Using \eqref{eq:technical_expansion} and \eqref{eq:binom_first_inverse_moment}, we derive
    \begin{eqnarray*}
        C'(p) &\overset{\eqref{eq:binom_first_inverse_moment}}{=}& -\frac{\sum\limits_{i=1}^M\frac{1}{i}}{p(1-p)} + \frac{\sum\limits_{i=1}^M\frac{1}{i}(1-p)^{M-i}}{p(1-p)^{M+1}}\\
        &=& -\sum\limits_{i=1}^M \frac{1}{ip(1-p)} + \sum\limits_{i=1}^M\frac{1}{ip(1-p)^{i+1}}\\
        &\overset{\eqref{eq:technical_expansion}}{=}& -\sum\limits_{i=1}^M\frac{1}{i}\left(\frac{1}{p} + \frac{1}{1-p}\right)\\
        &&\quad + \sum\limits_{i=1}^M\frac{1}{i}\left(\frac{1}{p} + \frac{1}{1-p} + \frac{1}{(1-p)^2} + \ldots + \frac{1}{(1-p)^{i+1}}\right)\\
        &=& \sum\limits_{i=1}^M\frac{1}{i}\left(\frac{1}{(1-p)^2} + \ldots + \frac{1}{(1-p)^{i+1}}\right) = \sum\limits_{i=1}^M \frac{1}{(1-p)^{i+1}}\sum\limits_{j=i}^M\frac{1}{j},
    \end{eqnarray*}
    hence 
    \begin{eqnarray*}
        C(p) = \hat{C} + \sum\limits_{i=1}^M\frac{1}{i}(1-p)^{-i}\sum\limits_{j=i}^M\frac{1}{j},
    \end{eqnarray*}
    where $\hat{C}$ is a real constant. Putting all together, we obtain
    \begin{eqnarray*}
        m_2(M,p) &=& C(p)(1-p)^M = \hat{C}(1-p)^M + \sum\limits_{i=1}^M\frac{1}{i}(1-p)^{M-i}\sum\limits_{j=i}^M\frac{1}{j}.
    \end{eqnarray*}
    Taking $m_2(M,0) = 1$ into account, we conclude that $\hat{C} = 1 - \sum_{i=1}^M\frac{1}{i}\sum_{j=i}^M\frac{1}{j}$ and obtain \eqref{eq:binom_second_inverse_moment}.
\end{proof}

Using this lemma, we derive the following result:
\begin{theorem}\label{thm:quality_of_avg_supp}
    Assume that peers participating in Moshpit Averaging have independent random vectors $\theta_1,\ldots,\theta_N$ with means $\overline{\theta}_1,\ldots,\overline{\theta}_N$ and variances bounded by $\sigma^2$ before the averaging. Let $\theta_1^T,\ldots,\theta_N^T$ be the outputs of Moshpit Averaging after $T$ iterations. Finally, we assume that each peer from the grid can be dropped out for the whole averaging process before averaging independently from other peers, i.e., $N \sim \text{Binom}(M^d,p)$. Then, for all $i = 1,\ldots,N$ we have
    \begin{equation}
        \EE\left[\left\|\theta_i^T - \EE_{\theta}\left[\theta_i^T\right]\right\|^2\right] \leq M^{T-1}\sigma^2 m_1(M-1,p)\left(m_2(M-1,p)\right)^{T-1},\label{eq:variance_bound_supp}
    \end{equation}
    where functions $m_1(M,p)$ and $m_2(M,p)$ are defined in \eqref{eq:binom_first_inverse_moment} and \eqref{eq:binom_second_inverse_moment} respectively, and $\EE_\theta\left[\cdot\right]$ denotes the expectation w.r.t.\ the randomness from $\theta_1,\ldots,\theta_N$. Moreover, if $p \ge \frac{2}{3}$ and $M \ge 11$, then $m_1(M-1,p) \le \frac{2}{M}$, $m_2(M-1,p) \le \frac{3}{M^2}$ and 
    \begin{equation}
        \EE\left[\left\|\theta_i^T - \EE_{\theta}\left[\theta_i^T\right]\right\|^2\right] \leq \frac{2\sigma^2}{M(\nicefrac{M}{3})^{T-1}}.\label{eq:variance_bound_2_supp}
    \end{equation}
\end{theorem}
\begin{proof}
First of all, we recall an equivalent formulation of Moshpit Averaging. Consider a hypercube $\{1,\ldots,M\}^d$. One can consider the elements of this hypercube as hyperindices and assign a unique hyperindex to each peer so that peers can be viewed as vertices in the hypercube. Then, during the $k$-th iteration of Moshpit All-Reduce, each worker computes the average among those peers that have hyperindices with the same values except the $k$-th index; in other words, peers compute averages along the $k$-th dimension of the hypercube. Next, if $N = 0$, we assume that $\theta_i^T = \EE_{\theta}\left[\theta_i^T\right]$ and \eqref{eq:variance_bound_supp} holds for free. Therefore, to derive \eqref{eq:variance_bound_supp}, we assume that $N > 0$.

More formally, we use the following notation: $\theta_{C_i} = \theta_i$ for all $i= 1,\ldots,N$, where $C_{i} = (c_{1}^i, c_2^i,\ldots, c_d^i)$, $c_{j}^i \in \{1,\ldots,M\}$ for all $j = 1,\ldots,M$, and $C_{i} \neq C_k$ for $i\neq k$. Let $\cC$ be the set of hyperindices corresponding to all peers. Next, we use $\theta_{C_i}^t$ to define the vector stored on $i$-th peer after $t$ iterations of Moshpit Averaging. Then, for all $i = 1,\ldots,N$ we have $\theta_{C_i}^0 = \theta_{C_i}$ and for all $t = 1,\ldots,d$
\begin{equation*}
    \theta_{C_i}^{t} = \frac{1}{b_{i,t}}\sum\limits_{k\in J_{i,t}}\theta_{C_k}^{t-1},
\end{equation*}
where $J_{i,t} = \{k \in N\mid C_k = (c_1^k,\ldots,c_d^k) \in \cC \text{ and } c_j^k = c_j^i\; \forall j \neq t\}$ and $b_{i,t} = |J_{i,t}|$. Using this, we derive the following formula for $\theta_{C_i}^t$:
\begin{equation*}
    \theta_i^T \equiv \theta_{C_i}^T = \frac{1}{b_{i,T}}\sum\limits_{i_1\in J_{i,T}}\frac{1}{b_{i_1,T-1}}\sum\limits_{i_2\in J_{i_1,T-1}}\frac{1}{b_{i_2,T-2}}\sum\limits_{i_3\in J_{i_2,T-1}}\ldots\frac{1}{b_{i_{T-1},1}}\sum\limits_{i_T\in J_{i_{T-1},1}}\theta_{i_{T}}.
\end{equation*}
Taking the expectation w.r.t. $\theta_1,\ldots,\theta_N$, we get
\begin{equation*}
    \EE_{\theta}\left[\theta_i^T\right] = \frac{1}{b_{i,T}}\sum\limits_{i_1\in J_{i,T}}\frac{1}{b_{i_1,T-1}}\sum\limits_{i_2\in J_{i_1,T-1}}\frac{1}{b_{i_2,T-2}}\sum\limits_{i_3\in J_{i_2,T-1}}\ldots\frac{1}{b_{i_{T-1},1}}\sum\limits_{i_T\in J_{i_{T-1},1}}\overline{\theta}_{i_{T}}.
\end{equation*}
Using the independence of $\theta_1,\ldots,\theta_N$, we derive
\begin{eqnarray*}
    \EE_\theta\left[\left\|\theta_i^T - \EE_{\theta}\left[\theta_i^T\right]\right\|^2\right] &=& \EE_\theta\left[\left\|\sum\limits_{i_1\in J_{i,T}}\sum\limits_{i_2\in J_{i_1,T-1}}\ldots \sum\limits_{i_{T}\in J_{i_{T-1},1}}\frac{\theta_{i_T} - \overline{\theta}_{i_T}}{b_{i,T} b_{i_1,T-1}\ldots b_{i_{T-1},1}}\right\|^2\right]\\
    &=& \sum\limits_{i_1\in J_{i,T}}\sum\limits_{i_2\in J_{i_1,T-1}}\ldots \sum\limits_{i_{T}\in J_{i_{T-1},1}}\frac{\EE_\theta\left[\|\theta_{i_T} - \overline{\theta}_{i_T}\|^2\right]}{b_{i,T}^2 b_{i_1,T-1}^2\ldots b_{i_{T-1},1}^2}\\
    &\le& \sum\limits_{i_1\in J_{i,T}}\sum\limits_{i_2\in J_{i_1,T-1}}\ldots \sum\limits_{i_{T}\in J_{i_{T-1},1}}\frac{\sigma^2}{b_{i,T}^2 b_{i_1,T-1}^2\ldots b_{i_{T-1},1}^2}\\
    &=& \sum\limits_{i_1\in J_{i,T}}\sum\limits_{i_2\in J_{i_1,T-1}}\ldots \sum\limits_{i_{T-1}\in J_{i_{T-2},2}}\frac{\sigma^2}{b_{i,T}^2 b_{i_1,T-1}^2\ldots b_{i_{T-2},2}^2b_{i_{T-1},1}}.
\end{eqnarray*}
Next, taking the full expectation from the both sides of the previous inequality and using the tower property, we obtain
\begin{equation}
     \EE\!\left[\!\left\|\theta_i^T - \EE_{\theta}\left[\theta_i^T\right]\right\|^2\!\right] \!\le\! \EE\!\left[\!\sum\limits_{i_1\in J_{i,T}}\sum\limits_{i_2\in J_{i_1,T-1}}\ldots \sum\limits_{i_{T-1}\in J_{i_{T-2},2}}\frac{\sigma^2}{b_{i,T}^2 b_{i_1,T-1}^2\ldots b_{i_{T-2},2}^2b_{i_{T-1},1}}\!\right]\!. \label{eq:rand_mix_thm_technical_1}
\end{equation}
Notice that $J_{i_k,T-k} \cap J_{i_{k+1},T-k-1} = \{i_{k+1}\}$ for all $k=0,\ldots,T-1$, where $i_0 = i$. Moreover, for $k_1, k_2 \in\{0,1,\ldots,T\}$, $k_1 < k_2$ either $J_{i_{k_1},T-k_1} \cap J_{i_{k_2},T-k_2} = \{k_2\}$ or $J_{i_{k_1},T-k_1} \cap J_{i_{k_2},T-k_2} = \varnothing$. The first situation is possible iff $i_{k_1} = i_{k_1+1} = \ldots i_{k_2-1}$.

Taking these observations about sets $J_{i_{k}, T-k}$ into account, we consider the sets $J_{i_k,T-k}' = J_{i_k,T-k}\setminus\{i_{k}\}$ for $k = 0, 1, \ldots, T-1$. These sets are pairwise disjoint and their cardinalities $b_{i_k,T-k}' = |J_{i_k,T-k}'|$ satisfy the following relations: $b_{i_k,T-k} = 1 + b_{i_k,T-k}' \ge \max\{1, b_{i_k,T-k}'\} =: \hat{b}_{i_k,T-k}$ for $k = 1, 2, \ldots, T-1$. Moreover, $b_{i,T}', b_{i_1,T-1}',\ldots, b_{i_{T-1},1}'$ are independent random variables from the binomial distribution $\text{Binom}(M-1, p)$. Finally, we notice that the number of terms in \eqref{eq:rand_mix_thm_technical_1} is upper-bounded by $M^{T-1}$, since $|J_{i,t}| \le M$ for all $i = 1,\ldots,N$ and $t=0,\ldots,T$.

Putting all together, we obtain
\begin{eqnarray*}
    \EE\left[\left\|\theta_i^T - \EE_{\theta}\left[\theta_i^T\right]\right\|^2\right] &\le& \EE\left[\sum\limits_{i_1\in J_{i,T}}\sum\limits_{i_2\in J_{i_1,T-1}}\ldots \sum\limits_{i_{T-1}\in J_{i_{T-2},2}}\frac{\sigma^2}{\hat b_{i,T}^2 \hat b_{i_1,T-1}^2\ldots \hat b_{i_{T-2},2}^2\hat b_{i_{T-1},1}}\right]\\
    &\le& M^{T-1}\sigma^2\EE\left[\frac{1}{\hat\xi_{1}^2 \hat\xi_{2}^2\ldots \hat\xi_{T-1}^2\hat\xi_{T}}\right]\\
    &=& M^{T-1}\sigma^2\EE\left[\frac{1}{\hat\xi_{1}^2}\right]\EE\left[\frac{1}{\hat\xi_{2}^2}\right]\ldots \EE\left[\frac{1}{\hat\xi_{T-1}^2}\right]\EE\left[\frac{1}{\hat\xi_{T}}\right],
\end{eqnarray*}
where $\hat \xi_k^2 = \max\{1,\xi_1^2\}$ for $k=1,\ldots,T$ and $\xi_1,\ldots,\xi_T$ are i.i.d.\ random variables having the binomial distribution $\text{Binom}(M-1, p)$. Then one can simplify the inequality above using Lemma~\ref{lem:ode_lemma} and get
\begin{eqnarray*}
    \EE\left[\left\|\theta_i^T - \EE_{\theta}\left[\theta_i^T\right]\right\|^2\right] &\le& M^{T-1}\sigma^2 m_1(M-1,p)\left(m_2(M-1,p)\right)^{T-1},
\end{eqnarray*}
where functions $m_1(M,p)$ and $m_2(M,p)$ are defined in \eqref{eq:binom_first_inverse_moment} and \eqref{eq:binom_second_inverse_moment} respectively.

Next, we simplify the obtained upper bound under the assumption that $M$ and $p$ are not too small; specifically, $M\ge 11$ and $p\ge \nicefrac{2}{3}$. From \eqref{eq:binom_first_inverse_moment}, we have
\begin{eqnarray*}
    m_1(M-1,p) &=& (1-p)^{M-1} + \sum\limits_{i=1}^{M-1}\frac{1}{i}\left((1-p)^{M-1-i} - (1-p)^{M-1}\right)\\
    &\le& (1-p)^{M-1}\sum\limits_{i=1}^{M-1}\frac{1}{i(1-p)^{i}}.
\end{eqnarray*}
Since
\begin{equation*}
    \frac{1}{(k+1)(1-p)^{k+1}}\cdot\frac{k(1-p)^k}{1} = \frac{k}{(k+1)(1-p)} \xrightarrow[k\to\infty]{}\frac{1}{1-p} \ge 3,
\end{equation*}
we have
\begin{equation*}
    (1-p)^{M-1}\sum\limits_{i=1}^{M-1}\frac{1}{i(1-p)^{i}} = \Theta\left((1-p)^M\cdot\frac{1}{M(1-p)^M}\right) = \Theta\left(\frac{1}{M}\right).
\end{equation*}
Using simple algebra, one can prove that for $M\ge 11$ and $p \ge\nicefrac{2}{3}$ the following inequality holds:
\begin{equation*}
    m_1(M-1,p)\le (1-p)^{M-1}\sum\limits_{i=1}^{M-1}\frac{1}{i(1-p)^{i}} \le \frac{2}{M}.
\end{equation*}
Similarly, we analyze $m_2(M-1, p)$:
\begin{eqnarray*}
    m_2(M-1,p) &=& (1-p)^{M-1} + \sum\limits_{i=1}^{M-1}\frac{1}{i}\left((1-p)^{M-1-i} - (1-p)^{M-1}\right)\sum\limits_{j=i}^{M-1}\frac{1}{j}\\
    &\le& (1-p)^{M-1}\sum\limits_{i=1}^{M-1}\frac{1}{i(1-p)^i}\sum\limits_{j=i}^{M-1}\frac{1}{j}.
\end{eqnarray*}
Since
\begin{eqnarray*}
    \frac{\frac{1}{k(1-p)^k}\sum\limits_{j=k}^{M-1}\frac{1}{j}}{\frac{1}{(k-1)(1-p)^{k-1}}\sum\limits_{j=k-1}^{M-1}\frac{1}{j}} &=& \frac{(k-1)\sum\limits_{j=k}^{M-1}\frac{1}{j}}{k(1-p)\left(\frac{1}{k-1} + \sum\limits_{j=k}^{M-1}\frac{1}{j}\right)} \ge \frac{3(k-1)\cdot\frac{1}{k}}{k\left(\frac{1}{k-1}+\frac{1}{k}\right)}\\
    &=& \frac{3(k-1)^2}{k(2k-1)}\xrightarrow[k\to\infty]{}  \frac{3}{2},
\end{eqnarray*}
we have
\begin{equation*}
    (1-p)^{M-1}\sum\limits_{i=1}^{M-1}\frac{1}{i(1-p)^i}\sum\limits_{j=i}^{M-1}\frac{1}{j} = \Theta\left((1-p)^M\cdot\frac{1}{M^2(1-p)^M}\right) = \Theta\left(\frac{1}{M^2}\right).
\end{equation*}
Next, one can prove with simple algebra that for $M\ge 11$ and $p \ge\nicefrac{2}{3}$ the following inequality holds:
\begin{equation*}
    m_2(M-1,p) \le (1-p)^{M-1}\sum\limits_{i=1}^{M-1}\frac{1}{i(1-p)^i}\sum\limits_{j=i}^{M-1}\frac{1}{j} \le \frac{3}{M^2}.
\end{equation*}
Plugging the obtained upper bounds for $m_1(M-1,p)$ and $m_2(M-1,p)$ in \eqref{eq:variance_bound_supp}, we obtain \eqref{eq:variance_bound_2_supp}.
\end{proof}

\section{Convergence Proofs of Moshpit SGD}\label{sect:missing_proofs_local_sgd}
In this section, we provide the complete statements of the theorems establishing the convergence of Moshpit SGD together with the full proofs. First, we introduce all necessary definitions, basic inequalities and auxiliary lemmas; then we prove the convergence in strongly convex and convex cases; lastly, we provide the proofs for the non-convex case.

\subsection{Definitions, Basic Facts and Auxiliary Results}\label{sect:basic_facts}

Below we provide several classical definitions and results which are used in our proofs.

\subsubsection{Standard Definitions from Optimization Theory}

\begin{definition}[$L$-smoothness]\label{def:L_smoothness}
A function $f:\R^n \to \R$ is called $L$-smooth if for all $x,y\in \R^n$, the following inequality holds:
\begin{equation}
    \|\nabla f(x) - \nabla f(y)\| \le L\|x-y\|.\label{eq:L_smoothness_def}
\end{equation}
\end{definition}
If the function $f$ is $L$-smooth, then for all $x,y\in\R^n$
\begin{equation}
    f(y) \le f(x) + \langle\nabla f(x), y-x \rangle + \frac{L}{2}\|y-x\|^2. \label{eq:L_smoothness_cor}
\end{equation}
Next, if $f$ is additionally convex and $x^*$ is its minimizer, then for all $x\in\R^d$
\begin{equation}
    \|\nabla f(x)\|^2 \le 2L\left(f(x) - f(x^*)\right). \label{eq:L_smoothness_cor_2}
\end{equation}

\begin{definition}[$\mu$-strong convexity]\label{def:str_cvx}
    A differentiable function $f:\R^n \to\R$ is called $\mu$-strongly convex if there exists a constant $\mu \ge 0$ such that for all $x,y\in \R^n$
    \begin{equation}
        f(y) \ge f(x) + \langle\nabla f(x), y-x \rangle + \frac{\mu}{2}\|y-x\|^2. \label{eq:str_cvx_def}
    \end{equation}
\end{definition}

\subsubsection{Basic Facts}
For all $a,b,\theta_1,\ldots,\theta_N\in\R^n$ and $\alpha > 0$, the following inequalities hold:
\begin{eqnarray}
    \|a+b\|^2 &\le& 2\|a\|^2 + 2\|b\|^2, \label{eq:a+b}\\
    \left\|\frac{1}{N}\sum\limits_{i=1}^N\theta_i\right\|^2 &\le& \frac{1}{N}\sum\limits_{i=1}^N\|\theta_i\|^2, \label{eq:jensen_ineq}\\
    \langle a,b\rangle &\le& \frac{\|a\|^2}{2\alpha} + \frac{\alpha\|b\|^2}{2}. \label{eq:young_inequality}
\end{eqnarray}

\subsubsection{Properties of Expectation}
\textbf{Variance decomposition.} For a random vector $\eta \in \R^d$ and any deterministic vector $x \in \R^d$, the variance satisfies
\begin{equation}\label{eq:variance_decomposition}
	\EE\left[\left\|\eta - \EE\eta\right\|^2\right] = \EE\left[\|\eta-x\|^2\right] - \left\|\EE\eta - x\right\|^2
\end{equation}

\textbf{Tower property of expectation.} For any random variables $\xi,\eta\in \R^d$ we have
\begin{equation}
	\EE\left[\xi\right] = \EE\left[\EE\left[\xi\mid \eta\right]\right]\label{eq:tower_property}
\end{equation}
under the assumption that $\EE[\xi]$ and $\EE\left[\EE\left[\xi\mid \eta\right]\right]$ are well-defined.

\subsubsection{Auxiliary Results}
For the readers' convenience, we list all auxiliary results that we use in our proofs below. The first result is classical and establishes that the gradient descent step is a contractive operator.
\begin{lemma}[Lemma 6 from \cite{karimireddy2020scaffold}]\label{lem:gd_contraction}
    For any $L$-smooth and $\mu$-strongly convex function $f:\R^n\to\R$, points $x,y\in \R^n$, and stepsize $\gamma \in (0,\nicefrac{1}{L}]$, the following inequality holds:
    \begin{equation}
        \|x - \gamma\nabla f(x) - y + \gamma\nabla f(y)\|^2 \le (1-\gamma\mu)\|x-y\|^2. \label{eq:gd_contraction}
    \end{equation}
\end{lemma}

The next two lemmas are useful for estimating typical recurrences appearing in the analysis.
\begin{lemma}[Lemma~I.2 from \cite{gorbunov2020local}]\label{lem:lemma_i_2_gorbunov}
    Let $\{r_k\}_{k\ge 0}$ satisfy
    \begin{equation*}
        r_K \le \frac{a}{\gamma W_K} + c_1\gamma + c_2\gamma^2
    \end{equation*}
    for all $K \ge 0$ with some constants $a,c_2 \ge 0$, $c_1 \ge 0$, where $w_k = (1-\gamma\mu(1-\delta_{pv,1}))^{-(k+1)}$, $W_K = \sum_{k=0}^Kw_k$, $\mu > 0$, $\delta_{pv,1}\in [0,1)$ and $\gamma \le \gamma_0$ for some $\gamma_0 > 0$, $\gamma_0 \le \nicefrac{1}{\mu(1-\delta_{pv,1})}$. Then, for all $K$ such that
    \begin{align*}
        \text{either  } & \frac{\ln\left(\max\left\{2, \min\left\{\nicefrac{a\mu^2(1-\delta_{pv,1})^2K^2}{c_1},\nicefrac{a\mu^3(1-\delta_{pv,1})^3K^3}{c_2}\right\}\right\}\right)}{K} \le 1\\
        \text{or  } & \gamma_0 \le \frac{\ln\left(\max\left\{2, \min\left\{\nicefrac{a\mu^2(1-\delta_{pv,1})^2K^2}{c_1},\nicefrac{a\mu^3(1-\delta_{pv,1})^3K^3}{c_2}\right\}\right\}\right)}{(1-\delta_{pv,1})\mu K}
    \end{align*}
    and
    \begin{equation*}
        \gamma = \min\left\{\gamma_0, \frac{\ln\left(\max\left\{2, \min\left\{\nicefrac{a\mu^2(1-\delta_{pv,1})^2K^2}{c_1},\nicefrac{a\mu^3(1-\delta_{pv,1})^3K^3}{c_2}\right\}\right\}\right)}{(1-\delta_{pv,1})\mu K}\right\}
    \end{equation*}
    we have that
    \begin{equation*}
        r_K = \widetilde{\cO}\left(\frac{a}{\gamma_0}\exp\left(-\gamma_0\mu(1-\delta_{pv,1})K\right) + \frac{c_1}{(1-\delta_{pv,1})\mu K} + \frac{c_2}{(1-\delta_{pv,1})^2\mu^2 K^2}\right).
    \end{equation*}
\end{lemma}

\begin{lemma}[Lemma~I.3 from \cite{gorbunov2020local}]\label{lem:lemma_i_3_gorbunov}
    Let $\{r_k\}_{k\ge 0}$ satisfy
    \begin{equation*}
        r_K \le \frac{a}{\gamma K} + c_1\gamma + c_2\gamma^2
    \end{equation*}
    for all $K \ge 0$ with some constants $a,c_2 \ge 0$, $c_1 \ge 0$ where $\gamma \le \gamma_0$ for some $\gamma_0 > 0$. Then for all $K$ and
    \begin{equation*}
        \gamma = \min\left\{\gamma_0, \sqrt{\frac{a}{c_1 K}}, \sqrt[3]{\frac{a}{c_2 K}}\right\}
    \end{equation*}
    we have that
    \begin{equation*}
        r_K = \cO\left(\frac{a}{\gamma_0 K} + \sqrt{\frac{ac_1}{K}} + \frac{\sqrt[3]{a^2c_2}}{K^{\nicefrac{2}{3}}}\right).
    \end{equation*}
\end{lemma}

Finally, the lemma below is useful for our convergence analysis in the non-convex case.
\begin{lemma}[Lemma~I.1 from \cite{gorbunov2020local}]\label{lem:lemma_i_1_gorbunov}
	For any $\tau$ random vectors $\xi_1,\ldots,\xi_\tau\in\R^d$ such that $\forall t=2,\ldots,\tau$ the random vector $\xi_t$ depends on $\xi_{1},\ldots,\xi_{t-1}$ and does not depend on $\xi_{t+1},\ldots,\xi_{\tau}$ the following inequality holds
	\begin{equation}
		\EE\left[\left\|\sum\limits_{t=1}^\tau\xi_t\right\|^2\right] \le e\tau\sum\limits_{t=1}^\tau\EE\left[\left\|\EE_t[\xi_{t}]\right\|^2\right] + e\sum\limits_{t=1}^\tau\EE\left[\left\|\xi_t-\EE_t[\xi_{t}]\right\|^2\right], \label{eq:lemma_i_1_gorbunov}
	\end{equation}
	where $\EE_t[\cdot]$ denotes the conditional expectation $\EE[\ \cdot\mid \xi_{t-1},\ldots,\xi_1]$.
\end{lemma}

\subsection{Convex Case}
In this section, we give the full proof of Theorem~\ref{thm:cvx_convergence} about the convergence of Moshpit SGD for convex and strongly convex problems. The scheme of the proof follows the similar steps as in the state-of-the-art analysis of Local-SGD \cite{khaled2020tighter,woodworth2020local,gorbunov2020local}. We start with the following lemma:
\begin{lemma}\label{lem:key_lemma_cvx}
    Let $f_1 = \ldots = f_N = f$, function $f$ be $\mu$-strongly convex (Def.~\ref{def:str_cvx}) and $L$-smooth (see Def.~\ref{def:L_smoothness}), and Assumptions~\ref{as:bounded_var}~and~\ref{as:averaging_quality} hold with $\Delta_{pv}^k = \delta_{pv,1}\gamma\mu\EE[\|\theta^k-\theta^*\|^2] + \gamma^2\delta_{pv,2}^2$ and $\widetilde{\theta} = \theta^*$, where $\theta^* \in \argmin_{\theta\in\R^n} f(\theta)$ and $\delta_{pv,1}\in [0,1)$, $\delta_{pv,2}\ge 0$. Then, for any $k \ge 0$ the iterates produced by Moshpit SGD with $\gamma \le \nicefrac{1}{4L}$ satisfy
    \begin{eqnarray}
        \gamma\EE\left[f(\theta^k) - f(\theta^*)\right] &\le& (1-\gamma\mu(1-\delta_{pv,1}))\EE\left[\|\theta^k - \theta^*\|^2\right] - \EE\left[\|\theta^{k+1} - \theta^*\|^2\right]\notag\\
        &&\quad+ \frac{3L\gamma}{2}\EE[V_k] + \gamma^2\left(\frac{\sigma^2}{N_{\min}} + \delta_{pv,2}^2\right),\label{eq:key_lemma_cvx}
    \end{eqnarray}
    where $V_k = \frac{1}{N_k}\sum_{i\in P_k}\|\theta_i^k - \theta^k\|^2$ and $\theta^k = \frac{1}{N_k}\sum_{i\in P_k}\theta_i^k$.
\end{lemma}
\begin{proof}
Recall that Assumption~\ref{as:averaging_quality} with $\Delta_{pv}^k = \delta_{pv,1}\gamma\mu\EE[\|\theta^k-\theta^*\|^2] + \gamma^2\delta_{pv,2}^2$ and $\widetilde{\theta} = \theta^*$ states
\begin{equation}
    \EE\left[\langle\theta^{k+1} - \widehat{\theta}^{k+1}, \theta^{k+1}+\widehat{\theta}^{k+1} - 2\theta^*\rangle\right] \le \delta_{pv,1}\gamma\mu\EE[\|\theta^k-\theta^*\|^2] + \gamma^2\delta_{pv,2}^2, \label{eq:key_lemma_cvx_tech_1}
\end{equation}
where $\widehat \theta^{k+1} = \frac{1}{N_{k}}\sum_{i\in P_{k}}(\theta_i^{k}-\gamma g_i^k)$. Next, the definition of $\widehat \theta^{k+1}$ implies
\begin{equation}
    \widehat \theta^{k+1} = \frac{1}{N_k}\sum\limits_{i\in P_{k}}\theta_i^{k} - \frac{\gamma}{N_k}\sum\limits_{i\in P_{k}} g_i^k = \theta^k - \gamma g^k,\notag
\end{equation}
where $g^k = \frac{1}{N_k}\sum_{i\in P_k}g_i^k$. Using this, we derive
\begin{eqnarray}
    \|\theta^{k+1} - \theta^*\|^2 &=& \|\widehat{\theta}^{k+1} - \theta^*\|^2 + 2\langle \theta^{k+1} - \widehat{\theta}^{k+1}, \widehat{\theta}^{k+1} - \theta^* \rangle + \|\theta^{k+1} - \widehat{\theta}^{k+1}\|^2\notag\\
    &=& \|\theta^k - \theta^* - \gamma g^k\|^2 +  \langle\theta^{k+1} - \widehat{\theta}^{k+1}, \theta^{k+1}+\widehat{\theta}^{k+1} - 2\theta^*\rangle \notag\\
    &=& \|\theta^k - \theta^*\|^2 -2\gamma\langle\theta^k - \theta^*, g^k\rangle + \gamma^2\|g^k\|^2\notag\\
    &&\quad +  \langle\theta^{k+1} - \widehat{\theta}^{k+1}, \theta^{k+1}+\widehat{\theta}^{k+1} - 2\theta^*\rangle. \notag
\end{eqnarray}
Taking the conditional expectation $\EE\left[\ \cdot \mid \theta^k\right] := \EE\left[\ \cdot \mid P_k, \theta_i^k, i\in P_k\right]$ from the both sides of the previous equation and using Assumption~\ref{as:bounded_var}, we obtain
\begin{eqnarray}
    \EE\left[\|\theta^{k+1} - \theta^*\|^2\mid \theta^k\right] &=& \|\theta^k - \theta^*\|^2 -2\gamma\left\langle\theta^k - \theta^*, \frac{1}{N_k}\sum\limits_{i\in P_k}\nabla f(\theta_i^k)\right\rangle\notag\\
    &&\quad + \gamma^2\EE\left[\left\|\frac{1}{N_k}\sum\limits_{i\in P_k}g_i^k\right\|^2\mid \theta^k\right] \notag\\
    &&\quad +  \EE\left[\langle\theta^{k+1} - \widehat{\theta}^{k+1}, \theta^{k+1}+\widehat{\theta}^{k+1} - 2\theta^*\rangle\mid \theta^k\right]. \label{eq:key_lemma_cvx_tech_2}
\end{eqnarray}
Next, we estimate the second and the third terms in the right-hand side of \eqref{eq:key_lemma_cvx_tech_2}. First,
\begin{eqnarray}
    -2\gamma\left\langle\theta^k - \theta^*, \frac{1}{N_k}\sum\limits_{i\in P_k}\nabla f(\theta_i^k)\right\rangle &=& \frac{2\gamma}{N_k}\sum\limits_{i\in P_k}\left(\langle\theta^* - \theta_i^k, \nabla f(\theta_i^k) \rangle + \langle\theta_i^k - \theta^k, \nabla f(\theta_i^k) \rangle \right)\notag\\
    &\overset{\eqref{eq:str_cvx_def},\eqref{eq:L_smoothness_cor}}{\le}& \frac{2\gamma}{N_k}\sum\limits_{i\in P_k}\left( f(\theta^*) - f(\theta_i^k) - \frac{\mu}{2}\|\theta_i^k - \theta^*\|^2\right)\notag\\
    &&\quad + \frac{2\gamma}{N_k}\sum\limits_{i\in P_k}\left(f(\theta_i^k) - f(\theta^k) + \frac{L}{2}\|\theta_i^k - \theta^k\|^2\right)\notag\\
    &\overset{\eqref{eq:jensen_ineq}}{\le}& 2\gamma\left(f(\theta^*) - f(\theta^k)\right) -\gamma\mu\|\theta^k - \theta^*\|^2 + L\gamma V_k, \label{eq:key_lemma_cvx_tech_3}
\end{eqnarray}
where $V_k = \frac{1}{N_k}\sum_{i\in P_k}\|\theta_i^k - \theta^k\|^2$. Secondly, since stochastic gradients $\{g_i^k\}_{i\in P_k}$ are computed independently, we get
\begin{eqnarray}
    \gamma^2\EE\left[\left\|\frac{1}{N_k}\sum\limits_{i\in P_k}g_i^k\right\|^2\mid \theta^k\right] &\overset{\eqref{eq:variance_decomposition}}{=}& \gamma^2\left\|\frac{1}{N_k}\sum\limits_{i\in P_k}\nabla f(\theta_i^k)\right\|^2\notag\\
    &&\quad + \gamma^2\EE\left[\left\|\frac{1}{N_k}\sum\limits_{i\in P_k}(g_i^k-\nabla f(\theta_i^k))\right\|^2\mid \theta^k\right]\notag\\
    &\overset{\eqref{eq:jensen_ineq}}{\le}& 2\gamma^2 \left\|\frac{1}{N_k}\sum\limits_{i\in P_k}(\nabla f(\theta_i^k)-\nabla f(\theta^k))\right\|^2 + 2\gamma^2\|\nabla f(\theta^k)\|^2 \notag\\
    &&\quad + \frac{\gamma^2}{N_k^2}\sum\limits_{i\in P_k}\EE\left[\|g_i^k - \nabla f(\theta_i^k)\|^2\mid \theta^k\right]\notag\\
    &\overset{\eqref{eq:jensen_ineq},\eqref{eq:L_smoothness_cor_2},\eqref{eq:bounded_variance}}{\le}& \frac{2\gamma^2}{N_k}\sum\limits_{i\in P_k}\|\nabla f(\theta_i^k)-\nabla f(\theta^k)\|^2 \notag\\
    &&\quad + 4L\gamma^2\left(f(\theta^k) - f(\theta^*)\right) + \frac{\gamma^2\sigma^2}{N_k}\notag\\
    &\overset{\eqref{eq:L_smoothness_def}}{\le}& \underbrace{\frac{2L^2\gamma^2}{N_k}\sum\limits_{i\in P_k}\|\theta_i^k - \theta^k\|^2}_{2L^2\gamma^2 V_k}\notag\\
    &&\quad + 4L\gamma^2\left(f(\theta^k) - f(\theta^*)\right) + \frac{\gamma^2\sigma^2}{N_{\min}}. \label{eq:key_lemma_cvx_tech_4}
\end{eqnarray}
Plugging \eqref{eq:key_lemma_cvx_tech_3} and \eqref{eq:key_lemma_cvx_tech_4} in \eqref{eq:key_lemma_cvx_tech_2}, we obtain
\begin{eqnarray}
    \EE\left[\|\theta^{k+1} - \theta^*\|^2\mid \theta^k\right] &\le& (1-\gamma\mu)\|\theta^k - \theta^*\|^2 - 2\gamma\left(1 - 2L\gamma\right)\left(f(\theta^k) - f(\theta^*)\right)\notag\\
    &&\quad + L\gamma\left(1+2L\gamma\right)V_k + \frac{\gamma^2\sigma^2}{N_{\min}} \notag\\
    &&\quad +  \EE\left[\langle\theta^{k+1} - \widehat{\theta}^{k+1}, \theta^{k+1}+\widehat{\theta}^{k+1} - 2\theta^*\rangle\mid \theta^k\right], \notag
\end{eqnarray}
and
\begin{eqnarray}
    \EE\left[\|\theta^{k+1} - \theta^*\|^2\right] &\overset{\eqref{eq:key_lemma_cvx_tech_1}}{\le}& (1-\gamma\mu(1-\delta_{pv,1}))\EE\left[\|\theta^k - \theta^*\|^2\right] - 2\gamma\left(1 - 2L\gamma\right)\EE\left[f(\theta^k) - f(\theta^*)\right]\notag\\
    &&\quad+ L\gamma\left(1+2L\gamma\right)\EE[V_k] + \gamma^2\left(\frac{\sigma^2}{N_{\min}} + \delta_{pv,2}^2\right)\notag\\
    &\le& (1-\gamma\mu(1-\delta_{pv,1}))\EE\left[\|\theta^k - \theta^*\|^2\right] - \gamma\EE\left[f(\theta^k) - f(\theta^*)\right]\notag\\
    &&\quad+ \frac{3L\gamma}{2}\EE[V_k] + \gamma^2\left(\frac{\sigma^2}{N_{\min}} + \delta_{pv,2}^2\right),\notag
\end{eqnarray}
where in the last inequality we use $\gamma \le \nicefrac{1}{4L}$.
\end{proof}

Next, we estimate the term $\EE[V_k]$ measuring the expected dissimilarity between local iterates and their global average at iteration $k$.

\begin{lemma}\label{lem:V_k_lemma_cvx}
    Let $f_1 = \ldots = f_N = f$, function $f$ be $\mu$-strongly convex (Def.~\ref{def:str_cvx}) and $L$-smooth (see Def.~\ref{def:L_smoothness}), and Assumptions~\ref{as:bounded_var}~and~\ref{as:averaging_quality} hold with $\Delta_{pv}^k = \delta_{pv,1}\gamma\mu\EE[\|\theta^k-\theta^*\|^2] + \gamma^2\delta_{pv,2}^2$ and $\widetilde{\theta} = \theta^*$, where $\theta^* \in \argmin_{\theta\in\R^n} f(\theta)$ and $\delta_{pv,1}\in [0,1)$, $\delta_{pv,2}\ge 0$. Then, for any $k \ge 0$ the iterates produced by Moshpit SGD with $\gamma \le \nicefrac{1}{4L}$ satisfy
    \begin{equation}
        \EE[V_k] \le 2\gamma^2\left(4\delta_{aq}^2 + (\tau-1)\sigma^2\right), \label{eq:V_k_bound_cvx}
    \end{equation}
    where $V_k = \frac{1}{N_k}\sum_{i\in P_k}\|\theta_i^k - \theta^k\|^2$ and $\theta^k = \frac{1}{N_k}\sum_{i\in P_k}\theta_i^k$.
\end{lemma}
\begin{proof}
    First of all, if $k = a\tau$ for some integer $a\ge 0$, then \eqref{eq:V_k_bound_cvx} follows from Assumption~\ref{as:averaging_quality} (eq.~\eqref{eq:quality_of_avg}). Therefore, we consider such $k$ that $k = a\tau + t'$ for some $t'\in (0,\tau)$. Then, for any $i,j \in P_{k}$, $i\neq j$
    \begin{eqnarray*}
        \EE\left[\|\theta_i^k - \theta_j^k\|^2\mid \theta^{k-1}\right] &=& \EE\left[\|\theta_i^{k-1} - \gamma g_i^{k-1} - \theta_j^{k-1} + \gamma g_{j}^{k-1}\|^2\mid \theta^{k-1}\right]\\
        &\overset{\eqref{eq:variance_decomposition}}{=}& \|\theta_i^{k-1} - \gamma \nabla f(\theta_i^{k-1}) - \theta_j^{k-1} + \gamma \nabla f(\theta_j^{k-1})\|^2\\
        &&\quad +\gamma^2\EE\left[\|g_i^{k-1} - \nabla f(\theta_i^{k-1}) + g_{j}^{k-1} - \nabla f(\theta_j^{k-1})\|^2\mid \theta^{k-1}\right].
    \end{eqnarray*}
    Using Lemma~\ref{lem:gd_contraction} and independence of $g_i^{k-1}$ and $g_j^{k-1}$ for given $\theta_i^{k-1}, \theta_j^{k-1}$, $i\neq j$ we derive
    \begin{eqnarray*}
        \EE\left[\|\theta_i^k - \theta_j^k\|^2\mid \theta^{k-1}\right] &\overset{\eqref{eq:gd_contraction}}{\le}& (1-\gamma\mu)\|\theta_i^{k-1} - \theta_j^{k-1}\|^2 +\gamma^2\EE\left[\|g_i^{k-1} - \nabla f(\theta_i^{k-1})\|^2\mid \theta^{k-1}\right]\\
        &&\quad +\gamma^2\EE\left[\|g_j^{k-1} - \nabla f(\theta_j^{k-1})\|^2\mid \theta^{k-1}\right]\\
        &\overset{\eqref{eq:bounded_variance}}{\le}& (1-\gamma\mu)\|\theta_i^{k-1} - \theta_j^{k-1}\|^2 + 2\gamma^2\sigma^2,
    \end{eqnarray*}
    from which we get the following: 
    \begin{equation}
        \EE_g\left[\|\theta_i^k - \theta_j^k\|^2\right] \le (1-\gamma\mu)\EE_g\left[\|\theta_i^{k-1} - \theta_j^{k-1}\|^2\right] + 2\gamma^2\sigma^2 \le \EE_g\left[\|\theta_i^{k-1} - \theta_j^{k-1}\|^2\right] + 2\gamma^2\sigma^2.\notag %
    \end{equation}
    Here, $\EE_g[\cdot]$ denotes the expectation conditioned on $\{P_k\}_{k = a\tau}^{(a+1)\tau-1}$. Unrolling the recurrence, we get
    \begin{eqnarray}
        \EE_g\left[\|\theta_i^k - \theta_j^k\|^2\right] &\le& \EE_g\left[\|\theta_i^{a\tau} - \theta_j^{a\tau}\|^2\right] + 2(k-a\tau)\gamma^2\sigma^2\notag \\
        &\le& \EE_g\left[\|\theta_i^{a\tau} - \theta_j^{a\tau}\|^2\right] + 2(\tau-1)\gamma^2\sigma^2.\label{eq:V_k_lemma_technical_1}
    \end{eqnarray}
    Using this, we estimate $\EE_{g}[V_k]$:
    \begin{eqnarray*}
        \EE_g[V_k] &=& \frac{1}{N_k}\sum\limits_{i\in P_k}\EE_g\left[\left\|\theta_i^k - \frac{1}{N_k}\sum\limits_{j\in P_k}\theta_j^k\right\|^2\right] \overset{\eqref{eq:jensen_ineq}}{\le} \frac{1}{N_k^2}\sum\limits_{i,j \in P_k}\EE_g\left[\|\theta_i^k - \theta_j^k\|^2\right]\\
        &\overset{\eqref{eq:V_k_lemma_technical_1}}{\le}& \frac{1}{N_k^2}\sum\limits_{i,j \in P_k}\EE_g\left[\|\theta_i^{a\tau} - \theta_j^{a\tau}\|^2\right] + 2(\tau-1)\gamma^2\sigma^2 \\
        &\overset{\eqref{eq:a+b}}{\le}& \frac{2}{N_k^2}\sum\limits_{i,j \in P_k}\left(\EE_g\left[\|\theta_i^{a\tau} - \theta^{a\tau}\|^2\right] + \EE_g\left[\|\theta_j^{a\tau} - \theta^{a\tau}\|^2\right]\right) + 2(\tau-1)\gamma^2\sigma^2\\
        &=& \frac{4}{N_k}\sum\limits_{i\in P_k}\EE_g\left[\|\theta_i^{a\tau} - \theta^{a\tau}\|^2\right]+ 2(\tau-1)\gamma^2\sigma^2\\
        &\le& \frac{4}{N_{a\tau}}\cdot\frac{N_{a\tau}}{N_k}\sum\limits_{i\in P_{a\tau}}\EE_g\left[\|\theta_i^{a\tau} - \theta^{a\tau}\|^2\right]+ 2(\tau-1)\gamma^2\sigma^2\\
        &\le& \EE_g\left[\frac{8}{N_{a\tau}}\sum\limits_{i\in P_{a\tau}}\|\theta_i^{a\tau} - \theta^{a\tau}\|^2\right]+ 2(\tau-1)\gamma^2\sigma^2,
    \end{eqnarray*}
    where in the last inequality we use $2N_{(a+1)\tau} = 2|P_{(a+1)\tau}| \ge |P_{a\tau}| = N_{a\tau}$ and $|N_k|\le |N_{k-1}|$ following from Assumption~\ref{as:averaging_quality}. Finally, we take the full expectation from the previous inequality:
    \begin{eqnarray*}
        \EE[V_k] &\overset{\eqref{eq:tower_property}}{\le}& 8\EE\left[\frac{1}{N_{a\tau}}\sum\limits_{i\in P_{a\tau}}\|\theta_i^{a\tau} - \theta^{a\tau}\|^2\right]+ 2(\tau-1)\gamma^2\sigma^2 \overset{\eqref{eq:quality_of_avg}}{\le} 2\gamma^2\left(4\delta_{aq}^2 + (\tau-1)\sigma^2\right).
    \end{eqnarray*}
    This finishes the proof.
\end{proof}

Combining Lemmas~\ref{lem:key_lemma_cvx}~and~\ref{lem:V_k_lemma_cvx}, we get the following result:
\begin{theorem}[Theorem~\ref{thm:cvx_convergence}, convergence in the convex case]\label{thm:cvx_convergence_supp}
    Let $f_1 = \ldots = f_N = f$ be $\mu$-strongly convex (Def.~\ref{def:str_cvx}) and $L$-smooth (see Def.~\ref{def:L_smoothness}), and Assumptions~\ref{as:bounded_var}~and~\ref{as:averaging_quality} hold with $\Delta_{pv}^k = \delta_{pv,1}\gamma\mu\EE[\|\theta^k-\theta^*\|^2] + \gamma^2\delta_{pv,2}^2$ and $\widetilde{\theta} = \theta^*$, where $\theta^* \in \argmin_{\theta\in\R^n} f(\theta)$ and $\delta_{pv,1}\in [0,1)$, $\delta_{pv,2}\ge 0$. Then, for any $K \ge 0$, the iterates produced by Moshpit SGD with $\gamma \le \nicefrac{1}{4L}$ satisfy
    \begin{eqnarray}
        \EE\left[f(\overline{\theta}^K) - f(\theta^*)\right] &\le& (1-\gamma\mu(1-\delta_{pv,1}))^K\frac{R_0^2}{\gamma}\notag\\
        &&\quad + \gamma\left(\frac{\sigma^2}{N_{\min}} + \delta_{pv,2}^2 + 3L\gamma\left(4\delta_{aq}^2 + (\tau-1)\sigma^2\right)\right), \label{eq:str_cvx_bound_supp}
    \end{eqnarray}
    when $\mu > 0$, and
    \begin{equation}
        \EE\left[f(\overline{\theta}^K) - f(\theta^*)\right] \le \frac{R_0^2}{\gamma K} + \gamma\left(\frac{\sigma^2}{N_{\min}} + \delta_{pv,2}^2 + 3L\gamma\left(4\delta_{aq}^2 + (\tau-1)\sigma^2\right)\right), \label{eq:cvx_bound_supp}
    \end{equation}
    when $\mu = 0$, where $R_0 = \|\theta^0 - \theta^*\|$, $\overline{\theta}^K = \frac{1}{W_K}\sum_{k=0}^Kw_k\theta^k = \frac{1}{W_K}\sum_{k=0}^K\frac{w_k}{N_k}\sum_{i\in P_k}\theta_i^k$, $w_k = (1-\gamma\mu(1-\delta_{pv,1}))^{-(k+1)}$, and $W_K = \sum_{k=0}^Kw_k$. That is, Moshpit SGD achieves $\EE[f(\overline{\theta}^K) - f(\theta^*)] \le \varepsilon$ after 
    \begin{equation}
        K = \widetilde{\cO}\left(\frac{L}{(1-\delta_{pv,1})\mu} +  \frac{\sigma^2}{N_{\min}(1-\delta_{pv,1})\mu\varepsilon} + \frac{\delta_{pv,2}^2}{(1-\delta_{pv,1})\mu\varepsilon} + \sqrt{\frac{L((\tau-1)\sigma^2+\delta_{aq}^2)}{(1-\delta_{pv,1})^2\mu^2\varepsilon}}\right)\label{eq:str_cvx_bound_2_supp}
    \end{equation}
    iterations with
    \begin{equation*}
        \gamma = \min\left\{\frac{1}{4L}, \frac{\ln\left(\max\left\{2, \min\left\{\frac{R_0^2\mu^2(1-\delta_{pv,1})^2K^2}{(\delta_{pv,2}^2 + \nicefrac{\sigma^2}{N_{\min}}) },\frac{R_0^2\mu^3(1-\delta_{pv,1})^3K^3}{3L\left(4\delta_{aq}^2 + (\tau-1)\sigma^2\right)}\right\}\right\}\right)}{(1-\delta_{pv,1})\mu K}\right\}
    \end{equation*}
    when $\mu > 0$, and after
    \begin{equation}
        K = \cO\left(\frac{LR_0^2}{\varepsilon} +  \frac{R_0^2\sigma^2}{N_{\min}\varepsilon^2} + \frac{R_0^2\delta_{pv,2}^2}{\varepsilon^2} + \frac{R_0^2\sqrt{L((\tau-1)\sigma^2+\delta_{aq}^2)}}{\varepsilon^{\nicefrac{3}{2}}}\right)\label{eq:cvx_bound_2_supp}
    \end{equation}
    iterations with
    \begin{equation*}
       \gamma = \min\left\{\frac{1}{4L} \sqrt{\frac{R_0}{(\delta_{pv,2}^2 + \nicefrac{\sigma^2}{N_{\min}})K}}, \sqrt[3]{\frac{R_0^2}{3L\left(4\delta_{aq}^2 + (\tau-1)\sigma^2\right) K}}\right\}
    \end{equation*}
    when $\mu = 0$.
\end{theorem}
\begin{proof}
    Plugging the result of Lemma~\ref{lem:V_k_lemma_cvx} in inequality \eqref{eq:key_lemma_cvx} from Lemma~\ref{lem:key_lemma_cvx}, we obtain
    \begin{eqnarray}
        \gamma\EE\left[f(\theta^k) - f(\theta^*)\right] &\le& (1-\gamma\mu(1-\delta_{pv,1}))\EE\left[\|\theta^k - \theta^*\|^2\right] - \EE\left[\|\theta^{k+1} - \theta^*\|^2\right]\notag\\
        &&\quad+ 3L\gamma^3\left(4\delta_{aq}^2 + (\tau-1)\sigma^2\right) + \gamma^2\left(\frac{\sigma^2}{N_{\min}} + \delta_{pv,2}^2\right).\notag
    \end{eqnarray}
    Next, we sum up these inequalities for $k=0,\ldots, K$ with weights $w_k = (1-\gamma\mu(1-\delta_{pv,1}))^{-(k+1)}$ and divide both sides by $\gamma W_K$, where $W_K = \sum_{k=0}^Kw_k$:
    \begin{eqnarray*}
        \frac{1}{W_K}\sum\limits_{k=0}^K w_k\EE\left[f(\theta^k) - f(\theta^*)\right] &\le& \frac{1}{\gamma W_K}\sum\limits_{k=0}^K(1-\gamma\mu(1-\delta_{pv,1}))w_k\EE\left[\|\theta^k - \theta^*\|^2\right]\notag\\
        &&\quad - \frac{1}{\gamma W_K}\sum\limits_{k=0}^K w_k\EE\left[\|\theta^{k+1} - \theta^*\|^2\right]\notag\\
        &&\quad+ \gamma\left(\frac{\sigma^2}{N_{\min}} + \delta_{pv,2}^2 + 3L\gamma\left(4\delta_{aq}^2 + (\tau-1)\sigma^2\right)\right)\\
        &=& \frac{1}{\gamma W_K}\sum\limits_{k=0}^K\left(w_{k-1}\EE\left[\|\theta^k - \theta^*\|^2\right] - w_k\EE\left[\|\theta^{k+1} - \theta^*\|^2\right]\right)\notag\\
        &&\quad+ \gamma\left(\frac{\sigma^2}{N_{\min}} + \delta_{pv,2}^2 + 3L\gamma\left(4\delta_{aq}^2 + (\tau-1)\sigma^2\right)\right)\\
        &=& \frac{w_{-1}\|\theta^0 - \theta^*\|^2 - w_K\EE\left[\|\theta^{K+1}-\theta^*\|^2\right]}{\gamma W_K}\\
        &&\quad+ \gamma\left(\frac{\sigma^2}{N_{\min}} + \delta_{pv,2}^2 + 3L\gamma\left(4\delta_{aq}^2 + (\tau-1)\sigma^2\right)\right)\\
        &\le& \frac{\|\theta^0 - \theta^*\|^2}{\gamma W_K} \\
        &&\quad + \gamma\left(\frac{\sigma^2}{N_{\min}} + \delta_{pv,2}^2 + 3L\gamma\left(4\delta_{aq}^2 + (\tau-1)\sigma^2\right)\right).
    \end{eqnarray*}
    Since $f$ is convex, we apply the Jensen's inquality
    \begin{eqnarray*}
        f\left(\frac{1}{W_K}\sum\limits_{k=0}^K w_k\theta^k\right) &\le& \frac{1}{W_K}\sum\limits_{k=0}^K w_k f(\theta^k)
    \end{eqnarray*}
    to the previous result and get
    \begin{eqnarray*}
        \EE\left[f(\overline{\theta}^K) - f(\theta^*)\right] &\le& \frac{R_0^2}{\gamma W_K} + \gamma\left(\frac{\sigma^2}{N_{\min}} + \delta_{pv,2}^2 + 3L\gamma\left(4\delta_{aq}^2 + (\tau-1)\sigma^2\right)\right),
    \end{eqnarray*}
    where $R_0 = \|\theta^0 - \theta^*\|$ and $\overline{\theta}^K = \frac{1}{W_K}\sum_{k=0}^Kw_k\theta^k = \frac{1}{W_K}\sum_{k=0}^K\frac{w_k}{N_k}\sum_{i\in P_k}\theta_i^k$. If $\mu > 0$, then $W_K \ge w_K \ge (1-\gamma\mu(1-\delta_{pv,1}))^{-K}$, implying \eqref{eq:str_cvx_bound_supp}. Next, $w_k = 1$ and $W_K = K$ when $\mu = 0$ gives \eqref{eq:cvx_bound_supp}. It remains to estimate the total number of iterations $K$ required by Moshpit SGD to find an $\varepsilon$-solution, i.e., to achieve $\EE[f(\overline{\theta}^K) - f(\theta^*)] \le \varepsilon$. Applying Lemma~\ref{lem:lemma_i_2_gorbunov} to \eqref{eq:str_cvx_bound_supp}, we get the following result: if $\mu > 0$ and 
    \begin{equation*}
        \gamma = \min\left\{\frac{1}{4L}, \frac{\ln\left(\max\left\{2, \min\left\{\frac{R_0^2\mu^2(1-\delta_{pv,1})^2K^2}{\delta_{pv,2}^2 + \nicefrac{\sigma^2}{N_{\min}} },\frac{R_0^2\mu^3(1-\delta_{pv,1})^3K^3}{3L\left(4\delta_{aq}^2 + (\tau-1)\sigma^2\right)}\right\}\right\}\right)}{(1-\delta_{pv,1})\mu K}\right\},
    \end{equation*}
    then $\EE\left[f(\overline{\theta}^K) - f(\theta^*)\right]$ equals
    \begin{equation*}
        \widetilde{\cO}\left(LR_0^2\exp\left(-\frac{\mu}{L}(1-\delta_{pv,1})K\right) + \frac{\delta_{pv,2}^2 + \nicefrac{\sigma^2}{N_{\min}}}{(1-\delta_{pv,1})\mu K} + \frac{L\left(\delta_{aq}^2 + (\tau-1)\sigma^2\right)}{(1-\delta_{pv,1})^2\mu^2 K^2}\right),
    \end{equation*}
    implying \eqref{eq:str_cvx_bound_2_supp}. Similarly, we apply Lemma~\ref{lem:lemma_i_3_gorbunov} to \eqref{eq:cvx_bound_supp} and get that for $\mu = 0$ and 
    \begin{equation*}
        \gamma = \min\left\{\frac{1}{4L} \sqrt{\frac{R_0}{(\delta_{pv,2}^2 + \nicefrac{\sigma^2}{N_{\min}})K}}, \sqrt[3]{\frac{R_0^2}{3L\left(4\delta_{aq}^2 + (\tau-1)\sigma^2\right) K}}\right\},
    \end{equation*}
    \begin{equation*}
        \EE\left[f(\overline{\theta}^K) - f(\theta^*)\right] = \cO\left(\frac{LR_0^2}{K} + \sqrt{\frac{R_0^2(\delta_{pv,2}^2 + \nicefrac{\sigma^2}{N_{\min}})}{K}} + \frac{\sqrt[3]{R_0^4L\left(\delta_{aq}^2 + (\tau-1)\sigma^2\right)}}{K^{\nicefrac{2}{3}}}\right),
    \end{equation*}
    implying \eqref{eq:cvx_bound_2_supp}.
\end{proof}

\subsection{Non-Convex Case}
In this section, we give the full proof of Theorem~\ref{thm:non_cvx_convergence} about convergence of Moshpit SGD for general non-convex problems. The proof follows the similar steps as in the state-of-the-art analysis of Local-SGD in non-convex case~\cite{li2019communication,koloskova2020unified}. We start with the following lemma:
\begin{lemma}\label{lem:key_lemma_non_cvx}
    Let $f_1 = \ldots = f_N = f$, function $f$ be $L$-smooth and bounded from below by $f_*$, and Assumptions~\ref{as:bounded_var}~and~\ref{as:averaging_quality} hold with $\Delta_{pv}^k = \delta_{pv,1}\gamma\EE[\|\nabla f(\theta^k)\|^2] + L\gamma^2\delta_{pv,2}^2$, $\delta_{pv,1}\in [0,\nicefrac{1}{2})$, $\delta_{pv,2}\ge 0$. Then, for any $K \ge 0$ the iterates produced by Moshpit SGD with $\gamma \le \nicefrac{(1-2\delta_{pv,1})}{8L}$ satisfy
    \begin{eqnarray}
         \frac{(1-2\delta_{pv,1})\gamma}{4}\sum\limits_{k=0}^{K-1}\EE\left[\|\nabla f(\theta^k)\|^2\right] &\le& f(\theta^0) - f_* + \gamma L^2\sum\limits_{k=0}^{K-1} \EE[V_k]\notag\\
         &&\quad + KL\gamma^2\left(\frac{\sigma^2}{N_{\min}} + \delta_{pv,2}^2\right),\label{eq:key_lemma_non_cvx}
    \end{eqnarray}
    where $V_k = \frac{1}{N_k}\sum_{i\in P_k}\|\theta_i^k - \theta^k\|^2$ and $\theta^k = \frac{1}{N_k}\sum_{i\in P_k}\theta_i^k$.
\end{lemma}
\begin{proof}
    Recall that Assumption~\ref{as:averaging_quality} with $\Delta_{pv}^k = \delta_{pv,1}\gamma\EE[\|\nabla f(\theta^k)\|^2] + L\gamma^2\delta_{pv,2}^2$ states
\begin{equation}
    \EE\left[\langle\nabla f(\theta^k), \theta^{k+1}-\widehat{\theta}^{k+1}\rangle + L\|\widehat{\theta}^{k+1} - \theta^{k+1}\|^2\right] \le \delta_{pv,1}\gamma\EE[\|\nabla f(\theta^k)\|^2] + L\gamma^2\delta_{pv,2}^2, \label{eq:key_lemma_non_cvx_tech_1}
\end{equation}
where $\widehat \theta^{k+1} = \frac{1}{N_{k}}\sum_{i\in P_{k}}(\theta_i^{k}-\gamma g_i^k)$. As for the convex case, the definition of $\widehat \theta^{k+1}$ implies
\begin{equation}
    \widehat \theta^{k+1} = \frac{1}{N_k}\sum\limits_{i\in P_{k}}\theta_i^{k} - \frac{\gamma}{N_k}\sum\limits_{i\in P_{k}} g_i^k = \theta^k - \gamma g^k,\notag
\end{equation}
where $g^k = \frac{1}{N_k}\sum_{i\in P_k}g_i^k$. Using this and L-smoothness of $f$, we derive
    \begin{eqnarray*}
        f(\theta^{k+1}) - f(\theta^k) &\overset{\eqref{eq:L_smoothness_cor}}{\le}& \langle\nabla f(\theta^k), \theta^{k+1} - \theta^k \rangle + \frac{L}{2}\|\theta^{k+1} - \theta^k\|^2\\
        &\overset{\eqref{eq:a+b}}{\le}& \langle\nabla f(\theta^k), \widehat{\theta}^{k+1} - \theta^k \rangle + \langle\nabla f(\theta^k), \theta^{k+1} - \widehat{\theta}^{k+1} \rangle\\
        &&\quad+ L\|\widehat{\theta}^{k+1} - \theta^k\|^2 + L\|\theta^{k+1} - \widehat{\theta}^{k+1}\|^2\\
        &=& - \gamma\langle\nabla f(\theta^k), g^k\rangle + L\gamma^2\|g^k\|^2 + \langle\nabla f(\theta^k), \theta^{k+1} - \widehat{\theta}^{k+1} \rangle\\
        &&\quad + L\|\theta^{k+1} - \widehat{\theta}^{k+1}\|^2,
    \end{eqnarray*}
    from which it follows that
    \begin{eqnarray}
        \EE\left[f(\theta^{k+1}) - f(\theta^k)\mid \theta^k\right] &\le& -\gamma\left\langle\nabla f(\theta^k), \frac{1}{N_k}\sum\limits_{i\in P_k}\nabla f(\theta_i^k) \right\rangle\notag\\
        &&\quad + \EE\left[\langle\nabla f(\theta^k), \theta^{k+1} - \widehat{\theta}^{k+1} \rangle\mid \theta^k\right]\notag\\
        &&\quad + \EE\left[L\|\theta^{k+1} - \widehat{\theta}^{k+1}\|^2\mid \theta^k\right]\notag\\
        &&\quad + L\gamma^2\EE\left[\left\|\frac{1}{N_k}\sum\limits_{i\in P_k}g_i^k\right\|^2\mid \theta^k\right],\label{eq:key_lemma_non_cvx_tech_2}
    \end{eqnarray}
    where $\EE\left[\ \cdot \mid \theta^k\right] := \EE\left[\ \cdot \mid P_k, \theta_i^k, i\in P_k\right]$. Next, we estimate the last three terms in the right-hand side of \eqref{eq:key_lemma_non_cvx_tech_2}. First of all,
\begin{eqnarray}
    -\gamma\left\langle\nabla f(\theta^k), \frac{1}{N_k}\sum\limits_{i\in P_k}\nabla f(\theta_i^k)\right\rangle &=& -\gamma\|\nabla f(\theta^k)\|^2 \notag\\
    &&\quad - \gamma\left\langle\nabla f(\theta^k), \frac{1}{N_k}\sum\limits_{i\in P_k}\nabla f(\theta_i^k) - \nabla f(\theta^k)\right\rangle \notag\\
    &\overset{\eqref{eq:young_inequality}}{\le}& -\gamma\|\nabla f(\theta^k)\|^2 + \frac{\gamma}{2}\|\nabla f(\theta^k)\|^2\notag\\
    &&\quad+ \frac{\gamma}{2}\left\|\frac{1}{N_k}\sum\limits_{i\in P_k}(\nabla f(\theta_i^k) - \nabla f(\theta^k))\right\|^2\notag\\
    &\overset{\eqref{eq:jensen_ineq}}{\le}& - \frac{\gamma}{2}\|\nabla f(\theta^k)\|^2 + \frac{\gamma}{2N_k}\sum\limits_{i\in P_k}\|\nabla f(\theta_i^k) - \nabla f(\theta^k)\|^2\notag\\
    &\overset{\eqref{eq:L_smoothness_def}}{\le}& - \frac{\gamma}{2}\|\nabla f(\theta^k)\|^2 + \frac{\gamma L^2}{2}V_k, \label{eq:key_lemma_non_cvx_tech_3}
\end{eqnarray}
where $V_k = \frac{1}{N_k}\sum_{i\in P_k}\|\theta_i^k - \theta^k\|^2$. Secondly, since the stochastic gradients $\{g_i^k\}_{i\in P_k}$ are computed independently, we derive
\begin{eqnarray}
    L\gamma^2\EE\left[\left\|\frac{1}{N_k}\sum\limits_{i\in P_k}g_i^k\right\|^2\mid \theta^k\right] &\overset{\eqref{eq:variance_decomposition}}{=}& L\gamma^2\left\|\frac{1}{N_k}\sum\limits_{i\in P_k}\nabla f(\theta_i^k)\right\|^2\notag\\
    &&\quad + L\gamma^2\EE\left[\left\|\frac{1}{N_k}\sum\limits_{i\in P_k}(g_i^k-\nabla f(\theta_i^k))\right\|^2\mid \theta^k\right]\notag\\
    &\overset{\eqref{eq:jensen_ineq}}{\le}& 2L\gamma^2 \left\|\frac{1}{N_k}\sum\limits_{i\in P_k}(\nabla f(\theta_i^k)-\nabla f(\theta^k))\right\|^2 \notag\\
    &&\quad + 2L\gamma^2\|\nabla f(\theta^k)\|^2 \notag\\
    &&\quad + \frac{\gamma^2L}{N_k^2}\sum\limits_{i\in P_k}\EE\left[\|g_i^k - \nabla f(\theta_i^k)\|^2\mid \theta^k\right]\notag\\
    &\overset{\eqref{eq:jensen_ineq},\eqref{eq:bounded_variance}}{\le}& \frac{2\gamma^2L}{N_k}\sum\limits_{i\in P_k}\|\nabla f(\theta_i^k)-\nabla f(\theta^k)\|^2\notag\\
    &&\quad + 2L\gamma^2\|\nabla f(\theta^k)\|^2 + \frac{\gamma^2L\sigma^2}{N_k}\notag\\
    &\overset{\eqref{eq:L_smoothness_def}}{\le}& \underbrace{\frac{2L^3\gamma^2}{N_k}\sum\limits_{i\in P_k}\|\theta_i^k - \theta^k\|^2}_{2L^3\gamma^2 V_k} + 2L\gamma^2\|\nabla f(\theta^k)\|^2\notag\\
    &&\quad + \frac{\gamma^2L\sigma^2}{N_{\min}}. \label{eq:key_lemma_non_cvx_tech_4}
\end{eqnarray}
Plugging \eqref{eq:key_lemma_non_cvx_tech_3} and \eqref{eq:key_lemma_non_cvx_tech_4} in \eqref{eq:key_lemma_non_cvx_tech_2}, we obtain
\begin{eqnarray}
    \EE\left[f(\theta^{k+1}) - f(\theta^k)\mid \theta^k\right] &\le& -\frac{\gamma}{2}\left(1 - 4L\gamma\right)\|\nabla f(\theta^k)\|^2 + \frac{\gamma L^2}{2}\left(1 + 4L\gamma\right)V_k + \frac{L\gamma^2\sigma^2}{N_{\min}}\notag\\
    &&\quad + \EE\left[\langle\nabla f(\theta^k), \theta^{k+1} - \widehat{\theta}^{k+1} \rangle + L\|\theta^{k+1} - \widehat{\theta}^{k+1}\|^2\mid \theta^k\right].\notag
\end{eqnarray}
Next, we take the full expectation from the both sides of the above inequality, apply the tower property \eqref{eq:tower_property} and take into account that $\gamma \le \nicefrac{(1-2\delta_{pv,1})}{8L}$:
\begin{eqnarray*}
    \EE\left[f(\theta^{k+1}) - f(\theta^k)\right] &\le& -\frac{\gamma}{2}\left(1 - 4L\gamma\right)\EE\left[\|\nabla f(\theta^k)\|^2\right] + \frac{\gamma L^2}{2}\left(1 + 4L\gamma\right)\EE[V_k] + \frac{L\gamma^2\sigma^2}{N_{\min}}\\
    &&\quad + \EE\left[\langle\nabla f(\theta^k), \theta^{k+1} - \widehat{\theta}^{k+1} \rangle + L\|\theta^{k+1} - \widehat{\theta}^{k+1}\|^2\right]\\
    &\overset{\eqref{eq:key_lemma_non_cvx_tech_1}}{\le}& -\frac{\gamma}{2}\left(1 - 2\delta_{pv,1} - 4L\gamma\right)\EE\left[\|\nabla f(\theta^k)\|^2\right] + \frac{\gamma L^2}{2}\left(1 + 4L\gamma\right)\EE[V_k] \notag\\
    &&\quad + L\gamma^2\left(\frac{\sigma^2}{N_{\min}} + \delta_{pv,2}^2\right)\\
    &\le& -\frac{(1-2\delta_{pv,1})\gamma}{4}\EE\left[\|\nabla f(\theta^k)\|^2\right] + \gamma L^2 \EE[V_k]\notag\\
    &&\quad + L\gamma^2\left(\frac{\sigma^2}{N_{\min}} + \delta_{pv,2}^2\right).
\end{eqnarray*}
Summing up the obtained inequalities for $k = 0,\ldots, K-1$ and rearranging the terms, we derive
\begin{eqnarray*}
    \frac{(1-2\delta_{pv,1})\gamma}{4}\sum\limits_{k=0}^{K-1}\EE\left[\|\nabla f(\theta^k)\|^2\right] &\le& \sum\limits_{k=0}^{K-1} \EE\left[f(\theta^k) - f(\theta^{k+1})\right] + \gamma L^2\sum\limits_{k=0}^{K-1} \EE[V_k]\notag\\
    &&\quad + KL\gamma^2\left(\frac{\sigma^2}{N_{\min}} + \delta_{pv,2}^2\right)\\
    &=& f(\theta^0) - \EE[f(\theta^{K})] + \gamma L^2\sum\limits_{k=0}^{K-1} \EE[V_k] \\
    &&\quad + KL\gamma^2\left(\frac{\sigma^2}{N_{\min}} + \delta_{pv,2}^2\right)\\
    &\le& f(\theta^0) - f_* + \gamma L^2\sum\limits_{k=0}^{K-1} \EE[V_k]\\
    &&\quad + KL\gamma^2\left(\frac{\sigma^2}{N_{\min}} + \delta_{pv,2}^2\right),
\end{eqnarray*}
where $f_*$ is a uniform lower bound for $f$.
\end{proof}
The next step towards completing the proof of Theorem~\ref{thm:non_cvx_convergence} gives the upper bound for $\sum_{k=0}^{K-1} \EE[V_k]$ that appeared in \eqref{eq:key_lemma_non_cvx}.

\begin{lemma}\label{lem:V_k_lemma_non_cvx}
    Let $f_1 = \ldots = f_N = f$ be $L$-smooth and bounded from below by $f_*$, and Assumptions~\ref{as:bounded_var}~and~\ref{as:averaging_quality} hold with $\Delta_{pv}^k = \delta_{pv,1}\gamma\EE[\|\nabla f(\theta^k)\|^2] + L\gamma^2\delta_{pv,2}^2$, $\delta_{pv,1}\in [0,\nicefrac{1}{2})$, $\delta_{pv,2}\ge 0$. Then, for any $K \ge 0$ the iterates produced by Moshpit SGD with $\gamma \le \nicefrac{1}{\left(4\sqrt{e}L(\tau-1)\right)}$ satisfy
    \begin{eqnarray}
        \sum\limits_{k=0}^{K-1}\EE[V_k] &\le& 8e\gamma^2(\tau-1)^2\sum\limits_{k=0}^{K-1}\EE[\|\nabla f(\theta^k)\|^2] + 4\gamma^2K\left(2\delta_{aq}^2 + e(\tau-1)\sigma^2\right) ,\label{eq:V_k_lemma_non_cvx}
    \end{eqnarray}
    where $V_k = \frac{1}{N_k}\sum_{i\in P_k}\|\theta_i^k - \theta^k\|^2$ and $\theta^k = \frac{1}{N_k}\sum_{i\in P_k}\theta_i^k$.
\end{lemma}
\begin{proof}
    First of all, consider $k$ such that $k = a\tau + t'$ for some $t'\in [0,\tau)$. Let $\EE_g[\cdot]$ denote the expectation conditioned on $\{P_t\}_{t=a\tau}^{(a+1)\tau-1}$. Then
     \begin{eqnarray}
         \EE_g[V_k] &=& \frac{1}{N_k}\sum\limits_{i\in P_k}\EE_g\left[\|\theta_i^k - \theta^k\|^2\right] \overset{\eqref{eq:variance_decomposition}}{\le} \frac{1}{N_k}\sum\limits_{i\in P_k}\EE_g\left[\|\theta_i^k - \theta^{a\tau}\|^2\right] \notag\\
         &=& \frac{1}{N_k}\sum\limits_{i\in P_k}\EE_g\left[\left\|\theta_i^{a\tau} - \theta^{a\tau} - \gamma\sum\limits_{t=a\tau}^{k-1} g_i^t\right\|^2\right]\notag\\
         &\overset{\eqref{eq:a+b}}{\le}& \frac{2}{N_k} \sum\limits_{i\in P_k}\EE_g\left[\|\theta_i^{a\tau} - \theta^{a\tau}\|^2\right] + \frac{2\gamma^2}{N_k}\sum\limits_{i\in P_k}\EE_g\left[\left\|\sum\limits_{t=a\tau}^{k-1} g_i^t\right\|^2\right]. \label{eq:V_k_lemma_non_cvx_tech_1}
     \end{eqnarray}
     Next, we estimate the second term in the right-hand side of \eqref{eq:V_k_lemma_non_cvx_tech_1} using Lemma~\ref{lem:lemma_i_1_gorbunov}:
     \begin{eqnarray}
         \frac{2\gamma^2}{N_k}\sum\limits_{i\in P_k}\EE_g\left[\left\|\sum\limits_{t=a\tau}^{k-1} g_i^t\right\|^2\right] &\overset{\eqref{eq:lemma_i_1_gorbunov}}{\le}& \frac{2e\gamma^2(k - a\tau)}{N_k} \sum\limits_{i\in P_k} \sum\limits_{t=a\tau}^{k-1}\EE_g[\|\nabla f(\theta_i^t)\|^2]\notag\\
         &&\quad + \frac{2e\gamma^2}{N_k}\sum\limits_{i\in P_k} \sum\limits_{t=a\tau}^{k-1}\EE_g[\|g_i^t - \nabla f(\theta_i^t)\|^2]\notag\\
         &\overset{\eqref{eq:a+b},\eqref{eq:bounded_variance}}{\le}& 4e\gamma^2(\tau-1) \sum\limits_{t=a\tau}^{k-1}\EE_g[\|\nabla f(\theta^t)\|^2] \notag\\
         &&\quad+ 4e\gamma^2(\tau-1) \sum\limits_{t=a\tau}^{k-1}\frac{1}{N_k}\sum\limits_{i\in P_k}\EE_g[\|\nabla f(\theta_i^t) - \nabla f(\theta^t)\|^2] \notag\\
         &&\quad+ 2e\gamma^2 (k - a\tau)\sigma^2\notag\\
         &\overset{\eqref{eq:L_smoothness_def}}{\le}& 4e\gamma^2(\tau-1) \sum\limits_{t=a\tau}^{k-1}\EE_g[\|\nabla f(\theta^t)\|^2]\notag\\
         &&\quad + 4e\gamma^2L^2(\tau-1) \sum\limits_{t=a\tau}^{k-1}\frac{N_t}{N_k}\cdot\frac{1}{N_t}\sum\limits_{i\in P_t}\EE_g[\|\theta_i^t - \theta^t\|^2]\notag\\
         &&\quad + 2e\gamma^2(\tau-1)\sigma^2\notag\\
         &\le& 4e\gamma^2(\tau-1) \sum\limits_{t=a\tau}^{k-1}\EE_g[\|\nabla f(\theta^t)\|^2] \notag\\
         &&\quad + 8e\gamma^2L^2(\tau-1) \sum\limits_{t=a\tau}^{k-1}\EE_g[V_t] + 2e\gamma^2(\tau-1)\sigma^2,\notag
     \end{eqnarray}
     where in the last two inequalities we use $N_k = |P_k| \le |P_{k-1}| = N_{k-1}$ for all $k\ge 1$ and $N_{a\tau} \le 2 N_{(a+1)\tau}$ for all integer $a \ge 0$. Plugging this inequality in \eqref{eq:V_k_lemma_non_cvx_tech_1} and taking the full expectation from the result, we get
     \begin{eqnarray}
         \EE[V_k] &\le& 2\EE\left[\frac{1}{N_k}\sum\limits_{i\in P_k}\|\theta_i^{a\tau} - \theta^{a\tau}\|^2\right] + 4e\gamma^2(\tau-1) \sum\limits_{t=a\tau}^{k-1}\EE[\|\nabla f(\theta^t)\|^2]\notag\\
         &&\quad + 8e\gamma^2L^2(\tau-1) \sum\limits_{t=a\tau}^{k-1}\EE[V_t] + 2e\gamma^2(\tau-1)\sigma^2\notag\\
         &\le& 4\EE\left[\frac{1}{N_{a\tau}}\sum\limits_{i\in P_{a\tau}}\|\theta_i^{a\tau} - \theta^{a\tau}\|^2\right] + 4e\gamma^2(\tau-1) \sum\limits_{t=a\tau}^{k-1}\EE[\|\nabla f(\theta^t)\|^2] \notag\\
         &&\quad + 8e\gamma^2L^2(\tau-1) \sum\limits_{t=a\tau}^{k-1}\EE[V_t] + 2e\gamma^2(\tau-1)\sigma^2\notag\\
         &\overset{\eqref{eq:quality_of_avg}}{\le}& 4e\gamma^2(\tau-1) \sum\limits_{t=a\tau}^{k-1}\EE[\|\nabla f(\theta^t)\|^2] + 8e\gamma^2L^2(\tau-1) \sum\limits_{t=a\tau}^{k-1}\EE[V_t]\notag\\
         &&\quad + 2\gamma^2\left(2\delta_{aq}^2 + e(\tau-1)\sigma^2\right),\notag
     \end{eqnarray}
     where in the second inequality we also use $N_k = |P_k| \le |P_{k-1}| = N_{k-1}$ for all $k\ge 1$ and $N_{a\tau} \le 2 N_{(a+1)\tau}$ for all integer $a \ge 0$. Summing up the obtained inequalities for $k = a\tau, a\tau+1,\ldots, K'$ for some $K' \in[a\tau, (a+1)\tau-1]$ we derive
     \begin{eqnarray*}
         \sum\limits_{k=a\tau}^{K'}\EE[V_k] &\le& 4e\gamma^2(\tau-1)\sum\limits_{k=a\tau}^{K'} \sum\limits_{t=a\tau}^{k-1}\EE[\|\nabla f(\theta^t)\|^2] + 8e\gamma^2L^2(\tau-1) \sum\limits_{k=a\tau}^{K'}\sum\limits_{t=a\tau}^{k-1}\EE[V_t]\\
         &&\quad + 2\gamma^2(K'-a\tau+1)\left(2\delta_{aq}^2 + e(\tau-1)\sigma^2\right)\\
         &\le& 4e\gamma^2(\tau-1)^2\sum\limits_{k=a\tau}^{K'} \EE[\|\nabla f(\theta^k)\|^2] + 8e\gamma^2L^2(\tau-1)^2 \sum\limits_{k=a\tau}^{K'}\EE[V_k]\\
         &&\quad + 2\gamma^2(K'-a\tau+1)\left(2\delta_{aq}^2 + e(\tau-1)\sigma^2\right)\\
         &\le& 4e\gamma^2(\tau-1)^2\sum\limits_{k=a\tau}^{K'} \EE[\|\nabla f(\theta^k)\|^2] + \frac{1}{2} \sum\limits_{k=a\tau}^{K'}\EE[V_k]\notag\\
         &&\quad + 2\gamma^2(K'-a\tau+1)\left(2\delta_{aq}^2 + e(\tau-1)\sigma^2\right),
     \end{eqnarray*}
     where in the last inequality we use $\gamma \le \nicefrac{1}{\left(4\sqrt{e}L(\tau-1)\right)}$. Rearranging the terms, we get that for $K' \ge 0$
     \begin{eqnarray*}
         \sum\limits_{k=a\tau}^{K'} \EE[V_k] &\le& 8e\gamma^2(\tau-1)^2\sum\limits_{k=a\tau}^{K'}\EE[\|\nabla f(\theta^k)\|^2] + 4\gamma^2(K'-a\tau+1)\left(2\delta_{aq}^2 + e(\tau-1)\sigma^2\right),
     \end{eqnarray*}
     where $a\ge 0$ is an integer such that $a\tau \le K' \le (a+1)\tau - 1$. Summing up the obtained inequalities for $K' = \tau-1, 2\tau-1,\ldots, \tau\lfloor\nicefrac{(K-1)}{\tau}\rfloor - 1, K-1$, we derive \eqref{eq:V_k_lemma_non_cvx}.
\end{proof}

Combining Lemmas~\ref{lem:key_lemma_non_cvx}~and~\ref{lem:V_k_lemma_non_cvx}, we get the following result:
\begin{theorem}[Theorem~\ref{thm:non_cvx_convergence}]
    Let $f_1 = \ldots = f_N = f$, function $f$ be $L$-smooth and bounded from below by $f_*$, and Assumptions~\ref{as:bounded_var}~and~\ref{as:averaging_quality} hold with $\Delta_{pv}^k = \delta_{pv,1}\gamma\EE[\|\nabla f(\theta^k)\|^2] + L\gamma^2\delta_{pv,2}^2$, $\delta_{pv,1}\in [0,\nicefrac{1}{2})$, $\delta_{pv,2}\ge 0$. Then, for any $K \ge 0$ the iterates produced by Moshpit SGD with
    \begin{equation*}
        \gamma \le \min\left\{\frac{1-2\delta_{pv,1}}{8L},\frac{\sqrt{1-2\delta_{pv,1}}}{8\sqrt{e}L(\tau-1)}\right\}
    \end{equation*}
    satisfy
    \begin{eqnarray}
        \EE\left[\|\nabla f(\theta_{\text{rand}}^K)\|^2\right] &\le& \frac{8\Delta_0}{(1-2\delta_{pv,1})K\gamma} \notag\\
        &&\quad + \frac{8L\gamma}{1-2\delta_{pv,1}}\left(\frac{\sigma^2}{N_{\min}} + \delta_{pv,2}^2 + 4\gamma L\left(2\delta_{aq}^2 + e(\tau-1)\sigma^2\right)\right), \label{eq:non_cvx_bound_supp}
    \end{eqnarray}
    where $\Delta_0 = f(\theta^0) - f_*$ and $\theta_{\text{rand}}^K$ is chosen uniformly at random from $\{\theta^0,\theta^1,\ldots,\theta^{K-1}\}$. That is, Moshpit SGD achieves $\EE\left[\|\nabla f(\theta_{\text{rand}}^K)\|^2\right] \le \varepsilon^2$ after 
    \begin{eqnarray}
        \cO\Bigg(\frac{L\Delta_0}{(1-2\delta_{pv,1})^2\varepsilon^2}\Bigg[1 +(\tau-1)\sqrt{1-2\delta_{pv,1}} + \frac{\delta_{pv,2}^2 + \nicefrac{\sigma^2}{N_{\min}}}{\varepsilon^2}&\notag\\
        &\hspace{-2cm} + \frac{\sqrt{(1-2\delta_{pv,1})(\delta_{aq}^2+(\tau-1)\sigma^2)}}{\varepsilon}\Bigg]\Bigg)\label{eq:non_cvx_bound_2_supp}
    \end{eqnarray}
    iterations with
    \begin{equation*}
        \gamma = \min\left\{\frac{1-2\delta_{pv,1}}{8L},\frac{\sqrt{1-2\delta_{pv,1}}}{8\sqrt{e}L(\tau-1)}, \sqrt{\frac{\Delta_0}{LK\left(\delta_{pv,2}^2 + \nicefrac{\sigma^2}{N_{\min}}\right)}}, \sqrt[3]{\frac{\Delta_0}{4L^2\left(2\delta_{aq}^2 + e(\tau-1)\sigma^2\right)}}\right\}.
    \end{equation*}
\end{theorem}
\begin{proof}[Proof of Theorem~\ref{thm:non_cvx_convergence}]
    Plugging the result of Lemma~\ref{lem:V_k_lemma_non_cvx} in the inequality \eqref{eq:key_lemma_non_cvx} from Lemma~\ref{lem:key_lemma_non_cvx}, we obtain
    \begin{eqnarray*}
        \frac{(1-2\delta_{pv,1})\gamma}{4}\sum\limits_{k=0}^{K-1}\EE\left[\|\nabla f(\theta^k)\|^2\right] &\le& f(\theta^0) - f_* + 8e\gamma^3L^2\tau(\tau-1)\sum\limits_{k=0}^{K-1}\EE[\|\nabla f(\theta^k)\|^2] \\
        &&\quad + KL\gamma^2\left(\frac{\sigma^2}{N_{\min}} + \delta_{pv,2}^2\right)\\
        &&\quad + 4KL^2\gamma^3\left(2\delta_{aq}^2 + e(\tau-1)\sigma^2\right)\\
        &\le& f(\theta^0) - f_* + \frac{(1-2\delta_{pv,1})\gamma}{8}\sum\limits_{k=0}^{K-1}\EE\left[\|\nabla f(\theta^k)\|^2\right] \\
         &&\quad + KL\gamma^2\left(\frac{\sigma^2}{N_{\min}} + \delta_{pv,2}^2\right)\\
        &&\quad + 4KL^2\gamma^3\left(2\delta_{aq}^2 + e(\tau-1)\sigma^2\right).
    \end{eqnarray*}
    Next,
    \begin{eqnarray*}
        \frac{1}{K}\sum\limits_{k=0}^K\EE\left[\|\nabla f(\theta^k)\|^2\right] &\le& \frac{8\Delta_0}{(1-2\delta_{pv,1})K\gamma} \\
        &&\quad + \frac{8L\gamma}{1-2\delta_{pv,1}}\left(\frac{\sigma^2}{N_{\min}} + \delta_{pv,2}^2 + 4\gamma L\left(2\delta_{aq}^2 + e(\tau-1)\sigma^2\right)\right),
    \end{eqnarray*}
    where $\Delta_0 = f(\theta^0) - f_*$. Since $\theta_{\text{rand}}^K$ is chosen uniformly at random from $\{\theta^0,\theta^1,\ldots,\theta^{K-1}\}$,
    \begin{equation*}
        \EE\left[\|\nabla f(\theta_{\text{rand}}^K)\|^2\right] \overset{\eqref{eq:tower_property}}{=} \frac{1}{K}\sum\limits_{k=0}^K\EE\left[\|\nabla f(\theta^k)\|^2\right]
    \end{equation*}
    and \eqref{eq:non_cvx_bound_supp} holds. Applying Lemma~\ref{lem:lemma_i_3_gorbunov} to \eqref{eq:non_cvx_bound_supp}, we get the following result: if
    \begin{equation*}
        \gamma = \min\left\{\frac{1-2\delta_{pv,1}}{8L},\frac{\sqrt{1-2\delta_{pv,1}}}{8\sqrt{e}L(\tau-1)}, \sqrt{\frac{\Delta_0}{LK\left(\delta_{pv,2}^2 + \nicefrac{\sigma^2}{N_{\min}}\right)}}, \sqrt[3]{\frac{\Delta_0}{4L^2\left(2\delta_{aq}^2 + e(\tau-1)\sigma^2\right)}}\right\},
    \end{equation*}
    then $\EE\left[\|\nabla f(\theta_{\text{rand}}^K)\|^2\right]$ equals
    \begin{equation*}
        \cO\!\left(\!\frac{L\Delta_0\left(1\!+\! (\tau\!-\!1)\sqrt{1\!-\!2\delta_{pv,1}}\right)}{(1\!-\!2\delta_{pv,1})^2K} + \sqrt{\frac{L\Delta_0\left(\delta_{pv,2}^2\! +\! \nicefrac{\sigma^2}{N_{\min}}\right)}{(1\!-\!2\delta_{pv,1})^2K}} + \frac{\sqrt[3]{L^2\Delta_0^2(\delta_{aq}^2\! +\! (\tau\!-\!1)\sigma^2)}}{(1\!-\!2\delta_{pv,1})K^{\nicefrac{2}{3}}}\!\right)\!,
    \end{equation*}
    which implies the desired convergence result from \eqref{eq:non_cvx_bound_2_supp}.
\end{proof}

\section{Decentralized matchmaking}
\label{sect:matchmaking}

In order to run group all-reduce over unreliable devices, Moshpit Averaging must be able to dynamically form groups of active devices that share the same key $C_i$.
In theory, this matchmaking can be implemented precisely as described in Algorithm~\ref{alg:moshpit}: each peer adds itself to a certain DHT key, waits for a said period of time, and then reads the same key to retrieve a list of its groupmates.

However, in practice, this kind of matchmaking would be extremely fragile: if any peer arrives late (for example, due to latency), it may join the group when other peers have already finished matchmaking. As a result, some workers will treat this peer as active, while others will behave as though there is no such peer at all, breaking the consensus and rendering all peers unable to run all-reduce in a stable manner.

To avoid this and other similar inconsistencies, Moshpit All-Reduce employs a more sophisticated matchmaking protocol with the following guarantees 
\begin{enumerate}
    \item Peers that join the same group are guaranteed to have the same list of groupmates;
    \item The group will have the maximum possible number of peers, unless some of them fail;
    \item If some peers fail, matchmaking will still form the group out of the remaining ones.
\end{enumerate}

To achieve this, each peer first declares itself onto the DHT (as in Algorithm~\ref{alg:moshpit}). Then, peers attempt to form groups by calling the \texttt{REQUEST\_JOIN\_GROUP} remote procedure call. Intuitively, if peer A calls this RPC on peer B, then \textit{peer A requests to join peer B's group}, which can be either accepted or rejected by the group ``leader'' B, which may or may not have other ``followers''.

If a peer is accepted to a group, it commits to stay active (i.e. to await other peers) for a set period of time and perform all-reduce with the peers supplied by the group ``leader''. On the other hand, a peer can be rejected if (a) the potential ``leader'' is already a follower in another group, (b) the group is already running all-reduce, or (c) if the ``leader'' failed or left during matchmaking.

To ensure that this protocol forms groups of maximum size, each peer generates a unique ``priority'' based on its local timestamp\footnote{More specifically, the priority is a tuple of $\texttt{(timestamp, peer\_id)}$, where \texttt{peer\_id} is used to break ties.}. Peers prioritize joining the group of neighbors that have the lowest ``priority''. Under normal circumstances, all workers will join the group of a peer that was first to start matchmaking according to its own local time. However, if this peer has failed or already finished matchmaking, the group will be formed around one of the remaining peers.

Matchmaking for 64 peers can take less than 1 second if all workers are located in the same cloud region and are highly synchronized. However, this can grow to 2.9 seconds for two different cloud regions and up to 9 seconds when training with commodity hardware around the world.

To ensure that this latency does not affect the training performance, Moshpit SGD performs matchmaking asynchronously in the background thread, while the model is accumulating gradients. All peers begin matchmaking 15 seconds before the estimated averaging round, so that in $\ge 95\%$ of averaging iterations, the matchmaking step is already finished by the time peers need to run all-reduce.
\section{Training with a dynamic number of peers}
\label{sect:load_state_from_peers}

Many practical setups with unreliable devices allow peers to join or leave at any time, which can produce undesirable side-effects. For instance, consider a participant that joins the ``swarm'' midway through the training process. If this participant starts with the initial model parameters, it can undo some of the progress made by other peers.

To circumvent this issue, we require each new participant to download the latest parameters from a random up-to-date peer discovered through DHT. The same technique is used to synchronize the optimizer statistics and the learning rate schedule. This protocol is also triggered if a peer becomes desynchronized with others, e.g., after a network freeze.

\section{Load balancing via linear programming}
\label{sect:load_balancing}

When running Moshpit Averaging on heterogeneous devices, one must regularly perform Butterfly All-Reduce among peers with uneven network bandwidth.
In order to speed up the protocol, we can make low-throughput peers receive, average, and send smaller partitions of the averaged vector; conversely, the high-throughput peers can process greater fractions of the input vector.
To compute the optimal partitioning, peers must solve an optimization problem that minimizes the total time spent on communication during all-reduce.

Consider a group of $M$ peers with network bandwidths $b_1, ..., b_M$, defined for simplicity as the minimum of the upload and download speed for each peer. Our objective is to find $w_i$ --- a fraction of all input vectors to be processed by the $i$-th peer.

In Butterfly All-Reduce, each peer $i$ splits its vector into parts and sends these parts to corresponding peers. Since there is no need to send $w_i$ to itself, $i$-th peer will upload a total of $1 - w_i$ of the vector to its peers.
On the receiving side, peer $i$ will average $w_i$ of the vector from all peers in its group. To do so, it must download $M-1$ vector parts of size $w_i$ from all other peers.
After that, peers distribute the averaged parts by running the same procedure in reverse (see Figure~\ref{fig:butterfly_allreduce}).

Thus, the communication time for each peer is proportional to $t_i = (1-w_i+(M-1) w_i) \cdot \frac{1}{b_i}$ and the total runtime of Butterfly All-Reduce is the maximum communication time over all peers: $T = \max_i t_i=\max_i (1-w_i+(M-1) w_i) \cdot \frac{1}{b_i}$. Formally, we minimize $T$ with respect to $w_i$ with two constraints on the fraction weights:
\begin{alignat*}{3}
\min_w&\quad &\max_i &(1-w_i +&(M-1)w_i)\cdot\frac{1}{b_i}&\\
\text{subject to}&\quad& \sum_{i=1}^M w_i = 1&&&\\
&&w_i \geq 0 &&&\forall i=1,\ldots,M
\end{alignat*}

Because the functions being maximized and the constraints are linear in $w_i$, this problem can be reduced to linear programming~\cite{kaplan1974application}. Namely, we can minimize a surrogate variable $\xi$ such that $\forall i, \ \xi \geq (1-w_i+(M-1)\cdot w_i) \cdot \frac{1}{b_i}$. The resulting linear program is formulated as follows:

\begin{alignat*}{3}
\min_{w,\xi}&\quad& \xi && &\\
\text{subject to}&\quad& \sum_{i=1}^M w_i& = 1 &&\\
&\quad& w_i& \geq 0 &&\quad \forall i=1,\ldots,M\\
&\quad&\xi&\geq (1-&w_i+(M-1)w_i)\cdot\frac{1}{b_i}&\quad\forall i=1,\ldots,M
\end{alignat*}

We solve this problem using the interior point method~\cite{andersen} implemented as part of the SciPy package (\texttt{scipy.optimize.linprog}).
Note that depending on the conditions given by participant bandwidth, optimal weights of specific peers might be equal to 0 in some cases. In essence, this allows our method to smoothly interpolate between data parallelism~\cite{valiant1990bridging}, parameter server~\cite{parameter_server_first} and sharded parameter server~\cite{sharded_ps_first} in manner similar to BytePS~\cite{byteps}.

\section{Detailed experimental setup}
\label{sect:detailed_setup}

In this section, we provide the detailed hardware configuration of servers used for each of our distributed training experiments.

\subsection{ImageNet training}\label{sect:detailed_setup_resnet}

Both homogeneous and heterogeneous training setups for ImageNet are provisioned in our on-premise infrastructure across multiple data centers and an office space (for the heterogeneous setup only).

\paragraph{Homogeneous.}For the homogeneous setup, we use 16 identical instances with the following specifications:
\begin{itemize}
    \item \textbf{GPU:} V100-PCIe,
    \item \textbf{CPU:} 6 vCPUs (Xeon E5-2650v4),
    \item \textbf{RAM:} 64GB.
\end{itemize}

\paragraph{Heterogeneous.}In turn, the heterogeneous setup contains multiple instance types listed in Table~\ref{fig:tab_setup_resnet}:
\begin{table}[h]
\centering
\caption{\textbf{Heterogeneous} setup for ImageNet training.}
\label{fig:tab_setup_resnet}
\renewcommand{\arraystretch}{1}
\begin{tabular}{@{}cccccc@{}}
\toprule
Instances & GPUs & GPU type & Cores & RAM, GB & CPU type \\ 
\midrule
4            & 1      & V100-PCIe  & 6        & 64     & E5-2650v4 \\
17           & 2      & GTX 1080Ti & 8        & 64     & E5-2650v4 \\
7            & 1      & GTX 1080Ti & 4        & 32     & E5-2650v4 \\
16           & 1      & P40  & 4        & 32     & E5-2667v2 \\
20           & 1      & M40-24GB  & 4        & 32     & E5-2667v2 \\

\bottomrule
\end{tabular}
\end{table}

\subsection{ALBERT training}\label{sect:detailed_setup_albert}

\paragraph{Homogeneous.}For the homogeneous setup, we use a single virtual machine with the following specifications:
\begin{itemize}
    \item \textbf{GPU:} $8{\times}$ V100-PCIe,
    \item \textbf{CPU:} 48 vCPUs (Xeon E5-2650v4),
    \item \textbf{RAM:} 488GB.
\end{itemize}

At the time of writing, the cloud rent cost for this instance is \textbf{\$24.48} per hour.

\paragraph{Heterogeneous.}Our heterogeneous setup is composed of two parts: AWS EC2 Spot instances and crowdsourced machines from the \texttt{Vast.ai} marketplace. For spot instances, we picked the smallest suitable instance size available from the cloud provider and further limited their bandwidth to 1Gb/s\footnote{We use \texttt{tc qdisc} Linux utility to artificially limit the network throughput, similarly to~\cite{MLSYS2019_d09bf415}}. As for marketplace instances, we report the hardware specifications for each worker gathered 1 hour after the start of ALBERT training.

Since both cloud and marketplace instances are preemptible, the actual cost of the server fleet will vary based on the current price. For simplicity, we report the maximum hourly price we ended up paying for this instance (enforced via maximum bid). Finally, some marketplace instances have missing specifications, such as unknown CPU type. This is likely caused by non-standard virtualization configured by the device owner. The resulting fleet configuration, shown in Table~\ref{fig:tab_setup}, costs up to \$15.43/hour, depending on the number of active instances.

\begin{table*}[ht!]
\centering
\caption{\textbf{Heterogeneous} setup for ALBERT training.}
\label{fig:tab_setup}
\small
\setlength{\tabcolsep}{2pt}
\hspace{7pt}\begin{tabular}{@{}ccccccc@{}}
\toprule
GPU           & Cores & RAM, GB & CPU type                       & Download, Mb/s & Upload, Mb/s &
Cost, \$/hour \\ 
\midrule
\multicolumn{7}{c}{Preemptible \texttt{g4dn.xlarge} instances ($32{\times}$)} \\
\midrule
T4            & 4         & 16     & Xeon Platinum 8259CL           & 1000          & 1000        & 0.1578         \\

\midrule
\multicolumn{7}{c}{Marketplace instances} \\    
\midrule
GTX 1070Ti    & 6         & 16     & E5-2640                        & 425           & 255         & 0.036         \\
GTX 1070Ti    & 6         & 16     & i3-6100T                       & 121           & 36          & 0.06          \\
GTX 1080Ti    & 4         & 20     & i3-6096P                       & 817           & 308         & 0.101         \\
GTX 1080Ti    & 20        & 129    & E5-2630v4                      & 660           & 475         & 0.182         \\
GTX 1080Ti    & 1         & 16     & i7-7700K                       & 245           & 210         & 0.302         \\
GTX 1080Ti    & 48        & 97     & Xeon Platinum 8124             & 583           & 539         & 0.217         \\
GTX 1080Ti    & 10        & 16     & Unknown                        & n/a           & n/a           & 0.15          \\
GTX 1080Ti    & 4         & 16     & Xeon Gold 6149                 & 98            & 100         & 0.2           \\ %
GTX 1080Ti    & 4         & 16     & Xeon Gold 6149                 & 99            & 98          & 0.2           \\ %
GTX 1080Ti    & 4         & 16     & Xeon Gold 6149                 & 99            & 99          & 0.2           \\ %
GTX 1080Ti    & 4         & 16     & Xeon Gold 6149                 & 99            & 99          & 0.2           \\ %
RTX 2070S     & 24        & 32     & E5-2620v2                      & 199           & 25          & 0.199         \\
RTX 2070S     & 32        & 97     & E5-2650                        & 162           & 64          & 0.285         \\
RTX 2080      & 6         & 16     & E5-2620v3                      & 271           & 287         & 0.25          \\
RTX 2080      & 24        & 32     & E5-2630v3                      & 199           & 25          & 0.302         \\
RTX 2080S     & 4         & 32     & E5-2697v4                      & 101           & 99          & 0.292         \\ %
RTX 2080S     & 4         & 32     & E5-2697v4                      & 93            & 99          & 0.292         \\ %
RTX 2080S     & 4         & 32     & E5-2697v4                      & 94            & 98          & 0.292         \\ %
RTX 2080S     & 4         & 32     & E5-2697v4                      & 94            & 98          & 0.292         \\ %
RTX 2080S     & 4         & 32     & E5-2697v4                      & 100           & 99          & 0.292         \\ %
RTX 2080Ti   & 4         & 16     & Ryzen Threadripper 3960x       & 279           & 271          & 0.35          \\
RTX 2080Ti   & 8         & 129    & E5-2670v3                      & 616           & 672          & 0.201         \\
RTX 2080Ti   & 6         & 32     & E5-2620v3                      & 217           & 61           & 0.22          \\
RTX 2080Ti   & 8         & 16     & E5-2697v2                      & 100           & 58           & 0.3           \\
RTX 2080Ti   & 8         & 21     & E5-2697v2                      & 145           & 49           & 0.243         \\
RTX 2080Ti    & 12        & 32     & Unknown                        & 111          & 92          & 0.326         \\
RTX 2080Ti    & 12        & 64     & E5-2690v3                      & 205          & 61          & 0.549         \\
RTX 3080      & 16        & 16     & i7-10700K                      & 69           & 49          & 0.462         \\
RTX 3090      & 14        & 32     & E5-2695v3                      & 93           & 37          & 0.498         \\
RTX 3090      & 16        & 32     & Ryzen 9 3950X                  & 338          & 38          & 0.511         \\
Titan RTX     & 4         & 32     & Xeon W-3223                   & 321           & 115          & 1             \\
Titan RTX     & 4         & 32     & Xeon Gold 6149                 & 99           & 100         & 0.702         \\ %
Titan V       & 8         & 32     & i7-7700K                       & 97           & 50          & 0.282         \\
V100-FHHL     & 8         & 60     & Xeon Gold 6148                 & 544          & 584         & 0.39          \\
\midrule
\multicolumn{6}{c}{Total hourly cost (as listed):} &\bf 15.43 \\    
\bottomrule
\end{tabular}
\end{table*}

\section{Additional averaging experiments}
\label{sect:extra_averaging}

In this section, we evaluate the averaging precision with the same methodology as in~\ref{sect:experiments_averaging}, but for multiple different worker configurations. 

Table~\ref{tab:full_averaging} provides the complete results of our experiments that were used to make conclusions in the main experimental section: instead of reporting the mean squared error for different iterations, we provide the number of rounds that was required to achieve the error of $10^{-9}$ and $10^{-4}$.

In Figure~\ref{fig:many_averagings}, plots 1--5 explore several combinations of grid sizes and failure rates, whereas plot 6 (bottom right) demonstrates a setup with the same number of peers ($10^6$) arranged into several different grid sizes and its relation to convergence. Note that $M{=}32$ outperforms the alternatives only for the specific failure rate of $0.001$.

\begin{table}[ht]
\centering
\caption{Averaging performance of different algorithms. Values denote the number of iterations required to achieve the error of $10^{-9}$ ($10^{-4}$ in parentheses), the best result is in bold.}
\vspace{1em}
\label{tab:full_averaging}
\begin{tabular}{@{}llccccc@{}}
\toprule
$N$  & $p$   & All-Reduce  & Gossip      & PushSum     & Random groups & Moshpit   \\ \midrule
512  & 0     & \bf 1.0 (1.0)   & 50.0 (50.0) & 47.6 (15.6) & 6.1 (3.0)     & 8.2 (3.5) \\
512  & 0.001 & \bf 1.6 (1.6)   & 50.0 (50.0) & 47.6 (15.6) & 6.3 (3.0)     & 8.1 (3.7) \\
512  & 0.005 & 10.9 (10.9) & 50.0 (50.0) & 47.8 (15.6) & \bf 6.3 (3.0)     & 8.7 (3.9) \\
512  & 0.01  & 41.7 (41.7) & 50.0 (50.0) & 47.8 (15.6) & \bf 6.6 (3.0)     & 9.1 (3.9) \\ \midrule
768  & 0     & \bf 1.0 (1.0)   & 50.0 (50.0) & 43.2 (13.8) & 6.2 (3.0)     & 6.0 (3.0) \\
768  & 0.001 & \bf 1.8 (1.8)   & 50.0 (50.0) & 43.2 (13.8) & 6.5 (3.0)     & 6.2 (3.0) \\
768  & 0.005 & 28.7 (28.7) & 50.0 (50.0) & 43.2 (14.1) & \bf 6.6 (3.0)     & \bf 6.6 (3.0) \\
768  & 0.01  & 50.0 (50.0) & 50.0 (50.0) & 43.9 (14.2) & 7.0 (3.0)     & \bf 6.8 (3.0) \\ \midrule
900  & 0     & \bf 1.0 (1.0)   & 50.0 (50.0) & 45.0 (14.7) & 6.4 (3.0)     & 5.0 (2.8) \\
900  & 0.001 & \bf 1.8 (1.8)   & 50.0 (50.0) & 45.0 (14.7) & 6.3 (3.0)     & 5.5 (3.0) \\
900  & 0.005 & 50.0 (50.0) & 50.0 (50.0) & 45.2 (14.7) & 6.7 (3.0)     &\bf  5.9 (3.0) \\
900  & 0.01  & 50.0 (50.0) & 50.0 (50.0) & 45.6 (14.9) & 7.0 (3.1)     & \bf 6.4 (3.1) \\ \midrule
1024 & 0     & \bf 1.0 (1.0)   & 50.0 (50.0) & 49.0 (16.2) & 6.2 (3.0)     & 2.0 (2.0) \\
1024 & 0.001 & \bf 2.0 (2.0)   & 50.0 (50.0) & 49.0 (16.3) & 6.5 (3.0)     & 3.4 (2.2) \\
1024 & 0.005 & 42.6 (42.6) & 50.0 (50.0) & 49.5 (16.3) & 6.7 (3.0)     & \bf 5.4 (2.9) \\
1024 & 0.01  & 50.0 (50.0) & 50.0 (50.0) & 49.5 (16.3) & 6.9 (3.1)     & \bf 5.9 (3.0) \\ \bottomrule
\end{tabular}
\end{table}

\begin{figure}[h]
    \centering
    \includegraphics[width=\linewidth]{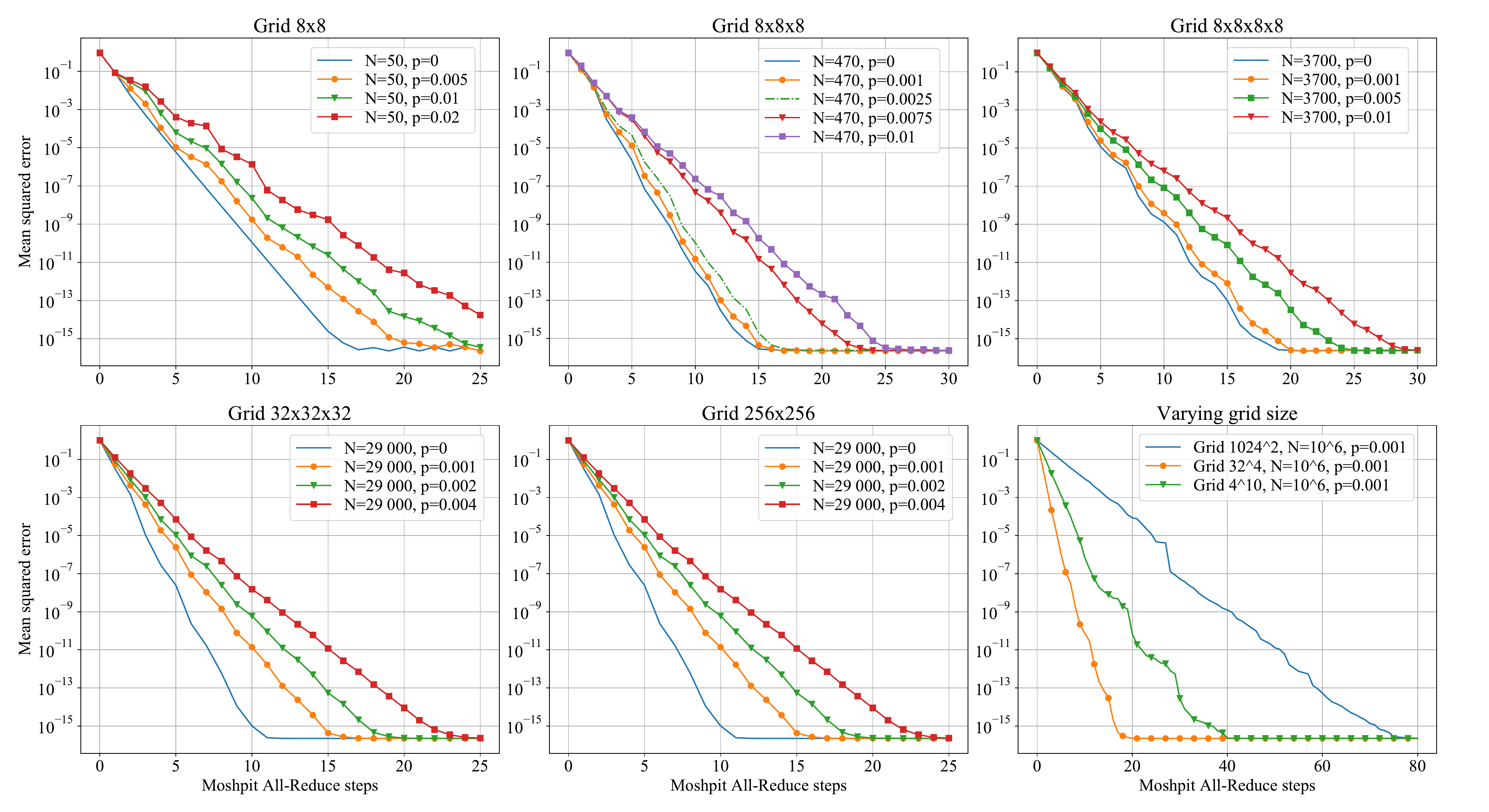}
    \vspace{-20pt}
    \caption{Averaging error of Moshpit All-Reduce as a function of the iteration number for different configurations and failure rates.}
    \label{fig:many_averagings}
\end{figure}

\section{Additional image classification experiments}
\label{sect:extra_classification}

Aside from the two evaluation scenarios provided in~\ref{sect:experiments_vision}, we also measure the performance of Moshpit-SGD in a non-distributed setup, i.e. on a single server with multiple GPUs. We conduct this experiment on the same $8{\times}$ V100 machine that was used in the \textbf{homogeneous} setup for training ALBERT (see Appendix~\ref{sect:detailed_setup_albert}).

\begin{figure}[h]
    \centering
    \begin{tabular}{cc}
    \hspace{-10pt}
        \includegraphics[width=0.5\textwidth]{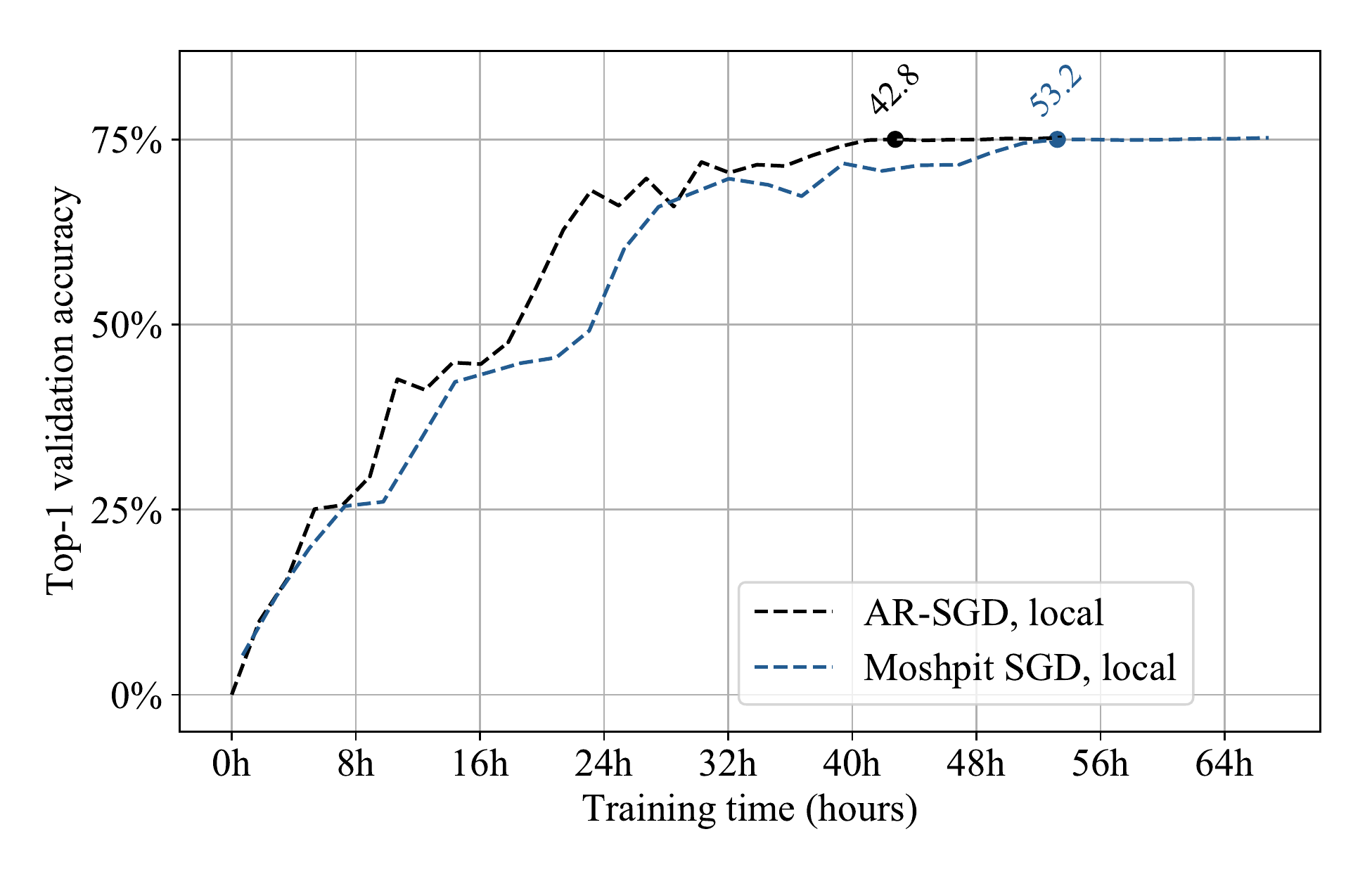} &
        \includegraphics[width=0.5\textwidth]{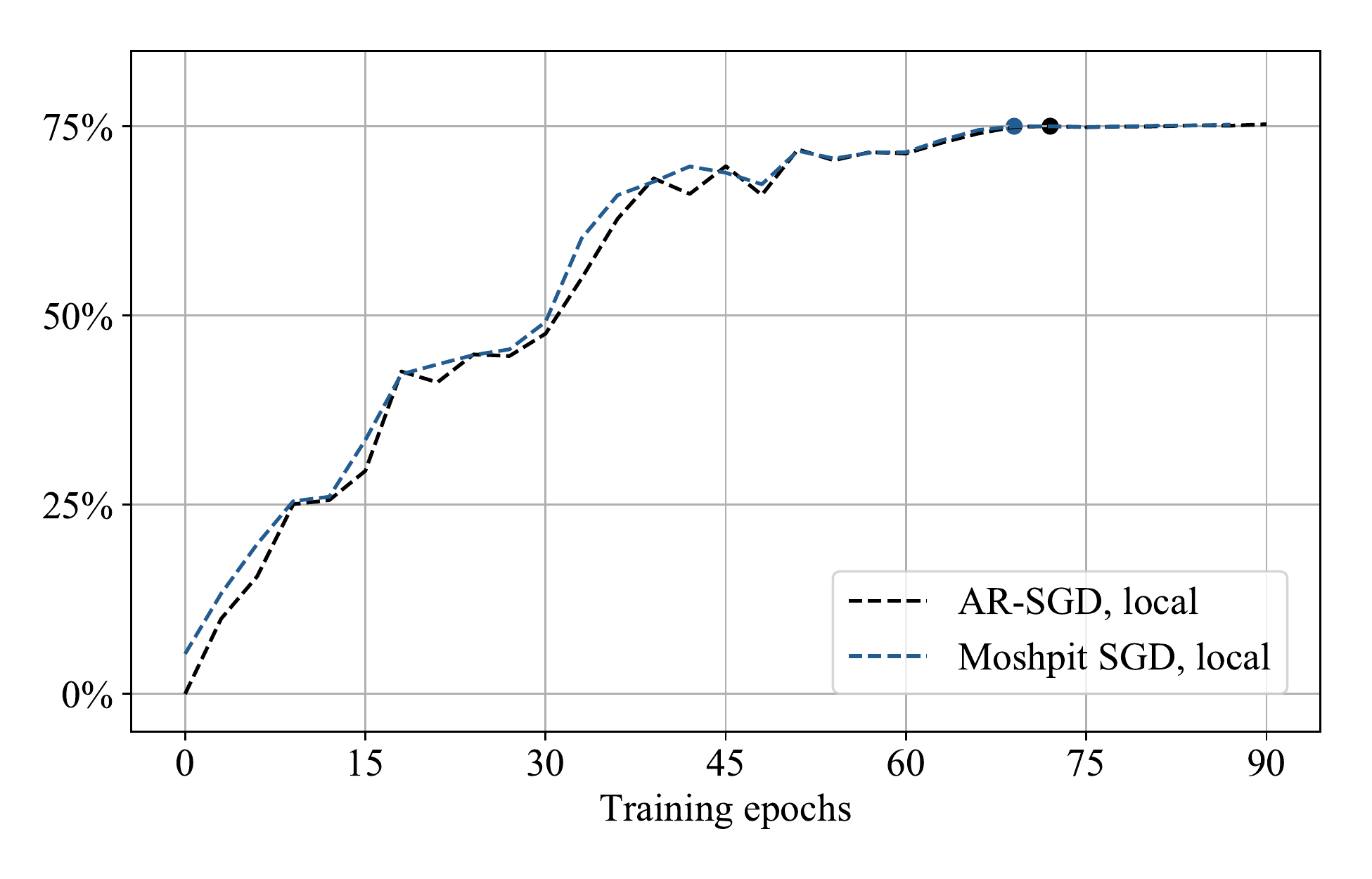}
    \end{tabular}
    \caption{
    ResNet-50 top-1 validation accuracy on ImageNet when training on a single node with $8{\times}$ V100-PCIe GPUs.
    \textbf{(Left)} Convergence in terms of training time, \textbf{(Right)} Convergence in terms of training epochs}
    \label{fig:resnet_local}\vspace{-8pt}
\end{figure}

As Figure~\ref{fig:resnet_local} demonstrates, Moshpit SGD is slower than AR-SGD by approximately $25\%$. This result is expected, since our implementation of Moshpit All-Reduce is more general and communicates over a TCP connection, whereas AR-SGD uses direct peer-to-peer GPU communication over PCIe. On average, this incurs a slowdown of $27\%$ in terms of training time.

\end{document}